\documentclass[11pt]{article}
\usepackage[preprint]{neurips_2026}
\makeatletter
\@ifpackageloaded{geometry}{\geometry{margin=1.1in}}{%
  \usepackage[margin=1.1in]{geometry}}
\makeatother
\usepackage{microtype}
\makeatletter
\AtBeginDocument{\@ifundefined{@noticestring}{}{\gdef\@noticestring{}}}
\makeatother
\usepackage[utf8]{inputenc}
\usepackage{booktabs}
\usepackage{amsmath}
\usepackage{enumitem}

\usepackage{amssymb}
\usepackage{amsthm}
\usepackage{tikz}
\usetikzlibrary{positioning,arrows.meta,calc,shapes,backgrounds,fit}
\usepackage{pgfplots}
\pgfplotsset{compat=1.18}
\usepgfplotslibrary{groupplots}
\usepackage{makecell}
\usepackage{colortbl}
\usepackage{graphicx}
\usepackage{tabularx}
\usepackage{placeins}
\usepackage{float}
\usepackage{url}
\usepackage[hidelinks]{hyperref}
\hypersetup{
  pdftitle={TRW: TRACE-RealWorld -- An Auditable Consistency Contract
    for World Models as Materialized Views},
  pdfauthor={Edward Y. Chang},
  pdfsubject={Auditable consistency contracts for operational world
    models},
  pdfkeywords={world models, materialized views, consistency, adaptive
    refresh, compensation, auditability}
}
\providecommand{\Description}[1]{}

\definecolor{atpblue}{HTML}{1F4E79}
\definecolor{atpteal}{HTML}{2C7A7B}
\definecolor{atpgreen}{HTML}{2F855A}
\definecolor{atpamber}{HTML}{C05621}
\definecolor{atpslate}{HTML}{4A5568}
\tikzset{
  layerbox/.style={draw=atpslate, fill=atpslate!8, rounded corners=2pt, align=center,
                   font=\small, inner sep=4pt, minimum height=8mm},
  gatebox/.style={draw=atpblue, fill=atpblue!12, rounded corners=2pt, align=center,
               font=\small, very thick, inner sep=5pt, minimum height=9mm},
  storebox/.style={draw=atpteal, fill=atpteal!10, rounded corners=2pt, align=center,
                font=\small, inner sep=4pt, minimum height=8mm},
  viewbox/.style={draw=atpgreen, fill=atpgreen!10, rounded corners=2pt, align=center,
               font=\small, inner sep=4pt, minimum height=8mm},
  flow/.style={-{Latex[length=2.2mm]}, atpslate, semithick},
  vflow/.style={-{Latex[length=2.2mm]}, atpgreen, semithick, dashed},
}

\DeclareUrlCommand\fld{\urlstyle{tt}}
\newcommand{\system}{\textsf{TRACE-RealWorld}}
\newcommand{\trw}{\textsf{TRW}}
\newif\ifextended
\ifdefined\submissiononly
  \extendedfalse
\else
  \extendedtrue
\fi
\newcommand{\extpointer}[1]{%
  \ifextended Appendix~\ref{#1}\else the extended report\fi}

\newtheorem{definition}{Definition}
\newtheorem{claimstmt}{Claim}
\newtheorem{theorem}{Theorem}
\newtheorem{lemma}{Lemma}
\newtheorem{proposition}{Proposition}
\newtheorem{corollary}{Corollary}

\usepackage{etoolbox}
\AtBeginEnvironment{tabular}{\small}

\title{TRW: TRACE-RealWorld \textemdash{} An Auditable Consistency
Contract for World Models as Materialized Views}
\author{Edward Y.\ Chang \\ Stanford University}

\begin{document}

\maketitle

\begin{abstract}
World models let agents plan against a predicted physical state,
but that state drifts as the world changes; re-observation is costly
and delayed, and repair can fail. We present \system{} (\trw{}),
to our knowledge the first commitment-level consistency contract
for world models. \trw{} treats predicted state as a materialized
view and a physical commitment as a read whose authorization can
expire. Typed, calibrated claims specify consequence-conditioned
freshness and priced verification. Adaptive refresh generalizes
dual-Kalman synchronization to consult the world when evidence
could change a decision, while dependency-scoped SagaLLM
compensation repairs reversible commitments invalidated after
authorization. We prove, under an event-aligned risk oracle and
explicit recovery-liveness assumptions, that synchronization and
compensation are each insufficient alone and that their composition
gives a conditional consistency guarantee. When those assumptions
are not discharged, the same argument yields an auditable
decomposition of violations into named debts rather than an
unconditional guarantee.
We implement \trw{} in Flood-SAR, a simulated search-and-rescue
workbench over real geography, and evaluate it through six
pre-registered questions with frozen operating points and held-out
seeds. Adaptive refresh reduces stale execution but does not
dominate fixed refresh on cost, coverage, or rescue outcomes.
In the composition campaign, localized repair reduces repair work
by 9.56 units per mission and restoration latency by 80.7 seconds
relative to global recovery, while their observed
residual-violation difference is zero without establishing
equivalence. Detection coverage is 0.83--0.89, and 10 of 97 invoked
restorations are incomplete by mission end: the campaign estimates
the theorem's empirical slack as sensitivity inputs rather than
claiming to discharge its assumptions or provide a simultaneous
population certificate. Exact replay reconstructs a disputed dispatch.
\trw{} thus makes a world model an auditable predictive interface
rather than a self-validating source of truth.
\end{abstract}

\section{Introduction}\label{sec:intro}

World models are the current bet for grounding agentic AI in
physical reality: an internal, digital replica of the world on which
an agent can plan, predict, and rehearse
interventions~\cite{lecun2022path,hafner2023dreamerv3,nvidia2025cosmos}.
The replica exists because the physical world does not permit
experimentation: no operator can flood a levee to see which routes
survive, so the do-operation of causal reasoning is affordable only
inside the model. This creates the discipline's least glamorous, most
consequential problem: the digital and physical worlds drift apart
the moment the agent stops looking. How should they be kept in sync,
when continuous synchronization means paying for a drone sortie or a
crew report at every tick, and updating every $k$ seconds is too slow
for the hour the levee breaks and wasteful for the days it holds?
We present \system{} (\trw{}), to our knowledge the first
commitment-level consistency contract for world models: typed
semantic claims, consequence-conditioned freshness, priced
heterogeneous verification, transactional repair, and replayable
provenance, bound together so that a commitment is authorized,
re-verified, or repaired under one auditable discipline. The theory
it establishes is that the deployment-critical property an
operational world model must control is consistency between the
model and the world it models, with two complementary halves: no refresh policy alone can
drive violations to zero at bounded price while meeting a service
obligation, no compensation policy alone can either, and their
composition yields a conditional guarantee when the oracle event,
recovery endpoint, and compensation interface align. Without those
conditions, the proof yields an accounting decomposition, with every
remaining assumption failure represented as a named, measurable
consistency debt,
the empirically instantiated inputs reported in
Section~\ref{sec:eval}
(Theorems~\ref{thm:no-sync} through~\ref{thm:composition},
Section~\ref{sec:theory}). The treatment is database discipline,
where the problem has a fifty-year-old
name~\cite{gray1993transaction,gupta1995maintenance,
huang1994divergence,olston2000trapp,olston2002best}: the predicted
world state is a materialized view over reality, an observation is a
partial refresh, a planned action is a read whose validity can expire
while the world changes, and when to look becomes a view-maintenance
policy that must answer for its choices.

The investment in the replica side is enormous: joint-embedding
architectures plan robot trajectories from raw
video~\cite{assran2025vjepa2}, latent-imagination agents master
control suites~\cite{hafner2023dreamerv3,hansen2022tdmpc}, and
generative environments~\cite{bruce2024genie} and physical-AI
foundation models~\cite{nvidia2025cosmos} anchor the roadmaps of
major laboratories. Nearly all of it targets predictive quality; the
data-management questions underneath, how fresh a representation must
be for a given use, at what verification cost, with what audit trail,
are largely unasked. This paper asks them with an abstraction this
community already owns, and proves why predictive quality cannot
answer them: within its declared dynamics, the stale-execution
floor of Theorem~\ref{thm:no-sync} is predictor-independent
(Corollary~\ref{cor:predictor}): conditional on the same observation
filtration, no richer predictor removes the final latency window,
so closing it is a contract obligation rather than a modeling one. The deployment-correctness question is not
whether the view is accurate at one instant, but whether a read
served from it remains justified while the base keeps changing, and
whether the system can show, after the fact, why it believed the
view fresh enough. Safety is the sharpest stake; a stale read also
breaks chained commitments, wastes resources on invalidated plans,
and blurs responsibility.

The problem recurs wherever agents commit against a drifting world;
disaster response, autonomous logistics, and power distribution all
share one structure: observation is priced, commitments are hard to
reverse, staleness propagates down dependency chains, and after an
incident someone must reconstruct what the system believed and why
it acted. For illustration and experiments we
use Flood-SAR, a search-and-rescue planning and reasoning workbench
we implemented over real geography, as the running example; the
geography, hydrology rasters, and network are real, while missions,
channels, failures, and outcomes are simulated. The contract
abstraction is not water-specific, though only Flood-SAR is
evaluated; Section~\ref{sec:discussion} discusses how far the
illustrations generalize.

What an application cares about lives above the sensor: not what
the water level is, but whether to send the boat through the channel
now or take the longer, safer route. The maintenance problem
decomposes into three questions: \emph{semantic elevation}, how raw
measurements become typed claims such as ``edge E12 impassable for
rescue boats'' with calibrated adequacy attached
(Section~\ref{sec:contract}); \emph{signal economy}, when the
system must look at the world at all and through which priced
channel (Section~\ref{sec:refresh}); and \emph{consequence
admission}, which claims at what freshness may authorize which class
of action (Section~\ref{sec:gate}).

\system{} answers all three. It extends TRACE's typed, versioned,
append-only records~\cite{chang2026trace} to a world that keeps
changing after the snapshot is taken: state, event, and practical
claims make the maintained object a calibrated, falsifiable claim
rather than a number. Each claim declares its own refresh condition,
four trigger families decide when the world must be consulted, and a
value-of-information rule decides through which channel,
generalizing dual-Kalman stream
management~\cite{jain2004adaptive} from query precision to
commitment precision; in the regime a model-conditional lemma
covers, the re-verification deadline is derived from a
consequence-scaled decision-error-loss budget rather than configured, and outside
it the deployed guards are declared and empirically validated. A
deterministic gate scales required freshness and confidence to the
action's reversibility and loss, and the decision itself is a
recorded, re-litigable object. The held-out result is a priced
trade-off, stated at the outset so the evaluation reads as designed
rather than discovered: adaptive refresh cuts stale executions
against the strongest fixed interval and verification spend against
the validity clock, and pays in coverage and completed opportunities
against that fixed interval. The contribution is not that adaptive
refresh wins; it is that the choice among refresh policies becomes
an explicit, priced, auditable decision instead of a configuration
accident.

\system{} also occupies a definite place in a reasoning stack,
sketched in \extpointer{app:figures}. Below it sits
TRACE~\cite{chang2026trace}, the memory facility: typed, versioned,
append-only records. Above it sits planning: schedules are chains of
commitments with transactional semantics in the sense of
SagaLLM~\cite{chang2025sagallm}, the planning layer of our broader
program on multi-LLM collaborative intelligence and System-2
reasoning~\cite{chang2024pathagi1,chang2026pathagi2}, and the
layers interact through conditions. Semantic elevation aggregates signals into conditions;
conditions license decisions; and when a condition changes during
the execution of a prior decision, the river closing after the boat
is dispatched, the decision must be invalidated, rolled back, or
compensated by the planning layer. \system{} is not a juxtaposition of synchronization and
compensation; it closes a loop that both lines of work, each ours,
left open. Synchronization can prevent some stale commitments but
cannot eliminate violations at bounded observation cost
(Theorem~\ref{thm:no-sync}); compensation can repair reversible
consequences but cannot undo irreversible actions
(Theorem~\ref{thm:no-comp}). \trw{} converts the two mechanisms,
two decades apart in our own work, the preventive half in
dual-Kalman stream management~\cite{jain2004adaptive} and the
corrective half in SagaLLM~\cite{chang2025sagallm}, into a single
commitment-level consistency contract; and the new object is the
thing being protected, the physical commitment itself, with its
gate, deadlines, reversibility classes, affected-set semantics,
evidence record, and measurable violation bound. This also delimits the boundary
of \system{}: it neither learns the predictor nor constructs the
plan. Its input is what a world model outputs, predicted state with
uncertainty, support, and out-of-distribution scores, and its output
is the set of claims those predictions justify, with calibrated
adequacy, validity intervals, and refresh conditions attached.

The composition, not any component, is the claim: \trw{} binds
predictive-view synchronization and dependency-scoped saga
compensation into one deployed, pre-registered contract. In
summary, \system{} makes the following contributions:
\begin{enumerate}[leftmargin=1.4em, itemsep=2pt, topsep=3pt]
\item a theory of digital-physical consistency over the
materialized-view formulation: insufficiency theorems for
synchronization alone and for compensation alone, a
filtration-conditional predictor-independence corollary, a
composition theorem stated at the interface, and a cost corollary
that predicts the measured trade-off regimes
(Sections~\ref{sec:contract} and~\ref{sec:theory});
\item the commitment-level contract and its mechanisms: typed claims
with declared refresh conditions, Kalman-gated adaptive refresh with
the model-conditional deadline lemma and scheduler separation, and
consequence-scaled gate semantics integrated with transactional
commitment chains (Sections~\ref{sec:refresh}
and~\ref{sec:gate});
\item the Flood-SAR workbench and its pre-registered evaluation:
no-refresh, fixed-interval, validity-clock, and adaptive refresh
under a declared-observables discipline, closed by the RQ6
composition campaign that crosses refresh policy with recovery
discipline to measure the contract end to end
(Sections~\ref{sec:floodsar} and~\ref{sec:eval}).
\end{enumerate}

\section{Related Work}\label{sec:related}

The term world model is currently overextended: four capabilities
travel under one phrase. A \emph{predictor} estimates the next
state; a \emph{simulator} generates candidate futures; a
\emph{digital twin} maintains an identity-linked replica
synchronized with a physical entity; and an \emph{operational
world model} additionally knows when its claims expire, when and
through which channel to re-observe, which actions its claims may
authorize, and how to recover when a late change proves them
wrong. Most prominent world models achieve the first two levels;
digital-twin architectures approach the third; \system{} supplies
the contract that defines the fourth, and the lineage below reads
against that distinction.

\paragraph{Synchronizing replicas, before the LLM era.}
Digital-physical synchronization itself is not new. Best-effort
cache synchronization refreshes when cached values diverge past a
precision bound~\cite{olston2002best}; adaptive filters allocate
update precision to query needs against communication
cost~\cite{olston2003adaptive}; and our dual-Kalman stream
architecture, the direct mechanism ancestor, contacts the source
only when innovation threatens declared query
precision~\cite{jain2004adaptive}. Divergence caching and
precision-constrained replication trade communication for bounded
approximation~\cite{huang1994divergence,olston2000trapp}, adaptive
view maintenance prices refresh against
benefit~\cite{gupta1995maintenance}, the digital-twin paradigm
posed continuous synchronization for flight
hardware~\cite{glaessgen2012twin}, and the Age of Incorrect
Information replaced raw staleness with downstream
incorrectness~\cite{maatouk2020aoii}. What all of these maintain
is a number, cache, estimate, or replica. \system{} maintains a
typed, calibrated claim that licenses an irreversible action,
which changes what freshness must mean and what a late failure
must trigger; and unlike an ordinary view, the base cannot be read
transactionally.

\paragraph{Event triggers and information value.}
Event- and self-triggered control replace periodic sampling with
state-dependent communication or
actuation~\cite{astrom1999comparison,astrom2002lebesgue,
tabuada2007event,heemels2012introduction}, and value of
information prices an observation by its expected decision
improvement~\cite{howard1966voi,kaelbling1998planning,krishnamurthy2016pomdp}.
Our trigger is not claimed as a new control law: the contribution
binds these ideas to typed claims, declared observables, a
deterministic consumer gate, and durable commitment records.

\paragraph{Modern world models, read by level.}
Latent-imagination and joint-embedding systems learn state for
control~\cite{hafner2023dreamerv3,hansen2022tdmpc,lecun2022path,
assran2025vjepa2}, and generative environments and physical-AI
foundation models extend the
agenda~\cite{bruce2024genie,nvidia2025cosmos}: strong predictors
and simulators, levels one and two, with no freshness, priced
re-observation, or commitment semantics. The learned line divides
by function: planning-oriented latent models roll compact state
forward~\cite{hafner2019planet,schrittwieser2020muzero,hansen2024tdmpc2},
an uncertainty-aware branch estimates where rollouts
degrade~\cite{chua2018pets}, with RWM-U propagating ensemble
epistemic uncertainty through long horizons on real
robots~\cite{li2025rwmu}, directly relevant to our support, OOD,
and uncertainty inputs but not synchronizing the model during
deployment, and a filtering branch learns the correction step
itself~\cite{revach2022kalmannet,karl2017dvbf}. At level three, a
data-centric digital-twin architecture aligns virtual and physical
entities through telemetry, analysis, and
actuation~\cite{digitaltwin2026datacentric}, without
decision-level freshness, priced channel choice, or transactional
repair. Closest to level four, sensing clocks convert a certified
world model's validity horizon into a proactive re-sensing
deadline~\cite{wang2026sensingclocks}: they answer when to sense,
while \system{} additionally asks through which channel, for which
commitment, at what consequence level, and what happens after a
late failure; and Pre-VLA verifies candidate actions before
physical execution~\cite{prevla2026}, a gate-shaped neighbor that
neither maintains changing evidence nor repairs dependent
commitments. Because confidence is part of the contract, uncertainty here is
not a penalty shaping a planner's objective but an admission
condition a deterministic gate checks and
records~\cite{guo2017calibration}.

\paragraph{Auditability and transactional plans.}
TRACE supplies typed, versioned reasoning
records~\cite{chang2026trace}; SagaLLM and classical sagas
structure dependent plans and
compensation~\cite{chang2025sagallm,garcia1987sagas}; Mnemosyne
governs episodic memory~\cite{chang2027mnemosyne}; the multi-agent
reasoning program these components serve is developed
in~\cite{chang2024pathagi1,chang2026pathagi2}; and runtime
verification and database reenactment reconstruct behavior from
recorded state~\cite{leucker2009brief,arab2018reenactment}. The
position, stated once: prior work synchronizes numerical replicas,
schedules sensing, propagates model uncertainty, or verifies
candidate actions; \system{} binds these functions at the
commitment level, where semantic freshness, consequence, channel
choice, transactional repair, and replayable provenance form one
consistency contract.

\section{The \system{} Contract}\label{sec:contract}

This section answers the semantic-elevation question: it fixes the
framing under which a world model is a maintainable object
(Section~\ref{sec:view}), then gives the record contract that
realizes the framing (Sections~\ref{sec:contract-claims}
and~\ref{sec:contract-observables}).

\subsection{World models as materialized views}\label{sec:view}

We state the framing once, precisely, and use it throughout.

\begin{definition}[World view]\label{def:view}
Let $S_t$ denote the state of the physical world at time $t$, observable
only through priced channels $\mathcal{C} = \{1,\dots,m\}$, where
channel $j$ returns $y \sim p_j(\cdot \mid S_t)$ at cost $c_j$. A world
view is a belief $b_t$ over $S_t$ maintained by a predictor: between
observations, $b_{t+1} = f(b_t, a_t)$ advances the view without reading
the base (incremental maintenance under a motion model); an observation
$y_t$ through any channel conditions the view (partial refresh).
\end{definition}

\begin{definition}[Commitment read]\label{def:read}
A commitment of class $c$ is a read against the view that authorizes a
durable action over horizon $h$. The class declares a precision
requirement: bounds on the view's calibrated uncertainty, support, and
out-of-distribution scores, and a maximum evidence age, that the read
must satisfy at authorization time and that the view must be expected
to satisfy over $[t, t+h]$.
\end{definition}

\begin{definition}[Refresh condition]\label{def:refresh}
A refresh condition is a triple
$\rho = (\omega, \theta, \mathcal{C}_\rho)$: a declared observable
$\omega$, a tolerance $\theta$, and a channel menu
$\mathcal{C}_\rho \subseteq \mathcal{C}$. The tuple parameterizes
the divergence clause: the condition fires when the declared
discrepancy $d_\omega$ between predicted and observed evidence on
$\omega$ exceeds $\theta$. The discrepancy is type-specific and the
claim declares it: for numerical observables
$d_\omega(u,v) = |u-v|$, absolute error; for categorical and
semantic claims, route-open predicates among them,
$d_\omega(u,v) = \mathbf{1}[u \neq v]$, so any $\theta < 1$ fires on
disagreement. Two
further clauses fire from fields of the claim and the pending
commitment rather than from $\rho$: when the view's predicted
adequacy for the pending commitment class falls below that class's
declared requirement, and when the claim's validity interval no
longer covers the commitment horizon.
\end{definition}

\begin{definition}[Decision margin]\label{def:margin}
For a pending commitment with admissible alternative actions $A$ and
loss $L$, the margin of belief $b$ is
$m(b) = \min_{a' \neq a^{*}_b} \mathbb{E}_{s\sim b}[L(a',s)]
- \mathbb{E}_{s\sim b}[L(a^{*}_b,s)]$, the expected-loss gap between
the chosen action and its best alternative, with $|A| \ge 2$ and
ties in $a^{*}_b$ broken by a fixed declared order on $A$. If only
one admissible action remains, $m(b) = +\infty$ by convention:
refresh relevance vanishes and the empty-alternative case is governed
by escalation, not refresh. The set of beliefs where the margin
vanishes is the boundary at which the optimal commitment switches.
\end{definition}

\begin{claimstmt}\label{claim:main}
Under this framing, the freshness-dependent safety question of
agentic world models is a view-maintenance problem with two properties not jointly represented
in classical caching: refreshes are priced physical acts with channel choice, and
maintenance decisions must themselves be auditable. \system{} is a
maintenance policy and an audit contract for exactly this setting.
\end{claimstmt}

Definition~\ref{def:margin} carries the paper's semantic weight:
refresh relevance is a property of the pending decision, belief,
loss, and available channel, not of signal variance alone. What
justifies polling the gauge is whether an observation could move
the belief across the boundary at which the optimal commitment
switches (Section~\ref{sec:refresh-rule}); the quantity maintained
is the justification of a commitment, and signal refresh is only
its means.

\subsection{World-state claims}\label{sec:contract-claims}

A world-state claim names its subject (a route, dock, asset, incident,
or schedule), a predicted value, the horizon over which it is intended
to hold, calibrated uncertainty, supporting evidence, and a refresh
condition. Its verdict is a record field drawn from the frozen TRACE
vocabulary (\texttt{accept}, \texttt{qualify}, \texttt{revise},
\texttt{defer}, \texttt{reject}, \fld{no_argument}); what a
consumer may do with the record is a separate decision (CLEAR, QUALIFY,
HOLD, BLOCK, ESCALATE) written to \fld{consumer_actions} after the
consumer reads it. The distinction is load-bearing for audit:
\texttt{defer} is a property of the claim's warrant, HOLD is a
plan-gate action, and collapsing them destroys the trail.

The base type is the frozen \texttt{TraceRecord} of TRACE schema
v1.0~\cite{chang2026trace}, with its typing, evidence, outcome, and
infrastructure fields; the schema is versioned, fields are added by
version and never silently mutated, and every consumer declares the
versions it reads. \system{} adds nothing to the frozen core. It uses
the schema's own extension path: the versioned experimental profile,
which already carries \fld{declared_observables} and
\fld{attribution_status}, gains two sibling fields:

\begin{quote}
\footnotesize
\begin{verbatim}
ExperimentalProfile = {
  declared_observables,
  attribution_status,
  refresh_condition,
  validity_interval}
\end{verbatim}
\end{quote}

The field \fld{refresh_condition} instantiates
Definition~\ref{def:refresh}; a \fld{validity_interval} bounds the
wall-clock span over which the predicted value is asserted, and a
commitment extending past \fld{valid_until}
cannot be cleared on that claim alone.

\emph{Predictor-version provenance.} Every claim derived from the
predictive view identifies the model that produced it
(\fld{predictor_version}, \fld{calibration_version}, timestamp,
claim family, adequacy status), and a predictor update is itself a
recorded structural event. The governing principle is that online
adaptation does not inherit authority: consequence-sensitive
commitments remain held until the new version satisfies the
applicable adequacy requirement, with what a superseded-version
claim may still authorize decided by the gate
(Section~\ref{sec:gate-conditions}).

\subsection{Refresh conditions are declared observables}\label{sec:contract-observables}

The unification that makes refresh auditable fits in one sentence:
an innovation tolerance is a pre-declared observable with a healthy
signature (innovation within tolerance, claim confirmed at passive
cost) and a pathological one (tolerance exceeded, claim revised
append-only), so a refresh decision is falsifiable and attributable
to its declared condition rather than an anonymous scheduler, and
the refresh policy stops being an operational knob and becomes an
auditable object.

\section{A Theory of Digital-Physical Consistency}\label{sec:theory}

The abstract stakes a claim that this section makes precise:
consistency between a world model and the world is the property to
guarantee, synchronization plus compensation yields a conditional
guarantee of it under named, measured assumptions, and neither
half alone can. We first say what consistency means for a view that
authorizes physical action, then name the assumptions under which the
guarantee is provable, and prove three results.
Theorem~\ref{thm:no-sync} shows that no refresh policy alone,
under a nonzero obligation to execute, drives consistency
violations to zero at bounded price, and that the obstruction is
independent of predictor quality. Theorem~\ref{thm:no-comp} shows
that compensation alone cannot either: it leaves the irreversible
floor untouched.
Theorem~\ref{thm:composition} proves endpoint consistency under an
event-aligned oracle, compositional rollback, dependency
completeness, and recovery liveness. Without those conditions it
gives a pathwise audit decomposition into additive, measurable
consistency debts. A corollary gives a fixed-deadline-sequence
counting comparison; it does not claim stochastic-policy dominance.
Full proofs follow the statements.

\subsection{Consistency and its price}\label{sec:theory-defs}

\begin{definition}[Point and extended stale execution]\label{def:stale}
Fix a commitment whose justifying claim asserts value $\hat\omega$
for its declared observable $\omega$ within tolerance $\theta$
(Definition~\ref{def:refresh}), and let $h_i\in\{\mathrm{point},
\mathrm{extended}\}$ be the commitment's declared hazard type.  A
point commitment has one execution instant $t_e$ and \emph{executes
stale} when
$d_\omega\!\left(\omega(S_{t_e}),\hat\omega\right)>\theta$.  An
extended commitment has a declared execution window $[t_b,t_e]$ and
executes stale when
\[
 Z_i^{\mathrm{ext}}
 =\mathbf 1\!\left\{\exists s\in[t_b,t_e]:
 d_\omega\!\left(\omega(S_s),\hat\omega\right)>\theta\right\}=1.
\]
Write $Z_i^{\mathrm{point}}$ for the corresponding terminal indicator
and $Z_i=Z_i^{h_i}$ for the indicator selected by the declaration.
The extended event contains the terminal event and can also record an
excursion that returns inside tolerance before $t_e$.  For numerical
observables the discrepancy is absolute error; the analytic model
below uses a one-sided, boundary-anchored specialization.  Flood-SAR
declares route traversals extended and instantaneous dispatch
authorizations point.  Staleness is a property of the realized world
path, not of the belief, and is adjudicated after the fact from the
record and the corresponding point or pathwise ground truth.
\end{definition}

\begin{definition}[Consistency violation]\label{def:violation}
A consistency violation is a commitment that executes stale under its
declared hazard type and
whose declared invariants have not been restored by mission end
through any registered corrective mechanism. Violations are the
composed system's failure unit: prevention failed before execution,
and every registered correction failed, or none applied, after it.
The definition is mechanism-neutral: localized saga repair and
global recovery are both registered mechanisms, so neither arm of
RQ6 wins by terminology.
\end{definition}

\begin{definition}[Two-tier consistency]\label{def:twotier}
A controller is \emph{consistent at level $\eta$ with recovery
endpoint $T$} over a declared mission class if (i), for every
pre-execution history on which an irreversible commitment executes,
its conditional stale-execution risk $\Pr(Z_i=1\mid\mathcal F)$ is at
most $\eta$ under the controller's observation filtration (execution
consistency; this is the flip risk only when the declared point or
interval stale event is aligned with that flip event), and (ii), almost
surely, every stale execution of a reversible commitment has its
declared invariants $I$ restored by $T$ through a registered
corrective mechanism, over the computed dependency-scoped affected set
(complete under A6-C, minimal under A6-S). Flood-SAR declares $T$
as mission end, matching the violation definition and observable
O3; restoration latency relative to detection is recorded and
reported (O3), but a latency-graded refinement is not part of the
theorem's endpoint
(eventual consistency). Here ``minimal'' means minimal at the level
of the commitment set: the computed set equals the true affected set.
It does not assert a minimum number of physical writes or minimum
repair cost.
\end{definition}

\begin{definition}[Consistency cost]\label{def:cost}
For a policy $\pi$ over a mission,
$J(\pi) = \mathbb{E}\!\left[\,C_{\mathrm{sync}} + C_{\mathrm{comp}}
+ L_{\mathrm{viol}}\,\right]$: the priced observations that fired,
the compensations executed (asset displacement, re-planning,
re-timing), and the losses of uncompensated stale executions. The
contract's obligation is not to minimize any single term; it is to
make each term attributable and, when the additional workload and
loss caps of Corollary~\ref{cor:price} are supplied, to bound $J$.
\end{definition}

Throughout, $N_{\mathrm{irr}}$ and $N_{\mathrm{rev}}$ denote the
numbers of irreversible and reversible commitments a mission
executes, classified by the gate's declared reversibility field.

\subsection{Assumptions}\label{sec:theory-assumptions}

Each assumption is named once, used by the theorems below, and
checked somewhere specific; none is hidden inside a mechanism.

\begin{description}[leftmargin=1.5em, itemsep=1pt, topsep=2pt]
\item[A1 (Semantic-risk transport, marginal).] Index the
gate-cleared irreversible commitments in a mission by $i$, let
$Z_i=Z_i^{h_i}$ be the indicator that commitment $i$ executes stale
under its declared point or extended hazard, and let
$\widehat r_i\in[0,1]$ be the semantic stale-risk score declared to the gate.
On the declared mission mixture, with the convention that the claim
is vacuous when $\mathbb E[N_{\mathrm{irr}}]=0$, calibration supplies
the exact transport inequality
\[
 \mathbb E\!\left[\sum_{i=1}^{N_{\mathrm{irr}}} Z_i\right]
 \leq
 \mathbb E\!\left[\sum_{i=1}^{N_{\mathrm{irr}}}\widehat r_i\right]
 + \mathbb E[N_{\mathrm{irr}}]\,\varepsilon_{\mathrm{cal}}.
\]
This is a claim-weighted marginal statement, not per-claim
conditional coverage. A binned ECE analysis licenses this
inequality only when its target event, bins, gate-cleared sampling
distribution, horizon, and consequence class coincide with the
quantities above; calibration of another label is not a substitute.
O6 is intended to estimate this gate-cleared transport. A
conditionally valid stale-risk oracle (Definition~\ref{def:oracle})
gives the stronger per-commitment statement and does not need A1 for
its own class.
\item[A2 (Bounded latency).] Sensing-to-authorization latency is at
least $\delta > 0$ on every channel, and invalidation detection
completes within $\delta_d$ of the invalidating event. The floor $\delta$ is a fact about
physical channels. The ceiling $\delta_d$ is a detection-latency
property, measured by RQ6 as the time from a registered world shock
to the detected invalidation, reported with a predeclared upper
confidence bound; the repository and gate overheads of
Section~\ref{sec:eval-rq4} are a component of that path, not a
substitute for the measurement.
\item[A3 (Compositional compensability and frame condition).]
Every reversible step $T_i$ registers a compensating transaction
$C_i$. For any finite affected set, successful execution of the
transactions in reverse commit order establishes the rollback
postcondition for the whole compensated suffix: after later
transactions have been compensated, $C_i$ restores the declared
invariants for the enlarged suffix and preserves every rollback
postcondition already established. Each $C_i$ writes only its
declared repair footprint, and the union of those footprints is the
declared affected slice. This compositional postcondition and frame
condition, rather than local invertibility alone, are the inductive
interface used below. Irreversible steps are gated, not compensated.
Enforced as a module contract; exercised by RQ6 (observable O3).
\item[A4 (Acyclic dependencies).] The dependency relation
$D=\{(o_i,o_j,c_{ij})\}$ of Section~\ref{sec:gate-chains} is a
finite directed acyclic graph, and every dependency edge points
from an earlier to a later commit, so reverse commit order is a
reverse topological order. Enforced by the scheduler's construction.
\item[A5 (Isolation).] Compensation executes with serial
equivalence over the affected set: no interleaved commit reads or
writes the set mid-repair. This is the exact module assumption used
by the proof; SagaLLM motivates the interface, but the present
theorem remains conditional on it~\cite{chang2025sagallm}.
\item[A6-C (Dependency completeness).] Every true dependency among
commitments appears in $D$, so the computed dependency-scoped
affected set contains every truly affected commitment. Completeness
is what restoration correctness requires: a missed dependent may
remain unrepaired and prevents a restoration guarantee. Measured
by RQ6's affected-set recall (O7).
\item[A6-S (Dependency soundness).] No spurious dependency appears
in $D$, so the computed set contains nothing the revision leaves
untouched. Soundness is what minimality and repair economy require;
restoration correctness does not need it. Measured by RQ6's
affected-set precision (O7). Both halves are declared architectural
invariants, tested rather than assumed: services exchange reasoning
state only through records, every commitment cites its justifying
claims, and citation-closure checks run in the contract tests; the
campaigns measure the invariants on the incident stream and do not
establish them generally. We write A6 for the conjunction of
A6-C and A6-S.
\item[A7 (Recovery-endpoint liveness).] For a theorem-covered
endpoint $T$, every invalidation whose restoration is claimed occurs
no later than $T-\delta_d-\delta_r$ for a declared $\delta_r>0$;
detection within $\delta_d$ atomically enqueues exactly one
restoration invocation for every member of the computed affected
set, and every such traversal whose compensations succeed completes
within $\delta_r$.
This is the liveness and terminal-grace condition needed to conclude
restoration by $T$. Without A7, the proof establishes rollback
correctness conditional on completion, while incompletion by $T$
belongs to the restoration-failure debt.
\end{description}

\subsection{Neither half suffices}\label{sec:theory-impossibility}

\paragraph{Common setup.} The substate follows $dx_t = \sigma\,dW_t$
with $\sigma > 0$. In this analytic subsection observations are
exact samples of $x$. Let $\mathcal H_u$ be the sigma-field generated
by samples with sampling times at most $u$, augmented by policy
randomization independent of $W$. Every observation channel has cost
at least $c_{\min} > 0$ and latency at least $\delta > 0$ (A2), so the
controller's decision filtration is
$\mathcal{F}^{\mathrm{obs}}_t=\mathcal H_{t-\delta}$: a sample
contributes to a decision at $t$ only if taken by $t-\delta$.
Observation epochs are stopping times, and an execution epoch after
the last contributing sample is measurable with respect to the
last-sample sigma-field together with exogenous demand and independent
policy randomization. These conditions license the strong Markov
step used below. A refresh policy is any channel-selection and timing
rule adapted to $\mathcal{F}^{\mathrm{obs}}$. For
Theorem~\ref{thm:no-sync}(ii) we additionally use the demand
condition stated there: commitments may require execution at any
time in the horizon, so the policy cannot concentrate its
observations around known execution epochs. Throughout,
$\Phi$ is the standard normal distribution function and
$z_p = \Phi^{-1}(1-p)$.

Throughout the analytic model, an irreversible commitment's
declared tolerance is \emph{boundary-anchored}: the claim asserts
that the state lies on the safe side of the action's decision
boundary at observed distance $d$, so its stale event
(Definition~\ref{def:stale}) is the one-sided terminal exceedance
toward the boundary and coincides with the action-flip event,
matching the gate's flip-risk semantics. A two-sided tolerance band
would add a safe-side exceedance that flips no action; no result
below bounds that term.

\begin{theorem}[Insufficiency of synchronization]\label{thm:no-sync}
Under the exact-observation and stopping-time setup above, let a
substate follow $dx_t = \sigma\, dW_t$ with $\sigma > 0$, let
every observation channel have cost at least $c_{\min} > 0$ and
latency at least $\delta > 0$ (A2), and consider a point commitment whose last
usable sample found boundary distance $d > 0$. Then, for
any refresh-only policy: (i) conditional on executing the
commitment without evidence newer than $t_e - \delta$, the
stale-execution probability is at least
$p_\delta = \Phi\!\left(-d/(\sigma\sqrt{\delta})\right) > 0$,
regardless of expenditure, and if the policy must execute with
probability at least $s > 0$ on this fixed-distance stratum, its
unconditional stale-execution probability is at least
$s\, p_\delta$; and (ii) relaxing the latency floor to
zero (the case most favorable to refresh), consider a service
stratum in which every authorization's observed boundary distance
is at most $d$, commitments may require execution at any time in a
horizon $H$ with demand instants not signaled in advance (the
policy cannot co-locate an observation with an execution), and the
service must remain available throughout; sustaining stale-execution probability at most $p \in (0, \tfrac12)$ on
this stratum costs at least
$c_{\min} H \sigma^2 z_p^2 / d^{\,2}$ with
$z_p = \Phi^{-1}(1-p)$, which diverges as $p \to 0$. No refresh-only
policy under a nonzero service obligation can therefore drive
stale-execution probability to zero at bounded observation cost;
abstention is not a loophole but the gate's own third option, HOLD,
priced as lost coverage.
\end{theorem}

\begin{proof}[Proof of Theorem~\ref{thm:no-sync}]
(i) Fix a commitment executed at $t_e$ and let $t_o \le t_e - \delta$
be the sampling time of the last exact sample that contributes to
$\mathcal{F}^{\mathrm{obs}}_{t_e}$, with boundary distance $d > 0$
on the safe side. Because no later sample contributes,
$\mathcal{F}^{\mathrm{obs}}_{t_e}=\mathcal H_{t_o}$ up to null sets.
By the strong Markov property at the observation stopping time $t_o$,
conditional on $\mathcal H_{t_o}$ and the measurable duration
$t_e-t_o$, the increment $x_{t_e}-x_{t_o}$ is Gaussian with mean
zero and variance $\sigma^2(t_e-t_o)$. Exogenous demand and policy
randomization are independent of $W$, so further conditioning on the
decision to execute does not tilt this law. The stale event of
Definition~\ref{def:stale} is
the terminal-exceedance event, the increment leaving the state past
the boundary at execution time (an endpoint statement; no
first-passage claim is made), whose conditional probability is
therefore
\[
\Pr\!\left(\text{stale} \mid \mathcal{F}^{\mathrm{obs}}_{t_e}\right)
= \Phi\!\left(\frac{-d}{\sigma\sqrt{t_e - t_o}}\right)
\ \ge\ \Phi\!\left(\frac{-d}{\sigma\sqrt{\delta}}\right)
= p_\delta,
\]
using $t_e - t_o \ge \delta$ and monotonicity of
$u \mapsto \Phi(-d/(\sigma\sqrt{u}))$. This is the conditional
statement: conditional on executing at $t_e$ with no evidence newer
than $t_e - \delta$ and last observed distance $d$, the
stale-execution probability is at least $p_\delta$. If the policy
must execute with probability at least $s > 0$ on this
fixed-distance stratum, the law of total probability gives an
unconditional stale-execution probability of at least
$s\, p_\delta$. Additional expenditure buys observations no later
than $t_e - \delta$: it can change the observed distance $d$, but
it cannot shrink the final window below $\delta$. Since $p_\delta$
does not depend on the policy, no refresh-only policy escapes the
conditional floor, and none under the service obligation escapes
the unconditional one.

(ii) Set $\delta = 0$ and fix a target $p \in (0, \tfrac12)$. At an
execution whose staleness age is $s$, the observed boundary
distance is at most $d$ by the stratum hypothesis, so the stale
probability is at least $\Phi(-d/(\sigma\sqrt{s}))$ by the identity
of part (i); keeping it at most $p$ therefore requires
$d/(\sigma\sqrt{s}) \ge z_p$, that is
$s \le k_p := d^{\,2}/(\sigma^2 z_p^2)$. Under the demand condition
a commitment may require execution at any $t \in [0, H]$, the
service must remain available throughout, and demand instants are
not signaled, so the policy cannot concentrate observations at
execution epochs and must keep the observation age at most $k_p$
at every time in the horizon; any policy doing so issues at least
$\lceil H / k_p \rceil$ observations, each costing at least
$c_{\min}$, for a total of at least
$c_{\min} H \sigma^2 z_p^2 / d^{\,2}$. As $p \to 0$,
$z_p \to \infty$ and the bound diverges. The conclusion is
conditional on execution: a policy that abstains has no stale
executions and no violations, at the price of the declared service
obligation; under any nonzero obligation to execute, zero
stale-execution probability at bounded observation cost is
unattainable by refresh alone.
\end{proof}

\begin{corollary}[Predictor independence, filtration-conditional]
\label{cor:predictor}
Conditional on the same exact-sample filtration and on the
independent-increment dynamics of Theorem~\ref{thm:no-sync}, the
floor $p_\delta$ is predictor-independent: the terminal-exceedance
event is
driven by increments no functional of the filtration can access, so
every predictor faces the same floor at the same observed distance.
A predictor with genuinely richer inputs, upstream gauges,
forecasts, spatial correlates, changes the filtration itself and
can improve the observed-distance distribution; what it cannot do
is remove the final latency window on the channels it actually
has. Within that scope, synchronization is a contract obligation of
the channel, not a capability gap a larger encoder closes.
\end{corollary}

\begin{proof}[Proof of Corollary~\ref{cor:predictor}]
The terminal-exceedance event in part (i) is a function of the
increment over
$(t_o, t_e]$, whose conditional law given $\mathcal H_{t_o}$ and
the measurable duration is supplied by the strong Markov property.
Any predictor computes its
belief as a functional of the observation history, its priors, and
internal randomness independent of $W$,
so for every predictor the conditional law of the increment given
that history is the same, and the conditional stale probability
given the last observed distance $d$ is
$\Phi(-d/(\sigma\sqrt{t_e - t_o}))$ identically. The floor
$p_\delta$ therefore binds uniformly over predictor classes. A
better predictor can influence which commitments are attempted and
at what observed distances; it cannot alter the floor conditional
on the distance, which is the sense in which synchronization is an
obligation of the channel rather than a property of the model.

\end{proof}

\paragraph{Extended commitments.}
Theorem~\ref{thm:no-sync} and Corollary~\ref{cor:predictor} evaluate
staleness at the execution instant, which is the correct hazard for a
commitment whose action is instantaneous. A commitment whose safety
requires its claim to hold for the duration of the action faces the
interval hazard instead; under the same one-sided Brownian horizon its
first-passage floor is exactly twice the terminal floor:
Lemma~\ref{lem:sustained} states the corresponding deadline and
Section~\ref{sec:refresh-theorem} prices it. The qualitative
conclusion is unchanged and the constants are strictly worse.

\begin{theorem}[Insufficiency of compensation]\label{thm:no-comp}
Under the setting of Theorem~\ref{thm:no-sync}, consider any policy
whose only consistency instrument is correction. Then: (i)
correction cannot change the stale execution of an already executed
irreversible commitment: under the model's exogenous observable
dynamics and an execution schedule held fixed across recovery
layers (the counterfactual O1 declares), the executed-stale count
is invariant to the recovery layer, and each executed irreversible
commitment at
observed boundary distance $d_i$ retains the floor
$p_{\delta,i}$, so
$\mathbb E[V]\geq\mathbb E[\sum_i p_{\delta,i}]$, where the sum is
over executed irreversible commitments, and no compensation term
reduces it,
because compensation acts after execution and irreversible actions
register no $C_i$ (A3). (ii) For a fixed reversible commitment
exposed over a single inter-observation interval from the same
observed boundary distance, the probability that its claim is
outside its declared tolerance at the check that ends the
interval, and hence is registered there as an invalidation carrying a repair
obligation for its committed dependents, is nondecreasing in the
interval length $k$; no bound on executed compensation cost is
claimed, because an obligation becomes an executed repair only
through the detection and invocation machinery that
Theorem~\ref{thm:composition}(iii) prices separately. Correction
alone therefore leaves the irreversible floor untouched and,
interval for interval, converts the price of consistency into
repair obligation rather than into a guarantee.
\end{theorem}

\begin{proof}[Proof of Theorem~\ref{thm:no-comp}]
(i) The stale event of Definition~\ref{def:stale} is determined by
the world state and the claim at execution time, both of which are
fixed before any compensation acts; with the model's observables
exogenous to compensation and the execution schedule held fixed
across recovery layers, removing or installing the compensation
layer therefore changes no executed-stale count. (RQ6
declares this invariance as observable O1's healthy signature, so
the proofs and the campaign test the same identity.) Irreversible
commitments register no compensating transaction (A3), so an
executed-stale irreversible commitment is uncompensated at mission
end and is a violation by Definition~\ref{def:violation}. Its
probability is at least $p_{\delta,i}$, the floor of
Theorem~\ref{thm:no-sync}(i) at commitment $i$'s own observed
boundary distance. Let $\mathcal G_i$ be the pre-execution
sigma-field containing the execution decision and distance $d_i$,
but not the post-execution outcome. The executed-index indicator and
$p_{\delta,i}$ are $\mathcal G_i$-measurable, and
Theorem~\ref{thm:no-sync}(i) gives the conditional floor on that
sigma-field. The tower property and linearity of expectation give
$\mathbb{E}[V] \ge \mathbb{E}\!\left[\sum_i p_{\delta,i}\right]$,
with the conditional reading $\sum_i p_{\delta,i}$ given the
pre-execution variables, hence the coarser conditional bound
$N_{\mathrm{irr}} \min_i p_{\delta,i}$.

(ii) Let $E_k$ be the event that the claim's observable lies
outside its declared tolerance band at the check ending an interval
of width $k$: the terminal-exceedance event the checker actually
evaluates, not a first-passage event. With the claim's asserted
value reset by an exact observation at the interval-opening check,
so the discrepancy starts the interval at zero, for zero-drift
Brownian motion the terminal discrepancy is distributed as
$\sigma\sqrt{k}\,|Z|$ with $Z$ standard normal, stochastically
increasing in $k$, so $\Pr(E_k)$ is nondecreasing in $k$; and
because the terminal variance grows without bound,
$\Pr(E_k) \to 1$ as $k \to \infty$. A registered terminal
exceedance at a check is exactly an invalidation, so for the fixed
commitment of the statement $\Pr(E_k)$ is the probability of
carrying a repair obligation at the check. The result deliberately
stops at the obligation: an excursion that returns inside tolerance
before the check is never registered, and whether a registered
obligation becomes an executed, costed repair depends on the
detection and invocation machinery that
Theorem~\ref{thm:composition}(iii) prices separately, so no bound
on executed compensation cost is claimed. The claim is per
commitment and per interval: aggregating to a mission-level load
would require workload assumptions the statement does not make, on
how many commitments are exposed, over how many intervals, and at
what observed distances, so no mission-level monotonicity is
asserted. Correction alone therefore leaves the irreversible floor
of part (i) untouched and, interval for interval, converts the
price of consistency into repair obligation rather than into a
guarantee.
\end{proof}

\subsection{The composition theorem}\label{sec:theory-composition}

The theorem is stated at an interface, so the division of labor
between the analytic lemma and the deployed system stays visible.

\begin{definition}[Deadline oracle and validity modes]\label{def:oracle}
At a verification epoch $t_i$, a \emph{deadline oracle} for an
irreversible commitment returns a declared semantic stale-risk path
$\widehat r_i(t)$, a deadline $\tau_i>0$, and a hazard-indexed risk
budget $\eta_i^{h_i}\in(0,1)$.  The budget is paired with the stale
event selected by the commitment's declared hazard type: for a point
commitment it is a point-event budget (equal to
$\eta(a_i)\in\{\eta_f,\eta_s\}$ in the two-action model of
Lemma~\ref{thm:deadline}); for an extended commitment it is a separately
declared interval-risk budget (denoted $\eta_i^{\mathrm{sus}}$ in
Lemma~\ref{lem:sustained}).  The oracle satisfies
$\widehat r_i(t)\leq\eta_i^{h_i}$ for
$t\in[t_i,t_i+\tau_i]$.  The gate is \emph{deadline-valid} when a
point execution instant, or the whole declared window of an extended
commitment, lies inside that interval unless a new verification
renews it.  The oracle is \emph{conditionally stale-valid} when,
for the true dynamics and every execution time in the interval,
\[
 \Pr(Z_i^{h_i}(t)=1\mid\mathcal F^{\mathrm{obs}}_{t_i},
                 \text{execute at }t)
 \leq \eta_i^{h_i},
\]
where $Z_i^{h_i}(t)$ is the indicator selected by
Definition~\ref{def:stale}: terminal exceedance for a point commitment
and interval exceedance over its declared execution window for an
extended commitment.  A score oracle that is only
deadline-valid supplies instead the marginal guarantee of A1.  A
flip-risk oracle is conditionally stale-valid only when the declared
stale event is contained in the flip event.  The boundary-anchored
analytic model makes the events equal for point commitments;
Lemma~\ref{lem:sustained} supplies the corresponding first-passage
oracle for an extended commitment under the same Brownian assumptions.
\end{definition}

Lemma~\ref{thm:deadline} (Section~\ref{sec:refresh-theorem})
exhibits a conditionally stale-valid oracle in closed form,
$\tau_i = d_i^{\,2}/(\sigma^2 z_{\eta(a_i)}^2)$, under zero-drift
Brownian dynamics, exact observations, its two-action loss, and the
boundary-alignment declaration preceding Theorem~\ref{thm:no-sync}.
Lemma~\ref{lem:sustained} supplies the extended-event analogue for a
separately declared interval-risk budget; linking that budget to loss
control requires a duration-wide loss declaration.
Outside that regime the deployed guards and declared thresholds of
Section~\ref{sec:refresh-rule} implement a score oracle; conditional
stale-validity must be established separately, or A1 supplies only a
mixture-level transport.  The selective-calibration analysis (O6,
Section~\ref{sec:eval-rq6}) is an empirical candidate for that
transport. Its reported upper limits may be substituted into a
population inequality only on their joint coverage event, as stated
after Theorem~\ref{thm:composition}.

\begin{theorem}[Conditional composition correctness and audit decomposition]
\label{thm:composition}
For missions drawn from the declared mixture, the composed contract
has the following properties, with assumptions stated separately for
each part.

(i) \emph{Irreversible execution.} Gate enforcement and a
deadline-valid oracle imply deterministically that every irreversible
point execution instant or extended execution window lies inside its
certified interval. If the oracle is
conditionally stale-valid, every such execution has conditional
stale risk at most $\eta_i^{h_i}$ and satisfies tier (i) of
Definition~\ref{def:twotier}. If only the declared score is available,
the gate condition $\widehat r_i\leq\eta_i^{h_i}$ together with A1 gives
the weaker mixture-level inequality
\[
 \mathbb E\!\left[\sum_{i=1}^{N_{\mathrm{irr}}}Z_i\right]
 \leq \mathbb E[N_{\mathrm{irr}}]
       (\eta+\varepsilon_{\mathrm{cal}}),
\]
where $Z_i=Z_i^{h_i}$ is selected by the commitment's declared hazard
type and $\eta=\sup_i\eta_i^{h_i}$. This marginal conclusion is not a
per-commitment conditional-coverage claim.

(ii) \emph{Reversible endpoint consistency.} Under A2--A5, A6-C,
and A7, if every invoked compensation succeeds, every covered stale
reversible commitment has its declared invariants restored by $T$.
Under A6-S as well, the computed commitment set equals the true
affected set and repair is commitment-set minimal. Without A7 the
same safety argument establishes the rollback postcondition only
conditional on completion, not completion by $T$.

(iii) \emph{Audit decomposition.} Without assuming A2, A6-C, or A7,
let $M_{\mathrm{det}}$ count reversible violations whose invalidation
is undetected by $T$; $M_{\mathrm{dep}}$ those detected by $T$ whose
commitment the computed affected set misses; and
$M_{\mathrm{inv}}$ those detected and included for which no
restoration invocation is scheduled. Let $F_{\mathrm{restore}}$
count invoked restorations that do not restore their target by $T$.
If each invocation targets one commitment, then the pathwise
decomposition gives
\[
 V \leq \sum_{i=1}^{N_{\mathrm{irr}}} Z_i
       +M_{\mathrm{det}}+M_{\mathrm{dep}}+M_{\mathrm{inv}}
       +F_{\mathrm{restore}}.
\]
Consequently, under A1 and the gate score bound,
\[
\mathbb{E}[V] \;\le\; \mathbb{E}[N_{\mathrm{irr}}]
\big(\eta + \varepsilon_{\mathrm{cal}}\big)
+ \mathbb{E}[M_{\mathrm{det}}]
+ \mathbb{E}[M_{\mathrm{dep}}]
+ \mathbb{E}[M_{\mathrm{inv}}]
+ \mathbb{E}[N_{\mathrm{rep}}]\, f_{\mathrm{restore}},
\]
where, when $\mathbb E[N_{\mathrm{rep}}]>0$,
$f_{\mathrm{restore}}=
\mathbb E[F_{\mathrm{restore}}]/\mathbb E[N_{\mathrm{rep}}]$,
and $f_{\mathrm{restore}}=0$ by convention otherwise. The middle
three terms are accounting cells, not ex ante guarantees; the bound
becomes predictive only when they are zero by verified assumptions or
are bounded independently of the violations being predicted.
\end{theorem}

\begin{proof}[Proof of Theorem~\ref{thm:composition}]
(i) The gate refuses to clear an irreversible
commitment in exactly three staleness situations: the claim's
validity interval fails to cover the commitment horizon, the
evidence age exceeds the declared bound, or the admissibility
deadline has lapsed without evidence adequate under the refresh
rule. The guarded scheduler ends every lapsed authorization in the
safe alternative, HOLD, or ESCALATE, and its tick guard handles the
Zeno case: when the theoretical deadline falls below one tick, no
authorization is issued at all, so neither a point execution instant
nor an extended execution window escapes a certified interval in that
case. This proves deadline validity.
If the oracle is conditionally stale-valid,
Definition~\ref{def:oracle} directly gives the per-execution bound.
For a score oracle, the gate gives $\widehat r_i\leq\eta$ for every
executed irreversible commitment; summing and applying A1 yields
$\mathbb E[\sum_i Z_i]\leq
\mathbb E[\sum_i\widehat r_i]
+\mathbb E[N_{\mathrm{irr}}]\varepsilon_{\mathrm{cal}}
\leq\mathbb E[N_{\mathrm{irr}}](\eta+
\varepsilon_{\mathrm{cal}})$. No conditional statement is inferred
from this marginal inequality.

(ii) Let $r$ be the revised claim and let the affected set $A(r)$
be the set of committed steps whose justification cites $r$
transitively under $D$; this is what the dependency traversal
computes. A6-C makes the computed set contain every truly affected
commitment, which restoration correctness rests on; A6-S makes it
contain nothing else, which minimality rests on. Minimality first:
under A6-S a commitment outside $A(r)$ cites no invalidated claim,
so its authorization is untouched; A3's frame condition confines
writes to the declared footprints of commitments in $A(r)$. Under
A6-C no truly dependent commitment escapes the traversal.
By A4 the reverse commit order over the finite set $A(r)$ is a
well-founded reverse topological order. The latest committed member
has no later member, so A3's suffix postcondition holds after its
successful compensation. Inductively, suppose the rollback
postcondition holds for the already compensated suffix. A3 states
that the next successful $C_i$ preserves that postcondition and
establishes it for the enlarged suffix. Thus, after all invocations
succeed, the rollback postcondition and $I$ hold over the whole
affected set. A5 excludes interleaved commits that could read or
write the set mid-repair, so no restored invariant is invalidated
during the traversal. A2 supplies detection within $\delta_d$; A7
enqueues every computed member and supplies both the terminal grace
and the $\delta_r$ completion bound, so the successful traversal
ends by $T$. Under A6-C all
truly affected commitments are restored; under A6-S the computed
set contains no unaffected commitment, giving commitment-set
minimality.

(iii) A violation is an executed-stale commitment whose declared
invariants no registered mechanism restores by $T$. Condition on a
mission and partition by class. An irreversible stale execution
admits no restoring mechanism, so it is a violation exactly when
$Z_i=1$, giving the first term pathwise. A
reversible violation is classified, exhaustively and disjointly, by
the first stage that failed on its path from staleness to
restoration: either its invalidation was not detected by $T$
(counted by $M_{\mathrm{det}}$); or it was detected but the
commitment lay outside the computed affected set (counted by
$M_{\mathrm{dep}}$); or it was detected and the computed affected
set contains the commitment, but no restoration invocation was ever
scheduled (counted by $M_{\mathrm{inv}}$); or a restoration was
invoked for it and failed to restore its target by $T$. These four
cells exhaust the reversible violations by construction. The
invocation invariant (one restoration invocation enqueued
atomically per member of the computed affected set, checked by the
workbench's contract tests) predicts that the third cell is empty;
the frozen per-violation classifier reports any occupant it
nevertheless finds, surfacing under its v1 cells as
\textsc{unclassified}, which flags the identity rather than
silently absorbing the case. For the
fourth cell, each invoked restoration targets exactly one
commitment
(one compensating transaction per saga step). A commitment in this
cell has at least one invocation and, being unrestored, all of its
endpoint-scoped invocations failed, including any retries; selecting one failed
invocation per commitment gives an injection from cell members to
failed invocations, since invocations of distinct commitments are
distinct. The cell's count is therefore at most
$F_{\mathrm{restore}}$, the number of failed invocations, with
retries only enlarging the right-hand side. This proves the pathwise
decomposition without A2, A6-C, or A7. A2 together with A7 empties
the first cell for theorem-covered invalidations; A6-C empties the
second, and A7's invocation invariant empties the third. Taking
expectations, applying part (i)'s A1 inequality, and
using
$\mathbb{E}[F_{\mathrm{restore}}] =
\mathbb{E}[N_{\mathrm{rep}}]\, f_{\mathrm{restore}}$ gives the
displayed expectation bound.
\end{proof}

\paragraph{Statistical instantiation.}
Theorem~\ref{thm:composition} is deterministic in the population
quantities $\varepsilon_{\mathrm{cal}}$ and
$f_{\mathrm{restore}}$. Let random upper limits
$\varepsilon_{\mathrm{cal}}^{U}$ and
$f_{\mathrm{restore}}^{U}$ be computed from an estimation sample.
The first limit must bound A1's risk-transport slack after the recorded
scores $\widehat r_i$ are subtracted; an upper limit on the stale-event
rate alone is not an estimate of calibration slack, although it can be
a vacuous conservative substitute because risk scores are nonnegative.
They may be substituted simultaneously only on a coverage event
\[
 \Pr\!\left(
 \varepsilon_{\mathrm{cal}}\leq\varepsilon_{\mathrm{cal}}^{U},
 \ f_{\mathrm{restore}}\leq f_{\mathrm{restore}}^{U}
 \right)\geq 1-\alpha .
\]
On that event the corresponding plug-in violation inequality holds
for a new mission from the declared mixture. Invocation-level Wilson
limits are not cluster-valid when invocations share missions; absent
a mission-cluster-valid simultaneous construction, the reported
upper limits are conservative sensitivity inputs, not a nominal
$1-\alpha$ theorem certificate.

\begin{corollary}[Finite accounting and fixed-sequence regimes]\label{cor:price}
Under the composed contract, a finite workload, finite per-check and
per-repair costs, the rate guard of Section~\ref{sec:refresh-theorem},
a bounded retry policy, and a declared per-violation loss cap, the
three terms of $J$ are finite, observable, and separately
attributable. Theorem~\ref{thm:composition}(iii) supplies the audit
decomposition behind $L_{\mathrm{viol}}$; it supplies a predictive
numerical bound only when the debt cells are independently bounded.
Proposition~\ref{prop:sched} supplies a check-count comparison for
$C_{\mathrm{sync}}$ (a spend comparison when per-check prices are
comparable), and no standalone bound for $C_{\mathrm{comp}}$. $C_{\mathrm{sync}}$ is the
deadline scheduler's $N$ checks priced by channel choice, and by
Proposition~\ref{prop:sched} any uniform scheduler honoring the
same declared deadline sequence uses at least the closed-span factor
$\bar\tau/\tau_{\min}$ under the proposition's counting convention;
$C_{\mathrm{comp}}$ follows the
invalidation rate, and prevention does not simply lower it: a
tighter tolerance detects more exits, raising invalidations while
lowering stale executions, so prevention converts undetected
staleness into detected, repairable invalidations; a quantitative
bound on $C_{\mathrm{comp}}$ would require workload and
repair-frequency assumptions the theorems do not make;
$L_{\mathrm{viol}}$ inherits the loss-capped version of
Theorem~\ref{thm:composition}(iii). The comparison with fixed
spacing is regime-dependent, and the governing quantity is deadline
heterogeneity: when $\bar\tau/\tau_{\min}$ is near one, a
well-chosen fixed interval honors the same deadlines at comparable
price and can win on opportunity cost; when deadlines are
heterogeneous, the adaptive advantage grows with the ratio
$\bar\tau/\tau_{\min}$ for this fixed-sequence counting comparison;
no stochastic-policy dominance follows. The reversals measured in RQ1 are
therefore
consistent with the theory's predicted regimes rather than
embarrassments to it.
\end{corollary}

\begin{proof}[Proof of Corollary~\ref{cor:price}]
The $C_{\mathrm{sync}}$ claim is Proposition~\ref{prop:sched}
verbatim: over a span $S_N$ of declared deadlines the deadline
scheduler issues $N$ checks, while any uniform scheduler honoring
the same deadlines must use spacing $k \le \tau_{\min}$ and, under
the proposition's closed-span convention, issues at least
$\lceil S_N/\tau_{\min}\rceil$ checks. The ratio is therefore at
least $\bar\tau/\tau_{\min}$, with equality permitted. For
$C_{\mathrm{comp}}$: an invalidation requires the observable to
exit its declared tolerance between consecutive certified epochs.
Tightening the tolerance does not lower this rate; it raises it,
because more exits are detected, while lowering the stale-execution
rate. The prevention-side control on total repair is therefore not
the tolerance but the validity horizons and deadlines, which limit
how long a commitment stays exposed; a quantitative bound on
accumulated repair would require workload and repair-frequency
assumptions not made here, so the corollary claims finiteness and
attribution for $C_{\mathrm{comp}}$, not a bound; prevention
converts undetected staleness into detected, repairable
invalidations rather than reducing correction to zero.
$L_{\mathrm{viol}}$ is bounded by
Theorem~\ref{thm:composition}(iii). For the regimes: if
$\bar\tau/\tau_{\min} = 1 + o(1)$, a uniform scheduler at
$k = \tau_{\min}$ issues $N(1 + o(1))$ checks and matches the
deadline scheduler's price, so the remaining compensation and
violation terms, together with separately reported opportunity cost,
decide the operating point; if
deadlines are heterogeneous, the fixed-sequence counting lower bound
$\bar\tau/\tau_{\min}$ exceeds one. This comparison does not extend
to policies whose observations generate different deadline
sequences.
\end{proof}

\paragraph{Scope.} The proofs establish exactly what the theorem
statements claim and nothing further. They do not establish
refresh-policy dominance (Corollary~\ref{cor:price} predicts
regimes in which fixed spacing wins), do not re-derive SagaLLM's
internal guarantees (A3 and A5 import them as cited machinery), and
are conditional where stated: on the analytic model where
Lemma~\ref{thm:deadline} supplies an event-aligned oracle, on
conditional stale-validity or A1's exact marginal transport where
deployed guards supply only scores, on A3's compositional rollback
contract, on A7 for an endpoint guarantee, and on the population
restoration ratio; empirical upper limits enter only on their joint
coverage event. Drift,
passive observation channels, and misspecified dynamics require
separate analysis, as declared in \extpointer{app:formal}.

\paragraph{What the theory claims.} The theorems bound violations
under named assumptions; they do not claim refresh dominance
(Corollary~\ref{cor:price} predicts where fixed intervals win), do
not re-prove SagaLLM's guarantees, and are conditional exactly
where stated, on the analytic model or a deployed system supplying
an event-aligned oracle, on A1's exact marginal transport, on
endpoint liveness, and on population debt bounds.
The interface is the contribution: any world model, JEPA-class
included, plugs in by supplying calibrated claims and a valid
deadline oracle, and its obligations are stated rather than
implied. Part (iii) goes one step further: every proof-to-system
gap becomes an additive, measurable consistency debt. This is an
audit decomposition; it becomes a predictive bound for the observed
implementation only when those debts are independently bounded.

\section{The Flood-SAR System}\label{sec:floodsar}

We
use Flood-SAR, a simulated search-and-rescue planning and reasoning
workbench we implemented over real geography, as the running
example. Flood-SAR simulates drones, boats, ambulances, ground
teams, and dynamic waterways over a real geography asset; the
geography and hydrology are real, the missions, channels, and
failures are simulated.
The contract abstraction is not
water-specific, though only Flood-SAR is evaluated: wildfire
perimeters, highway logistics, laboratory automation, and market or
geopolitical monitoring instantiate the same triple of predictive
claims, priced channels, and consequence classes.

\subsection{Architecture}\label{sec:floodsar-arch}

\begin{figure}[htbp]
\centering
\begin{tikzpicture}[node distance=2.6mm]
\node[layerbox, minimum width=78mm] (l1) {Commander: authority for irreversible action classes};
\node[layerbox, minimum width=78mm, below=of l1] (l2) {Mission controller: event-sourced state, dispatcher};
\node[storebox, minimum width=78mm, below=of l2] (l3) {TRACE runtime: schema v1.0, policy \fld{trace_v1.yaml},\\ hash-chained append-only store};
\node[viewbox, minimum width=78mm, text width=73mm, below=of l3] (l4) {Predictive world view (kinematics; transparent route score\\ + heuristic support/OOD + frozen isotonic) and adaptive refresh manager};
\node[layerbox, minimum width=78mm, below=of l4] (l5) {Planner and channels: drone, sensor, human, external feed};
\node[layerbox, minimum width=78mm, below=of l5] (l6) {Environment: flood dynamics over the Antioch delta asset};
\draw[flow] (l1) -- (l2); \draw[flow] (l2) -- (l3);
\draw[vflow] (l3) -- (l4); \draw[flow] (l4) -- (l5); \draw[flow] (l5) -- (l6);
\end{tikzpicture}
\caption{The six-layer Flood-SAR architecture. Every arrow crossing
into or out of the TRACE runtime is a record read or write; that no
service passes reasoning state in any other form is an architectural
invariant enforced by the interfaces and checked by the workbench's
contract tests.}
\label{fig:arch}
\vspace{-.1in}
\end{figure}
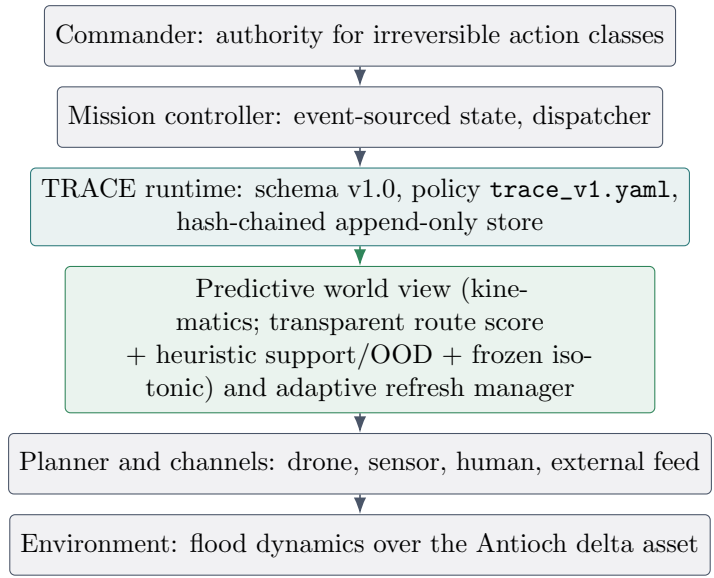

Figure~\ref{fig:arch} shows the architecture in six layers: a commander (human
authority) above a mission controller owning the event-sourced mission
state; the TRACE runtime holding schema, policy, and the append-only
store; the predictive world view and adaptive refresh manager; the
planner and verification channels; the environment. Controller model
state is deliberately separated from the event-sourced store: the view
may be wrong, but the record preserves what the system represented and why the gate accepted it, subject to the store-integrity assumptions
of Section~\ref{sec:floodsar-store}. The operational console over
the Antioch delta geography asset that all campaigns of
Section~\ref{sec:eval} use is shown in Figure~\ref{fig:ui}.

\begin{figure}[htbp]
\centering
\includegraphics[width=\columnwidth]{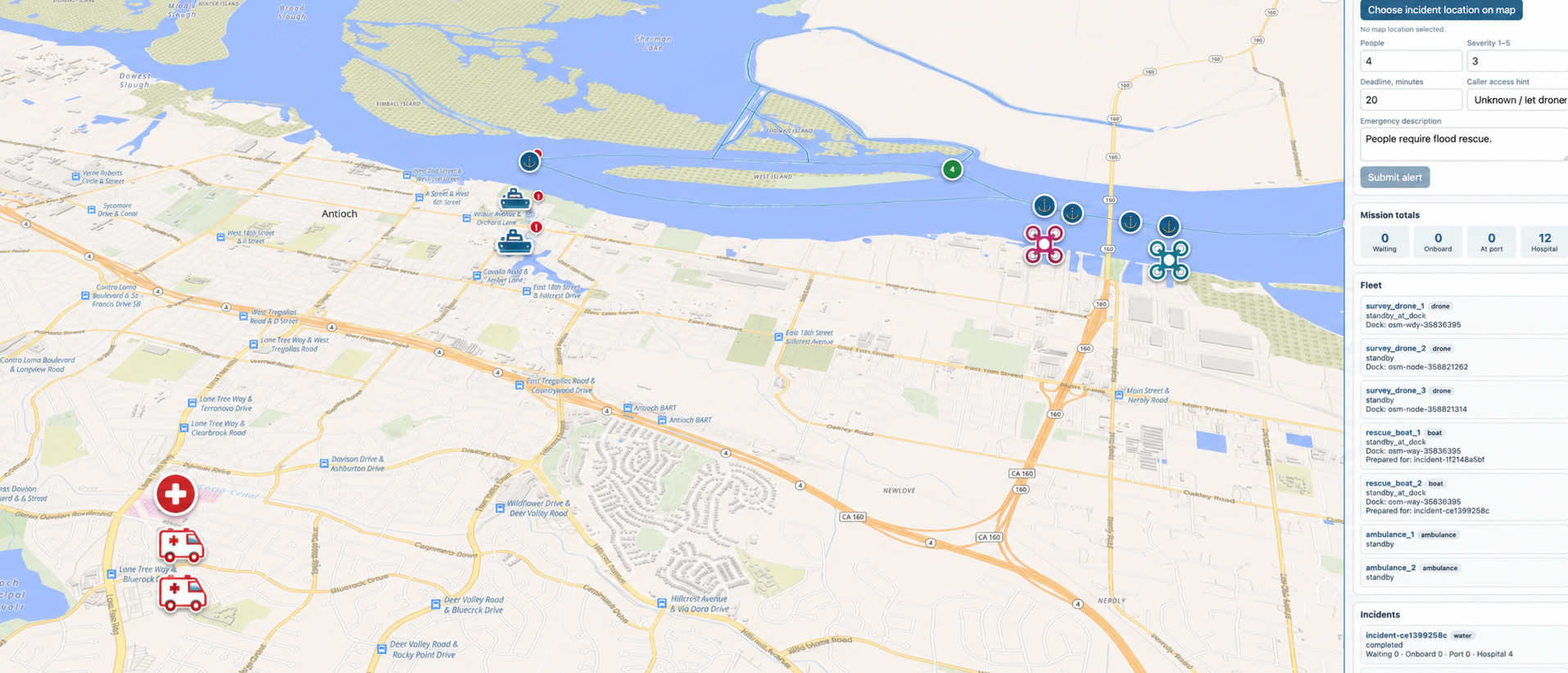}
\caption{The Flood-SAR operational console over the Antioch delta
geography asset. The map renders docks, waterways, hospital, and the
live positions of survey drones, rescue boats, and ambulances; the
right panel supports incident entry (people, severity, deadline,
caller access hint), reports mission totals (waiting, onboard, at
port, hospital; the evaluation's completion endpoint is the safe
transfer dock, and hospital time is a projected quantity, not an
observed clinical outcome), and lists per-asset fleet state with
the authorizing dock identifiers and per-incident outcome accounting. Every quantity shown is a
projection of the event-sourced store of
Section~\ref{sec:floodsar-store}; the console holds no authoritative
mission state.}
\label{fig:ui}
\end{figure}

\subsection{The record store}\label{sec:floodsar-store}
The store is event-sourced and hash-chained: every claim, gate
decision, refresh, and revision is an append-only record whose
integrity a full-chain scan verifies; storage mechanics and layout
detail appear in \extpointer{app:moved-system}; the operator console over this store appears in Figure~\ref{fig:ui}.

\subsection{The condition hierarchy}\label{sec:floodsar-hierarchy}
Conditions form a three-level hierarchy, signal to event to
practical claim, so that what the gate judges is the decision an
application cares about rather than a sensor value; the full
hierarchy and its field-level detail appear in
\extpointer{app:moved-mech}.

\subsection{Implementation status}\label{sec:floodsar-status}

The canonical D0.5 release realizes the contract end to end; its
test counts and release provenance are catalogued
in \extpointer{app:artifact}, and storage detail in
\extpointer{app:moved-system}, and the
running gate demonstration appears in
Section~\ref{sec:floodsar-map}. Everything measured in
Section~\ref{sec:eval} runs on this release with the frozen policy
\fld{trace_v1.yaml}.

\section{Maintaining Commitments: Adaptive Refresh and the Gate}\label{sec:refresh}

This section answers the remaining two questions over the contract of
Section~\ref{sec:contract}. Signal economy occupies
Sections~\ref{sec:refresh-kalman}
through~\ref{sec:refresh-theorem}: when to acquire evidence, through
which channel, and with what guarantees. Consequence admission
occupies Sections~\ref{sec:gate} through~\ref{sec:gate-chains}: what
an admitted claim may authorize, and how commitments chain.

\subsection{From Kalman state to hybrid beliefs}\label{sec:refresh-kalman}

For continuous substate $x_t$ (river level, travel time, vehicle
positions), the view advances
$\hat{x}_{t+1|t} = f(\hat{x}_{t|t}, a_t)$ and propagates uncertainty
$P_{t+1|t} = F_t P_{t|t} F_t^{\top} + Q_t$; an observation $y_t$ yields
the innovation $\nu_t = y_t - H_t \hat{x}_{t|t-1}$ and the standard
correction~\cite{kalman1960new}. Two quantities feed the refresh
logic: the predicted variance of the commitment-relevant functional,
$\sigma^2_c(t+h) = g_c^{\top} P_{t+h|t}\, g_c$, which grows with the
horizon and drives the predicted trigger; and the innovation from
whatever partial observations arrive for free, which drives the
observed trigger. This is the special case, not the theory; it is also
precisely the case the 2004 stream system implemented, with $\delta_i$
as the tolerance and the mirror filter as the innovation monitor.

\label{sec:refresh-abstract}%
Flood-SAR state, however, is hybrid: edge passability and dock reachability are
categorical and carry no covariance. The refresh principle of
Definition~\ref{def:refresh} is therefore stated over beliefs, not
Gaussians. The discrete case is instantiated by the workbench's
current risk package: a transparent surrogate produces a raw route
probability and heuristic support/OOD quantities, and a frozen isotonic
map calibrates the route probability. The seed-grouped OOF check gives
ECE $0.00477$ on the declared development mixture
(Section~\ref{sec:eval-calibration}); this is a marginal, horizon-1
route-traversability statement, not evidence that the heuristic
support/OOD quantities or a future learned predictor are calibrated.
A worked example fixes intuition: a passability claim ``edge E12
is traversable by rescue boat'' carries a calibrated probability, a
support score, and an OOD score; a drone overflight observing debris
where the view predicted clear water is an innovation on the
declared observable, and the claim is revised append-only. The 2004
system and the linear-Gaussian case are instances of the abstract
principle, not the principle itself.

\subsection{Trigger taxonomy}\label{sec:refresh-taxonomy}

Four trigger families organize every refresh the architecture can
fire: predicted (validity or predicted uncertainty crosses the action
bound before any observation arrives), observed (innovation from a
partial observation exceeds the declared tolerance), structural
(support or OOD leaves admitted competence), and normative (policy or
authority demands fresh evidence for the action class). Each family
is a declared observable in the sense of
Section~\ref{sec:contract-observables}; the table mapping each to
its policy fields in \fld{trace_v1.yaml} and to the fixed-priority
reason codes of the evaluation discipline is in
\extpointer{app:triggers}.

\subsection{The refresh decision rule}\label{sec:refresh-rule}

A fired trigger poses two decisions: whether to acquire evidence, and
through which channel. Let $L(a, s)$ be the loss of executing action
$a$ in true state $s$ under belief $b$, with channel costs $c_j$
scalarized into the same loss units. For channel $j$ with observation
model $p_j$, and $y \sim p_j(\cdot \mid b)$ drawn from the
posterior-predictive distribution, the value of information is the
preposterior quantity~\cite{howard1966voi}
\[
\mathrm{VoI}_j = \mathbb{E}_{s\sim b}[L(a^{*}_b, s)]
- \mathbb{E}_{y \sim p_j(\cdot\mid b)}\!\left[\mathbb{E}_{s \sim b|y}[L(a^{*}_{b|y}, s)]\right],
\]
which is strictly positive only when, with positive predictive
probability, observing $y$ changes the posterior-optimal action and
strictly reduces expected loss. The rule is one deterministic
procedure. Define the adequate set preposteriorly,
$J^{\mathrm{adm}}_c = \{\, j :
\Pr_{y}[\,\mathrm{Gate}(b \mid y) = \mathrm{CLEAR}\,] \ge
\gamma_c \,\}$ for the pending class's declared level $\gamma_c$.
Acquire evidence only when $J^{\mathrm{adm}}_c \neq \emptyset$ and
$\max_{j \in J^{\mathrm{adm}}_c} (\mathrm{VoI}_j - c_j) > 0$,
routing to the maximizer; otherwise the outcome is the safe
alternative, HOLD, or ESCALATE, never action on stale evidence. Two
clarifications. The comparison is between \emph{reducible} loss and
sensing cost, since irreducible loss cannot justify a purchase that
does not reduce it. And adequacy is checked before net value: a
channel with positive net value but inadequate posteriors cannot
authorize the commitment, which is exactly the channel a
cost-normalized ratio would prefer.

The upshot: refresh concentrates evidence where an observation could
move the belief across the boundary at which the pending commitment
switches, while a fixed interval spends uniformly; the worked
boundary example and three practical refinements (consequence
scaling, reversibility, routing auditability) appear in
\extpointer{app:moved-mech}.

\subsection{A model-conditional deadline lemma and a scheduler
separation}\label{sec:refresh-theorem}

Three quantities must not be conflated: the \emph{boundary
distance} $d = |x - \theta|$ (a state-space quantity), the
\emph{threshold-exceedance probability} $p_t$ with its
action-specific \emph{flip risk} $q_t(a)$ (posterior quantities),
and the \emph{decision margin} $m(b)$ of Definition~\ref{def:margin}
(an expected-loss gap). The lemma budgets the posterior expected loss
from authorizing the wrong action, not the margin: at the action
boundary the margin vanishes while this decision-error loss peaks at
$\Delta(L-\Delta)/L$. The deadline depends on $d$, and no statement in
this paper identifies the three quantities.

\paragraph{Setup.} A scalar substate follows
$dx_t = \sigma\, dW_t$; between verification epochs the observation
filtration gains nothing, and an exact observation at cost $c > 0$
is the only channel. Actions: $a_f$ with loss
$L \cdot \mathbf{1}[x_t > \theta]$; $a_s$ with certain loss
$\Delta$, $0 < \Delta < L$. Define the realized decision-error losses
$\ell_t^{\mathrm{err}}(a_f)=(L-\Delta)\mathbf{1}[x_t>\theta]$ and
$\ell_t^{\mathrm{err}}(a_s)=\Delta\mathbf{1}[x_t\leq\theta]$. Their
posterior expectations are
$\Lambda_t(a_f)=p_t(L-\Delta)$ and
$\Lambda_t(a_s)=(1-p_t)\Delta$. The trigger at level $\varepsilon_c$
fires when $\Lambda_t(a_t)\geq\varepsilon_c$, equivalently at an
action-specific
flip-probability threshold $\eta_f = \varepsilon_c/(L - \Delta)$ or
$\eta_s = \varepsilon_c/\Delta$, consequence scaling in the trigger
itself. Assume
$0 < \varepsilon_c < \tfrac{1}{2}\min\{\Delta, L - \Delta\}$
(both thresholds inside $(0, \tfrac12)$, held action belief-optimal
between epochs) and polled distance $d_i > 0$.

\begin{lemma}[Model-conditional decision-loss deadline]
\label{thm:deadline}
Under exact observation, zero drift, no intermediate evidence, and
before any accumulation point of the epoch times: after an observation
at boundary distance $d_i > 0$ while authorizing point action $a_i$, the
posterior expected decision-error loss satisfies
$\Lambda_t(a_i)\leq\varepsilon_c$ for
$t \in [t_i,\, t_i + \tau_i]$, where
$\tau_i = d_i^{\,2} \big/ \big(\sigma^2\, z_{\eta(a_i)}^2\big)$, with
$z_\eta = \Phi^{-1}(1 - \eta) > 0$ and
$\eta(a_i) \in \{\eta_f, \eta_s\}$. The deadline $\tau_i$ is an
\emph{admissibility deadline} on a continuing authorization, not
unconditionally a verification time: at $t_i + \tau_i$ the gate must
either acquire evidence adequate under
Section~\ref{sec:refresh-rule}'s rule or cease authorizing $a_i$, via
the safe alternative, HOLD, or ESCALATE.
\end{lemma}

\emph{Proof sketch.} With zero drift the posterior mean is frozen, so
the action-specific flip risk satisfies
$q_t(a_i) = \Phi(-d_i/(\sigma\sqrt{t - t_i}))$, rising
deterministically from $0$ toward $\tfrac{1}{2}$; for $a_f$ this is
$p_t$, and for $a_s$ it is $1 - p_t$, since $p_t$ itself falls from
one toward $\tfrac{1}{2}$. Hence $\Lambda_t(a_i)$ rises deterministically
and reaches $\varepsilon_c$ exactly at $t_i + \tau_i$, which is finite because
$\eta(a_i) < \tfrac{1}{2}$ gives $z_{\eta(a_i)} > 0$. Belief-optimality
holds on the certified interval in both cases: holding $a_f$ requires
$p_t \le \Delta/L$, and $p_t \le \eta_f = \varepsilon_c/(L-\Delta)$,
which the assumption bounds by $\Delta/L$ whether $\Delta \le L/2$
(then $\eta_f < \Delta/(2(L-\Delta)) \le \Delta/L$) or $\Delta > L/2$
(then $\eta_f < \tfrac12 < \Delta/L$); holding $a_s$ requires
$p_t \ge \Delta/L$, and $1 - p_t \le \eta_s = \varepsilon_c/\Delta$
gives $p_t \ge 1 - \eta_s > \Delta/L$ by the symmetric case split.
The decision-error-loss identities are immediate from the loss table.
$\square$

The corresponding realized-path cost identity, including its terminal
residual and the distinction between the analytical and guarded
schedulers, appears in \extpointer{app:formal}; no expectation-level
cost formula is claimed.

\paragraph{Sustained-validity oracle.}
Definition~\ref{def:stale} requires an extended commitment to remain
valid throughout its declared execution window.  The following result
prices that interval event directly; write
$\mathcal E(t_i,t)=\{\exists s\in[t_i,t]:
d_\omega(\omega(S_s),\hat\omega)>\theta\}$.
Measuring from the last exact
observation is conservative when the physical action begins later,
because its execution window is then a subinterval of the one bounded
below.

\begin{lemma}[Sustained-validity risk deadline]
\label{lem:sustained}
In the Brownian setting of Lemma~\ref{thm:deadline}, let $t_i$ be a
finite observation stopping time at which an exact observation places
the state at one-sided boundary distance $d_i>0$. For a declared
interval-risk budget $\eta_i^{\mathrm{sus}}\in(0,1)$, an extended
commitment satisfies
$\Pr[\mathcal{E}(t_i,t) \mid \mathcal{F}^{\mathrm{obs}}_{t_i}] \le
\eta_i^{\mathrm{sus}}$ for every
$t \in [t_i,\, t_i + \tau^{\mathrm{sus}}_i]$, where
\[
  \tau^{\mathrm{sus}}_i
  = \frac{d_i^{\,2}}
         {\sigma^2\, z^2_{\eta_i^{\mathrm{sus}}/2}},
  \qquad z_q=\Phi^{-1}(1-q)>0\quad(q<\tfrac12).
\]
For the full interval beginning at $t_i$, the bound holds with equality
at the endpoint and is tight. The gate may authorize an extended
commitment only when its whole execution window is contained in the
certified interval, as required by Definition~\ref{def:oracle}.
\end{lemma}

\begin{proof}[Proof of Lemma~\ref{lem:sustained}]
Take $x_{t_i} = \theta - d_i$, so exceedance is upward; the downward
case is identical after replacing $W$ by $-W$. For $t > t_i$,
$x_t - x_{t_i} = \sigma(W_t - W_{t_i})$, and under the
boundary-anchored tolerance the interval exceedance event is the
first-passage event
$\{\sup_{s \in [t_i,t]} \sigma(W_s - W_{t_i}) \ge d_i\}$.
Conditionally on $\mathcal{F}^{\mathrm{obs}}_{t_i}$ the normalized
increment process is a standard Brownian motion started at zero, by the
strong Markov property at the finite observation stopping time. The
reflection principle gives
$\Pr[\sup_{s \le u} W_s \ge b] = 2\,\Pr[W_u \ge b]
 = 2\,\Phi(-b/\sqrt{u})$ for $b > 0$, so with $u = t - t_i$ and
$b = d_i/\sigma$,
\[
  \Pr[\mathcal{E}(t_i,t) \mid \mathcal{F}^{\mathrm{obs}}_{t_i}]
  = 2\,\Phi\!\left(\frac{-d_i}{\sigma\sqrt{t - t_i}}\right),
\]
continuous and strictly increasing in $t$, tending to $0$ as
$t \downarrow t_i$ and to $1$ as $t \to \infty$. Requiring it at most
$\eta_i^{\mathrm{sus}}$ is equivalent to
$d_i/(\sigma\sqrt{t - t_i}) \ge
z_{\eta_i^{\mathrm{sus}}/2}$, that is
$t - t_i \le d_i^{\,2}/
(\sigma^2 z^2_{\eta_i^{\mathrm{sus}}/2})$. Monotonicity extends
the bound from the endpoint to the whole interval.
\end{proof}

The parameter $\eta_i^{\mathrm{sus}}$ is a direct probability budget
for any unsafe excursion during the execution window. It can be linked
to an application loss budget only after the application declares the loss
assigned to such an excursion. The point-action thresholds $\eta_f$ and
$\eta_s$ of Lemma~\ref{thm:deadline} do not automatically supply that
duration-wide loss semantics. When Lemma~\ref{lem:sustained} is used as
the oracle of Definition~\ref{def:oracle}, the gate records
$\eta_i^{h_i}=\eta_i^{\mathrm{sus}}$ and the commitment's extended
hazard type together; it does not import either point-action threshold.

\paragraph{The price of sustained validity.}
For the same boundary distance and the same numerical probability
budget $\eta$, the first sustained-validity deadline is the point
deadline with $z_\eta$ replaced by $z_{\eta/2}$. It therefore contracts by
$(z_\eta/z_{\eta/2})^2$ and the nominal check rate grows by its reciprocal:
$0.816$ and $1.226$ at $\eta = 0.01$, $0.704$ and $1.420$ at
$\eta = 0.05$, $0.607$ and $1.647$ at $\eta = 0.10$, and $0.431$ and
$2.319$ at $\eta = 0.20$. The reciprocal is a nominal continuous-time
check-rate multiplier for one fixed distance and span, not an exact
realized check-count or spend ratio: ceiling effects, channel prices,
guards, and the observations themselves alter subsequent deadlines.
For a latency window of length $\delta$, the Brownian interval floor is
$2\Phi(-d/(\sigma\sqrt{\delta}))$, exactly twice the corresponding
point floor. The fixed-horizon observation-cost argument of
Theorem~\ref{thm:no-sync}(ii) has an extended-event analogue with
$z_p$ replaced by $z_{p/2}$; Proposition~\ref{prop:scheduler} applies
only after a finite sequence of sustained deadlines has been fixed.

\paragraph{Accumulation under sustained validity.}
For a constant interval-risk budget $\eta^{\mathrm{sus}}$, because
$\sigma\sqrt{\tau^{\mathrm{sus}}_i} =
d_i/z_{\eta^{\mathrm{sus}}/2}$, the
polled-distance recursion of \extpointer{app:formal} becomes
$d_{i+1} = d_i\,|1 - \xi_i/z_{\eta^{\mathrm{sus}}/2}|$ for iid
standard-normal $\xi_i$. The strong law gives negative logarithmic
growth and almost-sure accumulation of the unguarded epoch times when
$z_{\eta^{\mathrm{sus}}/2} > z^{\ast} \approx 0.6427$, that is for
$\eta^{\mathrm{sus}} < 2(1 - \Phi(z^{\ast})) \approx 0.5204$, against
$\eta < 0.2602$ under the point hazard. If the numerical RQ6 limits
$0.12$ and $0.20$ are separately declared as interval-risk budgets,
both lie in this analytic regime. The HOLD and ESCALATE guards are then
required for those extended-event analogues: under the idealized
unguarded recursion the epoch times accumulate almost surely, whereas
the deployed contract ceases authorization when its tick, rate, or cost
guard binds. No optimality claim is made for that fallback.

\paragraph{When the point hazard is the right one.}
Lemma~\ref{thm:deadline} remains correct for point commitments and
nothing above supersedes it. The distinction is declared per commitment
class and recorded in the claim, so it is auditable on the same terms as
the rest of the contract: a class that declares itself point while
authorizing an action with extended exposure is a contract violation the
record exhibits, not a modeling choice hidden in a scheduler.

\begin{proposition}[Deterministic scheduler separation]\label{prop:scheduler}\label{prop:sched}
Fix a finite sequence of declared admissible deadlines
$\tau_1, \dots, \tau_N$, write $S_N = \sum_i \tau_i$ and
$\tau_{\min} = \min_i \tau_i$. For this comparison, the initial
check at time zero is free, a paid check is required at the closed
endpoint $S_N$ to continue authorization there, and a scheduler
honors the deadlines if every inter-check interval is no longer than
the deadline in force. The deadline scheduler makes $N$ paid checks,
at the cumulative deadline endpoints. A uniform-base scheduler with
spacing $k\leq\tau_{\min}$, allowing its terminal interval to be
shorter, requires at least $\lceil S_N/k\rceil$ paid checks. Hence
\[
 \frac{N_{\mathrm{uniform}}}{N}
 \geq \frac{\lceil S_N/\tau_{\min}\rceil}{N}
 \geq \frac{\bar\tau}{\tau_{\min}}.
\]
Equality is possible. If no check is required at the terminal
endpoint, the exact floor count replaces the ceiling and the final
ratio need not follow. This is a counting statement under the stated
closed-span convention, not a comparison between stochastic policies
whose observations alter their later deadlines.
\end{proposition}

\emph{Zeno and scope.} Near the boundary, epochs can accumulate, so
the workbench uses tick-resolution epochs and a hard guard ending
authorization in HOLD or ESCALATE when the deadline falls below one
tick or a declared rate or cost limit binds. The lemma is
model-conditional to zero-drift Brownian dynamics with exact
observations; the full Zeno calculation and scope discussion appear
in \extpointer{app:formal}.

\subsection{Gate conditions}\label{sec:gate}\label{sec:gate-conditions}

The gate evaluates a claim record against the declared policy and
issues a consumer decision. The deployed policy states the conditions
in the frozen vocabulary: support at least $0.60$; OOD at most
$0.35$; calibrated uncertainty at most $0.30$; prediction horizon at
most 8 steps; observation age at most 1800 seconds behind an
irreversible action; no undefeated defeaters; authority present for
the classes that require it (\fld{dispatch_rescue_boat},
\fld{deploy_ground_team}), while a reversible probe may be cleared
as qualified without commander authority. Decisions are evaluated in
the fixed precedence BLOCK, ESCALATE, QUALIFY, HOLD, CLEAR, so the
most restrictive applicable decision wins and the ordering is part of
the declared policy rather than an implementation accident. Freshness
is the condition this paper adds: the gate checks that the claim's
validity interval covers the commitment horizon, which is exactly
where a fixed-interval system silently assumes what \system{}
verifies.

A revalidation guard extends the conditions with model provenance
(Section~\ref{sec:contract-claims}): clearing a high-consequence
commitment additionally requires \fld{model_version_current}
$\wedge$ \fld{calibration_adequate_for_class}. If either conjunct
fails the claim cannot authorize a new high-consequence commitment,
may still support QUALIFY on a reversible probe, and the failure is
recorded with the superseding version identified.

The gate's object, finally, is the decision and not the signal. What it
clears or holds is a practical claim, a choice among admissible
alternatives with stated reversibility and opportunity cost, licensed
by the event claims beneath it. Holding the fast route while a safe
alternative remains admissible is a different outcome, and a different
record, from holding when no alternative survives, which escalates to
the commander rather than silently stalling the mission.

\subsection{The ActiveTRACE loop}\label{sec:gate-loop}

The ActiveTRACE loop closes prediction, gating, action, and refresh
into one auditable circuit of eight steps (observe, estimate,
predict, gate, act, refresh, revise, complete), each step writing or
reading fields of the frozen schema and no step passing reasoning
state any other way.

\subsection{Schedules are chains of commitments}\label{sec:gate-chains}

A rescue schedule is a chain, not a single commitment: dock selection
depends on passability claims, rendezvous timing on travel-time claims,
hospital arrival on both. \system{} treats each link as a record and
the chain as a transactional object in the sense of
SagaLLM~\cite{chang2025sagallm,garcia1987sagas}: a sequence of
locally atomic steps each paired with a compensating transaction,
$S=\{T_1,\ldots,T_n,C_n,\ldots,C_1\}$, invoked in reverse over
the committed prefix. SagaLLM motivates this module interface, but
the exact compositional rollback, frame, and isolation obligations
used by Theorem~\ref{thm:composition} remain explicit assumptions
A3--A5 rather than consequences silently imported by citation
\cite{chang2025sagallm}. Dependencies
are explicit, $D=\{(o_i,o_j,c_{ij})\}$: when a refresh revises an
upstream claim, the dependency traversal computes the
dependency-scoped set of affected downstream commitments, and compensation (free the asset,
re-plan the dock, re-time the rendezvous) proceeds under the same
gate discipline as the original dispatch, touching nothing outside
the affected set.

\subsection{Mechanisms mapped to the theory}\label{sec:floodsar-map}

The canonical release (D0.5, predictive scheduling) realizes the
paper's mechanisms directly. Validity horizons on travel-time and
passability claims realize refresh conditions; scheduling fields
(\texttt{scheduled}, \texttt{dispatched}, \fld{expected_arrival},
\fld{actual_arrival}, \fld{completed_at}) make the innovation
on arrival predictions directly observable. Drone-first verification
realizes VoI routing. Dock selection minimizing hospital arrival, early
ambulance dispatch, parallel readiness, and rendezvous states realize
commitment chains. A running demonstration exercises the gate end
to end (support-gate failure, \texttt{defer} with named missing
observation, held dispatch, drone revision, alternative branch
cleared); its trace is reproducible from the released
implementation.

\section{Evaluation}\label{sec:eval}

Table~\ref{tab:theory-map} maps each formal element of
Section~\ref{sec:theory} to the measurement that checks it. The
calibration precondition tests semantic elevation; RQ1 prices
signal economy against fixed and clock baselines; RQ6 exercises the
composition, measures selected debts, and asks whether localized
repair closes the violations refresh lets through and at what joint
price; and the secondary
studies cover mechanism attribution (RQ2), rescue outcomes (RQ3),
record-layer cost (RQ4), and the revalidation guard under
mid-mission predictor replacement (RQ5).

\begin{table}[htbp]
\caption{Each formal element of Section~\ref{sec:theory} linked
to its implementation evidence. Proof establishes the conditional
results (proofs in Section~\ref{sec:theory}); measurements estimate
assumption slack and empirical sensitivities at pre-registered sites.
Status reflects the completed RQ5 and RQ6 adjudication; bounded and
approximate entries are honest constants, not failures.}
\label{tab:theory-map}
\centering
\begingroup
\let\small\footnotesize
\setlength{\tabcolsep}{2pt}
\begin{tabular}{@{}llll@{}}
\toprule
Element & What must hold & Checked by & Status \\
\midrule
A1 & exact stale-risk transport & O6 & sensitivity input (.131) \\
A2 & ceiling $\delta_d$ & RQ6 (O2) & not uniform; coverage .83--.89 \\
A3, A5 & compositional rollback & module contract; O3 & exercised, assumed \\
A4 & acyclic $D$ & construction & held \\
A6-C & completeness (recall) & closure; O7 & misses in 2/58 repair missions \\
A6-S & soundness (precision) & closure; O7 & P .88--.94 \\
A7 & enqueue; grace; completion & O2, O3 & not certified generally \\
Thm.~\ref{thm:no-sync} & stale floor $p_\delta$ & synthetic analytic test & formula checked \\
Thm.~\ref{thm:no-comp} & post hoc invariance & RQ6 (O1, O8) & 160/160 \\
Thm.~\ref{thm:composition}(i) & event-aligned oracle & analytic lemma; O6 & marginal in deployment \\
Thm.~\ref{thm:composition}(ii) & endpoint restoration & A3--A7 & conditional \\
Thm.~\ref{thm:composition}(iii) & debt decomposition & RQ6 (O3, O5) & pathwise; cells incomplete \\
Cor.~\ref{cor:price} & fixed-sequence count & synthetic check & convention-specific \\
\bottomrule
\end{tabular}
\endgroup
\end{table}

Both campaigns the evaluation plan required have landed. RQ5,
the revalidation-guard stress test, witnesses the invariant behind
Theorem~\ref{thm:composition}(i) under mid-mission predictor
replacement (Section~\ref{sec:eval-rq5}). RQ6, the composition
campaign, checks the recovery-noninterference observable
Theorem~\ref{thm:no-comp} predicts through O1's exact invariance,
exercises the repair machinery of
Theorem~\ref{thm:composition}(ii), and supplies the empirical
sensitivities for part (iii): $f_{\mathrm{restore}}^{U} = 0.165$ on the
restoration-by-mission-end endpoint (O3) and
$\varepsilon_{\mathrm{cal}}^{U} = 0.131$ (O6)
(Section~\ref{sec:eval-rq6}). These values are not a joint
cluster-valid population certificate. RQ1 through RQ4 supply their rows'
evidence. Throughout, Theorem~\ref{thm:composition} is a proved
conditional contract whose assumptions are recorded or partially
measured, not a theorem
the campaigns empirically discharge: detection coverage is 0.83 to
0.89 (A2), the affected set is approximately, not exactly, the true
set (A6), A7 is not generally certified, and the map records each
constant as it is.

One design decision governs how every result below should be read.
The deployed predictor is a transparent kinematic surrogate with a
frozen calibration map, by design: every measured effect is thereby
attributable to the contract, the triggers, the gate, and the record
rather than to predictor quality, and the calibration precondition
is checkable precisely because the predictor is simple. The
contract's prediction for a richer predictor is falsifiable and
stated now: substituting an uncalibrated learned model behind the
same gate should raise hold and escalate rates, not clear rates,
until calibration within declared support is earned, and a gate that
clears more freely under a fancier, uncalibrated predictor is
broken.

\subsection{Methodology and discipline}\label{sec:eval-method}

The evaluation reuses the workbench's S1--S5 scenario suite; R-B,
drift plus diffusion in S2, hosts the primary comparison because
rising water changes passability at a rate the refresh policy must
track. The declared development grid contains 19 configurations: no
refresh; fixed-$k$ at $k\in\{5,15,45,120,300,600\}$ seconds; a
validity clock at $\alpha\in\{0.1,0.25,0.4,0.55,0.75,1.0\}$,
inspired by the sensing-clock method~\cite{wang2026sensingclocks}
without claiming its coverage certificate; and adaptive refresh at
$\varepsilon_c\in\{0.25,0.5,1,2,4,7\}$. Development seeds 1--20
produce 380 missions; the frozen campaign runs all 19 configurations
on validation seeds 101--125 (475 missions) and the three selected
policies on test seeds 1001--1080 (240 missions). All completed,
every event and TRACE chain verifies, and the common-random-number
checks pass. The unit of analysis is the commitment: each dispatch,
dock selection, and rendezvous is scored against its declared hazard
type---at the execution instant for a point commitment and over the
full recorded execution window for an extended commitment.

\emph{Fairness and tuning.} The compared policies differ only in
when they seek evidence; everything else, seeds, predictor,
calibration, planner, claim representation, gate, channel menu,
fallback, and loss definitions, is identical, and a gate HOLD
imposed on a fixed policy before its next scheduled observation is
charged to that policy. Validation freezes four cost budgets and
the selected points, fixed $k=600$, adaptive $\varepsilon_c=4$, and
clock $\alpha=.4$, with the operating-point artifact hashed before
any test outcome is opened; the budget values are recorded in that
frozen operating-point artifact.

\emph{Partitions, statistics, and metrics.} Development fits
models and thresholds, validation fixes grids and budgets, both are
frozen before the single test evaluation, and all policies share
common random numbers; held-out contrasts use a paired max-$t$
bootstrap with simultaneous 95\% bands over the declared family,
and ``pre-registered'' means manifest \fld{exp-freeze-v1} with an
execution commit and annotated tag recorded before test seeds are
drawn. To prevent conservative policies from gaming staleness,
coverage (unique executed over unique proposed irreversible
opportunities) accompanies every staleness number; adjudication is
deterministic with no model in the path, under fixed-priority
reason codes. A synthetic artifact checks the deadline formula, the
censored cost identity, and Proposition~\ref{prop:scheduler}
numerically; Flood-SAR tests only the lemma's qualitative
scheduling implication. Bootstrap mechanics, attrition codes, and
the seed-count power analysis are recorded in the artifact's
methodology package.

\emph{RQ5 and RQ6 protocols, as run.} Both campaigns ran on
self-contained stochastic workbenches with no V-JEPA or external
model weights, under protocols hashed before any seed was drawn and
verified by manifest (RQ5 hash \fld{d472828d748d}, seeds
3101--3180; RQ6 hash \fld{ef1fdb994db5}, seeds 2101--2180, 80
missions per cell, 320 seed-cell rows, disjoint from every
development, validation, and test seed). The RQ6 two-by-two crosses
two refresh policies with two registered recovery disciplines:
localized saga repair over the computed dependency-scoped affected
set versus global recovery, in which the dispatcher continues on the stale chain
until failure surfaces and recovery is a global replan; the
contrast is between two corrective mechanisms. The adaptive arm is
the consequence-aware adaptive policy (consequence-specific maximum
ages 240/420~s, refresh triggers .055/.10, risk limits .12/.20); it
is not the $\varepsilon_c=4$ workbench point of RQ1, and no mapping
between the two is claimed. One provenance caveat is declared
rather than hidden: the protocols froze $\alpha$, seed clustering,
and 10,000 replicates, but not the membership of the max-$t$
family, so the simultaneous bands reported below are conservative
post-run sensitivity families over all emitted contrasts (10 for
RQ5, 33 for RQ6), alongside the frozen individual paired-bootstrap
intervals.
The headline outcome is the
consistency violation of Definition~\ref{def:violation}, judged
mechanism-neutrally: a commitment stale under its declared point or
extended hazard whose declared invariants no registered corrective
mechanism restores by mission end. The protocol enumerates eight
observables: executed-stale invariance across recovery arms (O1),
shock-to-detected-invalidation latency, the A2 ceiling (O2),
compensation success and failure with restoration latency,
estimating $f_{\mathrm{restore}}$ and a sensitivity upper limit (O3), repair load under adaptive
versus fixed refresh (O4), residual consistency violations (O5),
selective calibration on gate-cleared irreversible commitments,
providing a candidate estimate of the A1 transport and a sensitivity
upper limit (O6), affected-set
precision and recall against the event store, the A6 check (O7),
and the outcome and priced-cost guards, including stranded assets
and orphaned dependents (O8). For robustness, we report the
conservative post-run simultaneous sensitivity families described
above; a band containing zero is reported as not separated.

\emph{AI-assisted analysis.} OpenAI ChatGPT and Codex assisted in
developing and inspecting data-analysis scripts and in formatting
tables and figures. They did not generate experimental observations
or select operating points. The author reran the analyses against
the frozen archives, independently re-derived every reported
quantity, and verified all tables and figures against the archived
records.

\subsection{Calibration precondition: are route claims meaningful?}
\label{sec:eval-calibration}

Every proposed dispatch and evacuation route claim, including claims
not cleared by the gate, is labeled by replaying immutable events and
testing whether the route remains open through arrival: of 300,442
development candidates, 169 are right-censored, leaving 300,273
horizon-1 labels. A weighted isotonic map is evaluated in five
folds that hold out complete seeds; the target is route
traversability, not generic plan success. Raw scores are strongly
miscalibrated (ECE $0.32705$); the frozen 16-level map reduces
seed-grouped OOF ECE to $0.00477$, with a rejected fixed-$k=300$
pilot and the full OOF detail in \extpointer{app:calibration};
held-out transfer is primary here.

After operating-point freeze, the unchanged map is applied to every
test claim. Among 62,977 labeled horizon-1 claims (57 are
right-censored), overall ECE is $0.00450$ [.00243, .00700], dispatch
$0.00779$ and evacuation $0.00310$, and both action-class upper
bounds remain below the predeclared $0.05$ threshold
(Table~\ref{tab:calibration-heldout}; development OOF detail in
\extpointer{app:calibration}). Calibration therefore transfers
to the held-out R-B mixture without refitting. The result is marginal
route traversability for dispatch and evacuation at horizon 1; it does
not certify support/OOD scores, longer horizons, other actions, or
total plan success. The ten-bin reliability diagram behind the
sparse-middle-bin caveat, policy-family transportability, and the
wiring diagnostics are reported in \extpointer{app:calibration}.

\begin{table}[htbp]
\centering
\caption{Frozen held-out calibration with seed-clustered 95\% intervals.
No validation or test refit is performed; dispatch and evacuation labels
denote swamping and missed delivery.}
\label{tab:calibration-heldout}
\begin{tabular}{@{}lrrr@{}}
\toprule
\textbf{Class} & \textbf{$n$} & \textbf{ECE [95\% CI]} & \textbf{Brier} \\
\midrule
All routes & 62,977 & .00450 [.00243,.00700] & .00874 \\
Dispatch & 35,028 & .00779 [.00394,.01192] & .00657 \\
Evacuation & 27,949 & .00310 [.00149,.00761] & .01146 \\
\bottomrule
\end{tabular}
\vspace{-.15in}
\end{table}

\subsection{RQ1: Does adaptive refresh dominate fixed intervals?}\label{sec:eval-rq1}

Development sweeps selected the frozen operating points but used
preliminary coverage accounting and are not confirmatory; they showed
the same qualitative warning that staleness must be reported with
coverage. The development point estimates on both declared axes,
with complete tables and paired diagnostics, are plotted in
\extpointer{app:rq1-dev}.

\paragraph{Frozen held-out process.}
Validation selects fixed $k=600$, adaptive $\varepsilon_c=4$, and
clock $\alpha=.4$ without access to test seeds; each runs on the
same 80 test seeds (1001--1080), all 240 missions complete, every
chain verifies, and the common-random-number check passes.
Table~\ref{tab:rq1-heldout} reports test means and the selected
simultaneous 95\% bands from the paired max-$t$ family of 25
predeclared contrasts; RQ1's rescued-person metric is a distinct
pre-registered accounting from RQ3's completed-rescue count and the
two are never compared across tables.

\begin{table}[htbp]
\centering
\caption{RQ1 held-out means at validation-frozen operating points
(80 missions per policy; stale, coverage, completion in percent) and
adaptive-minus-comparator paired differences with simultaneous 95\%
max-$t$ bands. No threshold or point is refit on these outcomes.}
\label{tab:rq1-heldout}
\begingroup
\let\small\footnotesize
\setlength{\tabcolsep}{2pt}
\begin{tabular}{@{}lrrrrr@{}}
\toprule
\textbf{Metric} & \textbf{Adapt.} & \textbf{Fixed} & \textbf{Clock}
 & \textbf{$\Delta$ vs.\ fixed} & \textbf{$\Delta$ vs.\ clock} \\
\midrule
Cost & 3.769 & 2.738 & 67.800
 & 1.03\,[.28,1.78] & $-64.0$\,[$-74.4$,$-53.7$] \\
Stale (pp) & .063 & .215 & 6.574
 & $-.151$\,[$-.28$,$-.03$] & $-6.51$\,[$-8.95$,$-4.07$] \\
Coverage & 97.18 & 99.12 & 97.67
 & $-1.93$\,[$-3.31$,$-.56$] & $-.49$\,[$-2.05$,1.07] \\
Completion & 83.23 & 84.58 & 85.18
 & $-1.35$\,[$-2.14$,$-.56$] & $-1.95$\,[$-2.84$,$-1.05$] \\
Rescued & 197.95 & 201.15 & 201.95
 & $-3.20$\,[$-5.05$,$-1.35$] & $-4.00$\,[$-5.94$,$-2.06$] \\
\bottomrule
\end{tabular}
\endgroup
\end{table}

\paragraph{Held-out answer to RQ1.}
Adaptive refresh does not dominate. Relative to fixed $k=600$, it
reduces stale execution by $0.151$ percentage points, but costs 1.032
more, covers 1.934 points fewer opportunities, completes 1.346 points
fewer people, and rescues 3.20 fewer people; every selected simultaneous
band excludes zero. Relative to the validity clock, adaptive is much
cheaper and 6.511 points less stale, but completes 1.945 points fewer
people; its coverage band includes zero. The confirmatory result is a
cost--staleness--coverage--completion trade-off, not an operational win.

\subsection{RQ6: Does the composition close the loop?}\label{sec:eval-rq6}

The consistency principle of Section~\ref{sec:intro} claims that
prevention and correction are complements; formally, the claim is
Theorem~\ref{thm:composition}, and the headline observable is the
consistency violation of Definition~\ref{def:violation}, the
quantity its part (iii) bounds. RQ6 tests the composition
on one incident stream under the protocol of
Section~\ref{sec:eval-method}: four conditions crossing the selected
refresh policies with two recovery disciplines, localized saga
repair over the computed dependency-scoped affected set versus a
global-recovery baseline in which the dispatcher continues on the stale chain until
failure surfaces and recovery is a global replan, on a fresh frozen
seed block with pre-registered world shocks; the contrast is
between two corrective mechanisms, not correction and its absence.
Four predictions were frozen before any seed: recovery must not
change stale execution; fixed refresh must invoke more
compensations than adaptive; compensation must reduce violations
for both policies; and the full contract must carry the fewest
violations without losing rescue or coverage beyond the band.
Because the violation definition is mechanism-neutral, the headline
count is an empirical comparison, not a terminology artifact, and
the findings adjudicate every prediction, including the two the
data refuse.

\begin{figure}[htbp]
\centering
\begin{tikzpicture}
\begin{groupplot}[
  group style={group size=3 by 1, horizontal sep=14mm},
  width=0.34\textwidth, height=0.24\textwidth,
  ybar, /pgf/bar width=9pt,
  symbolic x coords={Adaptive, Fixed},
  xtick=data,
  enlarge x limits=0.45,
  tick label style={font=\scriptsize},
  label style={font=\small},
  title style={font=\small},
  ymin=0,
  legend style={font=\scriptsize, draw=none, fill=white,
    fill opacity=0.85, text opacity=1, at={(0.03,0.97)},
    anchor=north west},
]
\nextgroupplot[title={Residual violations / mission}, ymax=1.05]
\addplot[fill=atpamber!75, draw=atpamber]
  coordinates {(Adaptive,0.300) (Fixed,0.838)};
\addplot[fill=atpteal!75, draw=atpteal]
  coordinates {(Adaptive,0.300) (Fixed,0.838)};
\legend{Global recovery, Localized saga}
\nextgroupplot[title={Repair work / mission}, ymax=18]
\addplot[fill=atpamber!75, draw=atpamber]
  coordinates {(Adaptive,6.513) (Fixed,15.875)};
\addplot[fill=atpteal!75, draw=atpteal]
  coordinates {(Adaptive,0.888) (Fixed,2.388)};
\nextgroupplot[title={Restoration latency (s)}, ymax=430]
\addplot[fill=atpamber!75, draw=atpamber]
  coordinates {(Adaptive,124.506) (Fixed,391.232)};
\addplot[fill=atpteal!75, draw=atpteal]
  coordinates {(Adaptive,69.845) (Fixed,284.578)};
\end{groupplot}
\end{tikzpicture}
\Description{Three bar panels of RQ6 cell means. Left: residual
violations per mission depend on the refresh policy and are
identical across recovery arms. Center and right: repair work and
restoration latency depend strongly on the recovery discipline,
with localized saga far below global recovery.}
\caption{The composition read at a glance (cell means, 80 missions
per cell). Refresh policy moves residual violations (left; bars are
identical across recovery arms, the tie that refuses the strict
synergy prediction), while recovery discipline moves repair work
and restoration latency (center, right). Prevention and correction
act at different stages: localization buys repair economy, not a
lower residual count. Exact cell means for all seven metrics appear in
Table~\ref{tab:rq6-cells}.}
\label{fig:rq6-stages}
\end{figure}
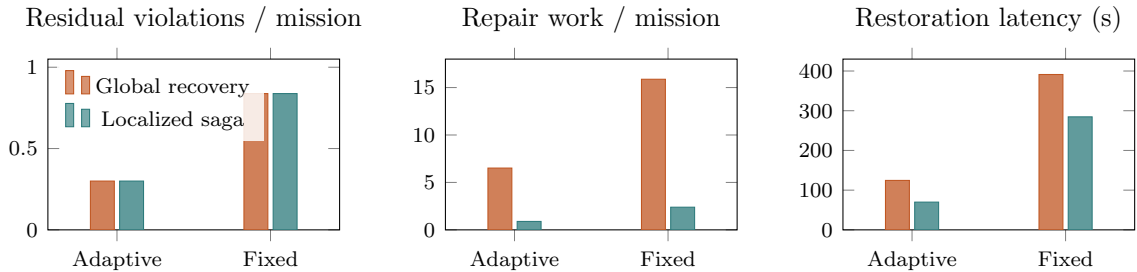

\begin{table}[htbp]
\centering
\caption{RQ6 held-out cell means (80 missions per cell).}
\label{tab:rq6-cells}
\begingroup
\let\small\footnotesize
\setlength{\tabcolsep}{3.5pt}
\begin{tabular}{llrrrrrrr}
\toprule
Refresh & Recovery & Stale & Residual & Completion & Rescued
  & Repair work & Total cost & Restore (s) \\
\midrule
Adaptive & Global
  & 0.600 & 0.300 & 0.788 & 3.163 & 6.513 & 34.275 & 124.506 \\
Adaptive & Localized
  & 0.600 & 0.300 & 0.775 & 3.163 & 0.888 & 28.650 & 69.845 \\
Fixed 600 s & Global
  & 1.625 & 0.838 & 0.463 & 2.600 & 15.875 & 45.875 & 391.232 \\
Fixed 600 s & Localized
  & 1.625 & 0.838 & 0.488 & 2.600 & 2.388 & 32.388 & 284.578 \\
\bottomrule
\end{tabular}
\endgroup
\end{table}

\begin{table}[htbp]
\centering
\caption{RQ6 paired factorial contrasts with simultaneous 95\% max-$t$
sensitivity bands over all 33 emitted contrasts. Refresh is adaptive minus
fixed; recovery is localized minus global.}
\label{tab:rq6-contrasts}
\begingroup\let\small\footnotesize
\begin{tabular}{llr}
\toprule
Outcome & Contrast & Estimate [95\% band] \\
\midrule
Executed stale & Refresh
  & $\mathbf{-1.025}\;[-1.508,-0.542]$ \\
Residual violations & Refresh
  & $\mathbf{-0.538}\;[-0.842,-0.233]$ \\
Mission completion & Refresh
  & $\mathbf{+0.306}\;[+0.114,+0.499]$ \\
Rescued people & Refresh
  & $\mathbf{+0.563}\;[+0.264,+0.861]$ \\
Observation load & Refresh
  & $\mathbf{-2.238}\;[-3.160,-1.315]$ \\
Total operational cost & Refresh
  & $\mathbf{-7.669}\;[-11.128,-4.209]$ \\
Restoration latency (s) & Refresh
  & $\mathbf{-240.730}\;[-374.096,-107.363]$ \\
Detection latency (s) & Refresh
  & $\mathbf{+169.056}\;[+10.649,+327.464]$ \\
Detection coverage & Refresh
  & $-0.064\;[-0.159,+0.031]$ \\
\midrule
Repair work & Recovery
  & $\mathbf{-9.556}\;[-14.946,-4.167]$ \\
Total operational cost & Recovery
  & $\mathbf{-9.556}\;[-14.946,-4.167]$ \\
Restoration latency (s) & Recovery
  & $\mathbf{-80.658}\;[-115.154,-46.162]$ \\
Residual violations & Recovery
  & $0.000\;[-0.078,+0.078]$ \\
Mission completion & Recovery
  & $+0.006\;[-0.049,+0.061]$ \\
Rescued people & Recovery
  & $0.000\;[-0.078,+0.078]$ \\
\midrule
Repair work & Interaction
  & $\mathbf{+7.863}\;[+2.493,+13.232]$ \\
Restoration latency (s) & Interaction
  & $\mathbf{+51.994}\;[+8.054,+95.933]$ \\
\bottomrule
\end{tabular}
\endgroup
\end{table}

\paragraph{Finding 1: adaptive refresh prevents, but not every
sensing metric improves.} Recovery could not change authorization:
executed-stale commitments were identical across recovery arms in
all 160 paired seed and refresh comparisons (O1, exact). Averaged
over recovery arms, the consequence-aware adaptive policy reduced
executed-stale commitments from 1.625 to 0.600 per mission
($-1.025$, band $[-1.508, -0.542]$), reduced residual violations by
0.538, improved completion by 30.6 points and rescued people by
0.563 per mission, and lowered observation load and total
operational cost (Figure~\ref{fig:rq6-stages},
Tables~\ref{tab:rq6-contrasts} and~\ref{tab:rq6-cells}). Registered-shock detection, however,
was 169.1~s slower ($[+10.6, +327.5]$), and the 6.4-point
detection-coverage loss did not survive the simultaneous family.
Prevention here concerns decision-time stale execution, not uniform
dominance of every sensing observable. This workbench and its
adaptive policy differ from RQ1's workbench and frozen
$\varepsilon_c=4$ point, so this effect does not overturn RQ1's
held-out trade-off; it measures a different policy on a different
stream.

\paragraph{Finding 2: localization preserves the end-state at a
fraction of the repair cost.} Relative to global recovery,
localized saga repair reduced repair work and total cost by 9.556
units ($[-14.946, -4.167]$) and restoration latency by 80.7~s
($[-115.2, -46.1]$), while the residual-violation difference was
exactly zero ($[-0.078, +0.078]$), rescued people were unchanged,
and the completion band included zero. The affected-set estimand
must be conditioned correctly: mission-level means score a mission
with no invoked recovery as 1.0 by convention, and most missions
invoke none, so those means overstate exactness. We report affected-set quality
conditional on repair-invoking missions to preserve the registered
mission-seed analysis and clustering unit; commitment-level records
permit finer per-repair diagnostics, which agree qualitatively and
are not substituted for the frozen mission-level estimand. Among
the 58 localized missions invoking at least one repair (19
adaptive, 39 fixed), mean mission-level precision is 0.877 and
0.943 and mean mission-level recall is 1.000 and 0.995; 41 of 58
recovered exactly the true affected set, 16 contained at least one
false-positive dependent, and 2 missed at least one true dependent,
with one mission exhibiting both error types. Global recovery,
conditioned the same way, has mean mission-level precision 0.159
(adaptive) and 0.197 (fixed) at recall 1.000, so localization wins
decisively on the correct estimand while exact affected-set
identity remains the majority case, not the rule (O7). The measured
benefit of localization is repair economy, not a lower residual
count; and no detected difference is not formal equivalence, since
no equivalence margin was frozen.

\paragraph{Finding 3: the mechanisms are stage-complementary; the
strict synergy prediction fails.} Fixed refresh triggered more
localized compensations than adaptive (0.888 versus 0.325
invocations per mission; the fixed arm's 0.888 coincides
numerically with the adaptive arm's 0.888 repair-work units in the
work table, a coincidence, not a transposition; the frozen family
carried no interval for this contrast, so the comparison is
descriptive), while adaptive refresh prevented
more stale executions before recovery could act. Yet adaptive with
localized repair and adaptive with global recovery both ended at
0.300 residual violations per mission: the third declared
prediction, that compensation reduces residual violations under
both refresh policies, is not supported, and the fourth holds only
as a tie. The composition realizes prevention plus economical
correction; it does not strictly improve the headline residual
outcome over adaptive refresh with global recovery.

\begin{table}[H]
\centering
\caption{The debt ledger of Theorem~\ref{thm:composition}(iii).
$\varepsilon_{\mathrm{cal}}^{U}$ and $f_{\mathrm{restore}}^{U}$
are sensitivity upper limits, substitutable only on a valid joint
coverage event; $U_{\mathrm{det}}$ and the
affected-set miss counts are exposure diagnostics and do not
substitute for realized $M_{\mathrm{det}}$ and $M_{\mathrm{dep}}$
counts, which await the frozen per-violation classifier's
output.}
\label{tab:debt-ledger}
\scriptsize
\setlength{\tabcolsep}{2pt}
\begin{tabularx}{\textwidth}{@{}p{2.0cm}p{2.8cm}p{3.5cm}X@{}}
\toprule
Stage & Commitment-level cell & Observed now & Role in the bound \\
\midrule
Calibration & candidate transport slack & $\varepsilon_{\mathrm{cal}}^{U}=.131$ (O6) & population $\mathbb{E}[N_{\mathrm{irr}}](\eta+\varepsilon_{\mathrm{cal}})$ \\
Detection & $M_{\mathrm{det}}$ & proxy $U_{\mathrm{det}}$: .650/.413 & $\mathbb{E}[M_{\mathrm{det}}]$; classifier \\
Dependency reach & $M_{\mathrm{dep}}$ & misses in 2/58 missions & $\mathbb{E}[M_{\mathrm{dep}}]$; classifier \\
Invocation scheduling & $M_{\mathrm{inv}}$ & contract invariant & $\mathbb{E}[M_{\mathrm{inv}}]$; classifier \\
Restoration & failed-restoration cell & $10/97$; $f_{\mathrm{restore}}^{U}=.165$ & population $\mathbb{E}[N_{\mathrm{rep}}]f_{\mathrm{restore}}$ \\
\bottomrule
\end{tabularx}
\end{table}

\paragraph{Finding 4: population estimands and empirical sensitivities.} Localized
compensation was invoked 97 times. The technical
compensation-failure rate is 7 of 97 (0.072, one-sided upper
0.128). The population quantity in
Theorem~\ref{thm:composition}(iii), and the endpoint the frozen O3
numerator names, is restoration by
mission end: 10 of 97 invoked repairs had not restored consistency
by then (the 7 technical failures and 3 further repairs whose
restoration did not complete in time), giving
$\widehat f_{\mathrm{restore}} = 0.103$ with one-sided Wilson upper bound
$f_{\mathrm{restore}}^{U} = 0.165$. The 97 invocations cluster in
58 missions; the ordinary one-sided Wilson endpoint is 0.165 and a
mission-stratified cluster bootstrap gives 0.161, and we report the
larger 0.165 as a conservative sensitivity choice, not as a claim
that the Wilson interval is cluster-valid. Accordingly 0.165 is a
plug-in sensitivity value, not a nominal theorem certificate; the
relevant endpoint is nevertheless restoration failure, not the
technical-failure value 0.128. All 12 selective-calibration strata met their
identification floors, with simultaneous upper bounds from 0.060 to
0.131, so $\varepsilon_{\mathrm{cal}}^{U} = 0.131$ (O6), again as
a sensitivity upper limit; three
point estimates exceed 0.05, so the gate-cleared stratum is not
certified at the aggregate precondition's level. O6 licenses A1
only if its hazard-specific event target, recorded score, conditioning
stratum, and claim weighting equal A1's displayed transport. Under the audit-decomposition form of
Theorem~\ref{thm:composition}(iii), neither A2 nor A6-C is
assumed; their failures are priced as debts, each defined as a
deterministic reclassification of already-declared observables (O2,
O3, O5, O7) that the frozen per-violation classifier computes; the
present numbers are channel proxies, not classified cells.
The seed-level aggregates measure each debt channel's exposure but
do not yet classify the violations into cells. Undetected-invalidation exposure, written $U_{\mathrm{det}}$,
runs 0.650 per mission under adaptive refresh against 0.413 under
fixed. $U_{\mathrm{det}}$ counts invalidation events, not violated
commitments: $M_{\mathrm{det}}$ is by definition a subset of the
residual violations (0.300 and 0.838 per mission) and cannot exceed
them, and because one undetected invalidation can strand several
commitments, $U_{\mathrm{det}}$ does not bound $M_{\mathrm{det}}$
in general; the classifier, not the proxy, supplies the cell. Failed invoked restorations run 0.025 against 0.100 per localized
mission and, by the one-commitment-per-invocation property, bound
the restoration cell, with
$\widehat{\mathbb{E}}[N_{\mathrm{rep}}]\,
\widehat f_{\mathrm{restore}} = 10/160 = 0.0625$ exactly as a
sample-mean count of failed invocations. Exposure to the two
channels runs in opposite directions across refresh policies; how
the residual violations actually distribute over the cells is
precisely what the frozen per-violation classifier reports; until
it does, each exposure is reported as an exposure and no cell count
is claimed. RQ6 likewise does not establish exact A6:
41 of the 58 repair missions recovered exactly the true affected
set. Restoration correctness is conditional on A6-C, whose recall
misses occurred in 2 of the 58 repair-invoking missions; whether
those misses touch one invocation or several is an invocation-level
attribution the seed-level archive does not support, so no
invocation-level rate is quoted. Exact repair minimality remains
conditional on A6-S.
The ten unrestored invocations decompose as seven technical
failures and three late completions; whether the two recall-miss
events intersect them is not derivable from seed-level aggregates,
so $M_{\mathrm{dep}}$ remains a separate term until the classifier
reports the union. Tightening the calibration constant is a
calibration task. Closing the A6 gaps is a dependency-modeling and
validation task. Neither is a proof task.

\paragraph{Answer to RQ6.} The theorem defines the complete
accounting; the campaign supplies calibration and restoration
sensitivities and exposure diagnostics, but does not yet numerically
instantiate every term, and it does not deliver the strongest
empirical inequality or certify A7. Adaptive refresh
improves pre-authorization freshness and mission outcomes on this
stream; localized saga repairs after authorization with
substantially less work, latency, and cost than global recovery;
recovery cannot change stale execution, exactly as required; and
localization does not reduce residual violations relative to global
recovery. RQ6 therefore supports a stage-separated consistency
contract and an efficiency claim, not a claim that the composed arm
uniquely minimizes residual violations. Under the frozen decision
rules no frontier-superiority statement is licensed, and none is
made.

\subsection{Secondary studies: RQ2--RQ5 and re-litigation}
\label{sec:eval-secondary}

The remaining questions support attribution, outcome accounting,
component cost, and guard correctness; each is summarized here with
full detail in the appendices.

\paragraph{RQ2 (mechanisms).}\label{sec:eval-rq2}

Reported in full in \extpointer{app:rq2-main}. Across 160
missions crossing four leave-one-family-out trigger variants with
three routing rules, only two leave-one-family-out ablations show
identified effects: removing
predicted triggers raises cost by 1.592 and stale execution by 0.361
percentage points, and removing normative triggers cuts coverage by
11.64 points and completion by 5.89 points; channel adequacy blocks
cheap but uninformative structural checks, while cheapest-adequate
and maximum-net-value routing are indistinguishable on this stream.
The mechanism is plural, the trigger families do different jobs, and
no pooled near-boundary explanation survives the quantile
diagnostics.

\paragraph{RQ3 (rescue outcomes).}\label{sec:eval-rq3}

Development accounting and sensitivity are in
\extpointer{app:rq3-dev}. On the 80 held-out missions the answer is a reversal
the theory prices: adaptive refresh completes 4.35 fewer people than
fixed $k=600$ and 6.05 fewer than the validity clock, both
simultaneous bands excluding zero, while reducing stale loss against
both (against fixed, $-15.0$ with band excluding zero) and showing no
distinguishable weighted-loss difference, both bands including zero.
Lower technical staleness does not by itself establish better mission
outcome, and the projected hospital time remains a common geography
bridge rather than a policy-specific clinical outcome; the reversal
is consistent with the regime analysis of
Corollary~\ref{cor:price}, and the record prices it.

\paragraph{Held-out process and accounting.}
We replay the three validation-frozen R-B points over all 80 test
seeds. Every mission conserves the 240 generated people among
delivered to the safe transfer dock, picked up but still in transit,
and not picked up; pickup is never relabeled as completion, and the
call-to-hospital quantity remains a projection adding the measured
2079-s geography bridge, not an observed policy-specific clinical
outcome. The paired max-$t$ family contains 12 contrasts and uses
10,000 draws.

\begin{table}[htbp]
\centering
\caption{RQ3 held-out outcome means, 80 missions per frozen policy.
Time is projected call-to-hospital; weighted loss uses the frozen
consequence composite.}
\label{tab:rq3-heldout}
\begingroup
\setlength{\tabcolsep}{4pt}
\begin{tabular}{@{}lrrr@{}}
\toprule
\textbf{Metric} & \textbf{Adaptive $\varepsilon_c{=}4$} &
\textbf{Fixed $k{=}600$} & \textbf{Clock $\alpha{=}.4$} \\
\midrule
Completed people & 198.75 & 203.10 & 204.80 \\
Not picked up & 37.90 & 35.40 & 34.55 \\
Projected time (s) & 2118.52 & 2118.06 & 2175.69 \\
Stale loss & 6.25 & 21.25 & 778.75 \\
False holds & 4.4625 & 0.0000 & 1.5375 \\
Weighted loss & 7646.40 & 7103.76 & 7730.53 \\
\bottomrule
\end{tabular}
\endgroup
\end{table}

\begin{table}[htbp]
\centering
\caption{RQ3 held-out paired differences, adaptive minus comparator,
with simultaneous 95\% max-$t$ bands in brackets.}
\label{tab:rq3-heldout-paired}
\begingroup
\setlength{\tabcolsep}{2pt}
\begin{tabular}{@{}lrr@{}}
\toprule
\textbf{Metric} & \textbf{vs.\ fixed $k{=}600$} & \textbf{vs.\ clock $\alpha{=}.4$} \\
\midrule
Completed people & $-4.35$ {\scriptsize[$-6.97$,$-1.73$]} & $-6.05$ {\scriptsize[$-8.79$,$-3.31$]} \\
Not picked up & 2.50 {\scriptsize[$-1.83$,6.83]} & 3.35 {\scriptsize[$-.16$,6.86]} \\
Projected time (s) & .462 {\scriptsize[.44,.49]} & $-57.173$ {\scriptsize[$-58.3$,$-56.1$]} \\
Stale loss & $-15.0$ {\scriptsize[$-28.2$,$-1.8$]} & $-772.5$ {\scriptsize[$-1084$,$-461$]} \\
False holds & 4.463 {\scriptsize[4.01,4.92]} & 2.925 {\scriptsize[1.95,3.90]} \\
Weighted loss & 542.6 {\scriptsize[$-323$,1408]} & $-84.1$ {\scriptsize[$-793$,625]} \\
\bottomrule
\end{tabular}
\endgroup
\end{table}

\paragraph{Held-out answer to RQ3.}
Adaptive refresh does not improve completed rescues. It completes 4.35
fewer people than fixed $k=600$ and 6.05 fewer than the validity clock;
both simultaneous bands exclude zero. It reduces stale loss relative to
both, but creates more false holds and has no distinguishable
weighted-loss advantage because both weighted-loss bands include zero.
Thus lower technical staleness does not by itself establish better
mission outcome. End-to-end hospital claims still require
policy-specific geography and clinical modeling rather than the common
bridge used here.

\paragraph{RQ4 (audit-layer cost).}\label{sec:eval-rq4}

Reported in full in \extpointer{app:rq4-detail}. On two frozen
traces, the online buffered append path costs about 0.1 ms per TRACE
version (56.9 ms across the representative 492-version replay), gate
evaluation adds microseconds, and full-chain verification is linear
at 80.45 microseconds per version ($R^2=0.999999$), which places
incremental appends online and full-chain scans offline or at
checkpoints. The predeclared high-volume stress trace (16,334
versions) costs 2.40 s and 61.2 MiB, so record volume, not
single-append latency, is the quantity to watch. The benchmark is a
single-process, buffered measurement, not a durable-commit one; it
supports a modest component-cost claim and nothing stronger.

\paragraph{Held-out R-C and pre-registered re-litigation.}\label{sec:eval-relit}

The final stress campaign runs the validation-frozen adaptive
$\varepsilon_c=4$ and fixed $k=600$ points over common seeds
1001--1040; all 80 missions complete and every event and TRACE chain
verifies. Under the six-contrast paired max-$t$ family, adaptive
reduces the stale-rate point estimate but does not separate from
fixed on staleness or weighted stale loss, while raising cost,
lowering coverage, and rescuing fewer people with bands excluding
zero; the campaign table and case detail appear in
\extpointer{app:rc}.

Before any R-C outcome is opened, the registered case rule selects
the first adaptive stale execution: seed 1013 at $t=3241$ s, boat
group 027 south, authorized by \fld{trace-84c9ca78d3616c6a0e64}
version 2. Replay reconstructs the authorization, evidence age,
belief, and trigger state exactly, and the two registered
single-artifact repairs, zero evidence latency and a tighter
innovation tolerance, both preserve the same stale proposal; the
six-step replay appears in \extpointer{app:relit-heldout}. We report
nonclosure rather than searching post hoc for an unregistered
flipping intervention: a negative causal result and a positive
reproducibility result, since the record makes the boundary of the
supported explanation explicit.

\paragraph{RQ5 (revalidation-guard stress test).}\label{sec:eval-rq5}

The campaign evaluates the revalidation guard of
Section~\ref{sec:gate-conditions} under the same freeze discipline:
paired seeds run a no-change control and a mid-mission
version-replacement scenario (recorded as a structural event, the
new version initially unqualified for high-consequence classes),
with and without the guard. The pre-registered invariant is that
under the guard no high-consequence commitment clears on a
superseded or unqualified version; the control bounds guard
overhead, and what the campaign prices is coverage and completion
cost. The measured contrasts appear in
Table~\ref{tab:rq5-heldout}; the per-seed measures and a
representative replacement timeline are in the artifact's RQ5
evidence package.

\begin{table}[htbp]
\centering
\caption{RQ5 held-out replacement stress test (80 paired seeds).
Differences are guard minus no guard; brackets are simultaneous 95\%
max-$t$ sensitivity bands over all 10 reported RQ5 contrasts, with
endpoints rounded outward to two decimals; full precision is in the
artifact.}
\label{tab:rq5-heldout}
\begingroup
\let\small\footnotesize
\setlength{\tabcolsep}{1.5pt}
\begin{tabular}{@{}lcr@{}}
\toprule
Outcome & No guard\,$\to$\,guard & Diff.\ [95\% band] \\
\midrule
Unsafe high-conseq.\ clears
  & $3.688 \to 0$ & $\mathbf{-3.688}\,[-4.30,-3.08]$ \\
Held pending revalidation
  & $0 \to 3.688$ & $\mathbf{+3.688}\,[+3.08,+4.30]$ \\
Added verification cost
  & $0 \to 4.021$ & $\mathbf{+4.021}\,[+3.35,+4.69]$ \\
Mission completion
  & $1.00 \to 0.95$ & $-0.050\,[-0.14,+0.04]$ \\
Rescued people
  & $3.45 \to 3.28$ & $-0.175\,[-0.50,+0.15]$ \\
\bottomrule
\end{tabular}
\endgroup
\end{table}

\paragraph{Results.} The revalidation invariant held in every
held-out trace (seeds 3101--3180, 80 paired seeds per arm and
scenario). In the replacement scenario the unguarded gate issued
3.688 high-consequence clears per mission on a superseded or
unqualified version; the guard reduced the count to zero ($-3.688$,
band $[-4.291, -3.084]$) and converted the same number into
explicit holds and escalations, at 4.021 units of additional
verification cost ($[+3.359, +4.684]$)
(Table~\ref{tab:rq5-heldout}). Completion fell from 1.000 to 0.950
and rescued people from 3.450 to 3.275; both simultaneous bands
include zero while the individual paired intervals exclude it, so a
mission-level cost is possible but not established under the
conservative family. In the no-replacement control the guard added
no measured overhead, and every model-version and gate transition
replayed exactly. The experiment supports a correctness and
accountability claim: the guard does not make problematic
commitments disappear; it converts unsafe clears into explicit,
priced holds.

\subsection{Falsification criteria}\label{sec:eval-falsification}

Declared before any run, and scoped: dominance is falsified if any
validation-tuned fixed-$k$ point or the budget-matched clock
dominates the adaptive frontier on both RQ1 axes beyond the
clustered intervals; the audit claim if refresh decisions prove
unattributable to their declared conditions; and the calibration
precondition if the deployed package exceeds its declared ECE bound
per horizon or class. A null frontier with an intact audit trail
leaves the contract contribution standing.
The RQ5 invariant is falsified if any held-out trace shows a
high-consequence commitment cleared on a superseded or unqualified
predictor version.

The held-out comparison resolves these criteria. Unconditional
adaptive dominance is falsified: adaptive significantly reduces
staleness relative to fixed $k=600$ but is significantly worse on
cost, coverage, completion, and rescued people, while the validity
clock is far costlier and staler yet completes significantly more
rescues; RQ3 likewise falsifies a held-out adaptive rescue-advantage
claim, with weighted-loss bands including zero. The marginal
proposal-population calibration precondition of
Section~\ref{sec:eval-calibration} is not falsified: overall and
action-class ECE upper bounds remain below $.05$ without refitting.
The stronger selective guarantee for RQ6's gate-cleared policy by
consequence by horizon strata is not established: all simultaneous
upper bounds exceed $.05$, and three point estimates exceed it
(Section~\ref{sec:eval-rq6}, O6); equivalence with A1's exact
hazard event, recorded score, conditioning stratum, and weighting must
also be verified before substitution.
The audit claim is
not falsified, because the selected decision is fully
reconstructable and attributable; the stronger causal-repair claim
fails because neither registered counterfactual flips the held-out
plan. RQ4 is not a durability test, so it neither establishes nor
falsifies production storage overhead.

\section{Discussion and Limitations}\label{sec:discussion}

Adaptive refresh cannot detect a change absent from its process model
or visible evidence; support, OOD, and normative gates bound where
the predictor may authorize action, but thresholds remain declared,
versioned institutional choices. The audit guarantee is therefore
specific: the record reconstructs what the system represented and why
it acted, not whether the representation was true.
A complementary post-deployment learning contract asks whether recorded
failures identify a recurrent cause, revise decision-generating state,
and transfer to unseen instances from the same causal class
\cite{chang2026causalguidance}. \trw{} can supply discrepancy, timing,
and provenance evidence for that process, but replay and compensation
alone do not establish diagnosis or transferable improvement.
A related mechanism is deferred: escalating repeated
out-of-tolerance innovations into a structural
predictor-inadequacy condition is a natural companion to the
revalidation guard, but it is an empirical detection claim
requiring its own calibrated false-escalation budget and frozen
campaign, and we do not report it here.

Calibration is equally scoped: the unchanged map transfers to 62,977
held-out horizon-1 dispatch and evacuation route claims but does not
certify longer horizons, other actions, support/OOD scores, or
another scenario; most development mass lies near probabilities zero
and one; and the predictor remains a transparent surrogate rather
than the JEPA-class model the interface targets.

The empirical results expose two estimand traps. A conservative
controller can reduce stale execution by authorizing little, so
every staleness result is paired with coverage; and lower stale loss
is not itself a rescue benefit, since completion means delivery to
the safe transfer dock, hospital time is a projection over one
measured bridge, and both held-out weighted-loss contrasts include
zero. An operator would choose the adaptive point when the
verification budget binds, when stale execution carries regulatory
weight beyond its rescue cost, or when audit obligations are
non-negotiable; where none of these binds, the record supports
choosing fixed $k=600$, and that too is the contract working.

RQ2 explains mechanism roles without establishing routing
superiority, and RQ4 is a single-process, buffered component
benchmark, not durable-service latency; both scopes are stated
where the results are.

The consistency reading spans the stack, and RQ6 measures its
composition on one incident stream: the adjudication licenses
prevention and repair economy, adaptive refresh reduced stale
authorization and residual violations while localization matched
global recovery's residual and rescue outcomes at a fraction of the
repair work and latency, and it refuses the strict synergy claim,
since localization did not further lower residual violations. The
corrective half's internal guarantees remain evaluated in the SagaLLM
campaigns~\cite{chang2025sagallm}, not re-derived here, and the
composition on a second domain, under production durability, remains
open.

\paragraph{Predictor substitution.} We deliberately defer
learned-predictor substitution rather than attach an uncalibrated
semantic probe after observing the present campaigns. Our next
frozen campaign compares the transparent surrogate, an unqualified
V-JEPA predictor, and a development-calibrated V-JEPA predictor
under the same \system{} gate. It tests whether unqualified outputs
are converted to HOLD or ESCALATE rather than unsafe CLEAR
decisions, whether calibration restores in-support coverage while
preserving OOD refusal, and whether detection and restoration debts
remain equivalent under matched evidence, shocks, and repair
instances; closed-loop debt rates are reported separately, because
predictor-dependent selection can change the commitment and repair
mixture. The protocol is frozen and hashed as a next-campaign
commitment; a hash proves temporal freezing, not third-party
preregistration.

Finally, the held-out re-litigation reconstructs the disputed
decision but does not close its cause: no registered shock intersects
the claim interval and neither counterfactual flips the plan.
Production durability, external chain anchoring, a second domain, and
channel-stress evaluation remain open. Transfer beyond Flood-SAR is a
design argument, not an empirical claim: the three motivating settings
of Section~\ref{sec:intro} share exactly the structure
Definitions~\ref{def:view}--\ref{def:refresh} formalize, priced
channels, irreversible reads, and dependency chains, and what changes
across them is the channel menu and the loss table, which are policy
inputs rather than architecture.

\section{Future Work}\label{sec:future}

The formal interface is now explicit, but the empirical distance
between a debt decomposition and a population bound remains open.
Theorem~\ref{thm:composition}(iii) makes that work enumerable: each
unverified interface appears either as a named debt cell or as a
population parameter whose sampling uncertainty must be controlled.

\paragraph{Calibration transport remains a sensitivity.}
RQ6 declares consequence-specific risk limits $0.12$ and $0.20$, not
$0.05$; the latter is the calibration precondition's threshold. Adding
the reported diagnostic $\varepsilon^{U}_{\mathrm{cal}}=0.131$ gives
classwise plug-in sensitivities $0.251$ and $0.331$, and the aggregate
display in Theorem~\ref{thm:composition}(i) uses the supremum $0.20$.
None of these numbers is presently a certificate: the $0.131$ limit is
not yet a joint mission-cluster-valid bound on A1's score-subtracted
transport slack, and O6's exact event and weighting alignment has not
been verified. Algebraically, $\eta$ and
$\varepsilon_{\mathrm{cal}}$ enter the bound with the same coefficient,
so an equal numerical reduction in either improves it equally; their
engineering costs differ, with smaller $\eta$ generally buying more
verification and smaller calibration slack requiring a better score on
the gate-cleared stratum.

\paragraph{Detection, dependency reach, and liveness remain open.}
The present campaign neither establishes A2's uniform detection ceiling
nor A6-C nor A7: detection coverage is $0.83$--$0.89$, two of 58
repair-invoking missions miss a true dependent, and late invalidations
or traversals can run past $T$. A future system can close these cells by
verifying A2, A6-C, and A7 as invariants, or a future campaign can retain
them as nonzero debts and estimate simultaneous population upper bounds
for $\mathbb E[M_{\mathrm{det}}]$,
$\mathbb E[M_{\mathrm{dep}}]$, and
$\mathbb E[M_{\mathrm{inv}}]$. Establishing A7 specifically requires
the joint distribution of invalidation arrival, detection latency,
affected-set size, enqueue success, and traversal completion relative
to $T$, including the fraction arriving inside the declared terminal
grace window; a mean latency is insufficient.

\paragraph{The per-violation classifier requires a new frozen campaign.}
The debt cells are produced by assigning each violation to the first
failed stage. A classifier frozen only after the present outcomes are
seen cannot support confirmatory population inference, because it can
move mass among cells without changing their total. The classifier,
tie-breaking rules, and version must therefore be frozen in the policy
trace before a new campaign. The present archive may be reclassified for
diagnosis, but those counts remain explicitly post hoc.

\paragraph{The specified statistical constructions await execution.}
Appendix~\ref{app:formal} specifies a mission-cluster simultaneous
construction for A1's score-subtracted transport slack and the
restoration-failure ratio, together with the definition checks needed to
show that O6 targets A1. A confirmatory run must recompute the frozen
estimands on whole-mission resamples, report the number of independent
missions and the intraclass dependence diagnostics, control the joint
coverage family that is actually substituted, and report Monte Carlo
error for its upper quantiles. Until both the coverage construction and
the event/score/stratum checks pass, the reported upper limits remain
sensitivity inputs.

\paragraph{Drift.}
The deadline lemmas are conditional on zero drift, exact observation, and
no intermediate evidence. For a point commitment under Brownian motion
with constant drift, the relevant terminal-exceedance probability is
still Gaussian but can be nonmonotone when the drift points away from
the boundary. For an extended commitment, the relevant object is the
first-passage distribution with drift. Passive observations can alter
either deadline. Flood dynamics have positive drift, which is why the
zero-drift formulas are used only to test the qualitative scheduling
implication; separate terminal and first-passage analyses are needed to
cover the application quantitatively and would also change the
accumulation calculation.

\paragraph{Predictor independence beyond independent increments.}
The floor and its predictor independence rest on increments no
functional of the observation filtration can access. That is exactly
true for Brownian motion and only approximately true for a
hydrologically routed river, whose level hours ahead is substantially
predictable from upstream gauges. The honest scope is the conditional one
Corollary~\ref{cor:predictor} states. How much of the floor survives
under predictable dynamics, and whether richer channels shrink it
materially rather than formally, is open.

\paragraph{Instantaneous versus cumulative objectives.}
The trigger enforces a one-step expected decision-error-loss budget at
the decision boundary. It is not claimed to optimize cumulative action
loss plus sensing cost over a mission, and the realized-path identity is a censored
two-sided bound rather than an expectation-level cost formula. An
optimal-stopping formulation of the joint objective, with the
accumulation guard as a constraint rather than a fallback, would replace
a defensible heuristic with a characterized policy. The accumulation
analysis makes this more pressing than it first appears: under the
sustained-validity hazard the unguarded schedule accumulates for the
interval-risk analogues of the RQ6 limits $0.12$ and $0.20$, so the
guard's optimality cannot be inferred from the deadline calculation
itself.

\paragraph{Generality of the evidence.}
The abstraction is stated for any world model supplying calibrated claims
and a valid deadline oracle, and wildfire perimeters, highway logistics,
laboratory automation, and market monitoring instantiate the same triple
of predictive claims, priced channels, and consequence classes. Only
Flood-SAR is evaluated, on real geography with simulated dynamics, so the
results speak to the contract's mechanics rather than to sensing
economics under real acquisition failures, channel contention, or
adversarial conditions. A second instantiation in a domain with
genuinely different channel prices and consequence structure is the
cheapest available test of the interface claim, and it would also expose
whether the four trigger families are complete or merely sufficient for
one application.

\paragraph{What would make this a certificate.}
Collecting the above, a nominal population inequality for new missions
from the declared mixture requires: an event-aligned conditional oracle
or a jointly covered A1 transport bound; verified A3--A5 module
contracts; either verified A2, A6-C, and A7 invariants or jointly covered
upper bounds for all three associated debt expectations; a classifier
and estimands frozen before the confirmatory campaign; and simultaneous
mission-cluster-valid upper limits for every estimated term, including
$f_{\mathrm{restore}}$. A6-S is additionally required for the
commitment-set minimality claim. These conditions produce the theorem's
declared population bound at its stated nominal coverage; they do not
certify dynamics, domains, or losses outside that declaration. Drift,
richer observation filtrations, cumulative objectives, and second-domain
evaluation concern reach rather than the validity of the present
conditional statement. The debt decomposition is useful precisely
because it keeps those two questions separate.

\section{Conclusion}\label{sec:conclusion}

\system{} treats a predicted world state as a materialized view over
the physical world and a commitment as a read whose authorization can
expire. Typed claims carry calibrated adequacy, validity, and refresh
conditions; adaptive triggers route priced verification; a
consequence-scaled gate decides what may proceed; and TRACE records
each basis, decision, and revision. The contribution is this contract
around the trigger, supported by the model-conditional deadline lemma
and its deployed guards.

The frozen experiments establish trade-offs, not dominance:
adaptive refresh lowers stale execution but loses coverage,
completion, and rescued people against the strongest fixed
interval; calibration transfers without refitting; RQ6 confirms
exact recovery noninterference and repair economy while refusing
the strict synergy prediction; and the revalidation guard converts
every unsafe clear into an explicit, priced hold.

The contribution is not an empirical proof that the deployed
system is correct. It is a consistency contract that makes each
proof-to-system assumption explicit, records the evidence relevant
to it, and measures its empirical slack: imperfect detection
coverage, approximate dependency soundness, observed completeness
misses, uncertified selective calibration, and late restorations
enter the record as auditable
quantities rather than exceptions, which is the discipline the
contract exists to enforce.

\FloatBarrier
\section*{Acknowledgments}
We thank Longling Geng for executing the revalidation-guard stress
campaign (RQ5) and Jaylan Roy for executing the composition
campaign (RQ6). We also thank our collaborators on the dual-Kalman
SIGMOD 2004 paper~\cite{jain2004adaptive} and the SagaLLM VLDB 2025
paper~\cite{chang2025sagallm}, which form the pillars of this
work.

\bibliographystyle{plainnat}
\bibliography{trace_worldmodel}
\clearpage

\ifextended
\appendix

\section*{Appendix Contents}
\noindent
\begin{tabular}{@{}l@{\quad}l@{\qquad}r@{}}
\textbf{\ref{app:figures}} & Supplementary Figures & \pageref{app:figures} \\
\textbf{\ref{app:formal}} & Formal Extensions and Theorem-Regime Checks & \pageref{app:formal} \\
\textbf{\ref{app:calibration}} & Calibration Development Diagnostics & \pageref{app:calibration} \\
\textbf{\ref{app:rq1-dev}} & Development Operating-Point Evidence & \pageref{app:rq1-dev} \\
\textbf{\ref{app:triggers}} & Trigger Taxonomy Reference & \pageref{app:triggers} \\
\textbf{\ref{app:rq2-main}} & RQ2 Mechanism Study & \pageref{app:rq2-main} \\
\textbf{\ref{app:rq2-dev}} & RQ2 Trigger Diagnostics & \pageref{app:rq2-dev} \\
\textbf{\ref{app:rq3-dev}} & RQ3 Development Outcomes and Sensitivity & \pageref{app:rq3-dev} \\
\textbf{\ref{app:rq4-detail}} & RQ4 Scaling Detail & \pageref{app:rq4-detail} \\
\textbf{\ref{app:rc}} & Final R-C Campaign & \pageref{app:rc} \\
\textbf{\ref{app:relit-heldout}} & Re-litigation Evidence & \pageref{app:relit-heldout} \\
\textbf{\ref{app:moved-system}} & Record-Store Mechanics and Provenance Scope & \pageref{app:moved-system} \\
\textbf{\ref{app:moved-mech}} & The Condition Hierarchy and the Refresh Rule in Practice & \pageref{app:moved-mech} \\
\textbf{\ref{app:artifact}} & Artifact Inventory & \pageref{app:artifact} \\
\end{tabular}
\bigskip

\section{Supplementary Figures}\label{app:figures}

Figure~\ref{fig:stack} sketches the reasoning stack and the
boundary it fixes; the operational console appears as
Figure~\ref{fig:ui} in the main text.

\begin{figure}[htbp]
\centering
\begin{tikzpicture}[
  node distance=2.2mm,
  layer/.style={draw=atpslate!70, rounded corners=1.5pt, align=center,
    font=\scriptsize, minimum height=7.5mm, text width=0.86\columnwidth,
    inner sep=2pt, fill=white},
  lab/.style={font=\scriptsize\itshape, text=atpslate}
]
\node[layer, fill=atpblue!8] (plan) {Planning and reasoning
(SagaLLM): commitment chains,\\ invalidation, rollback,
compensation};
\node[layer, below=of plan, fill=atpteal!10, draw=atpteal!70] (twm)
{\textbf{\system{}}: semantic elevation, signal economy,\\
consequence admission};
\node[layer, below=of twm, fill=atpamber!10] (trace) {TRACE record
store: typed, versioned,\\ append-only, hash-chained claims};
\node[layer, below=of trace] (wm) {World-model predictor: state,
uncertainty,\\ support, OOD};
\node[layer, below=of wm, fill=atpslate!6] (world) {Physical world
and priced observation channels};
\draw[<->, atpslate!80, thick] (plan) -- (twm)
  node[midway, right=1mm, lab] {conditions license decisions};
\draw[<->, atpslate!80, thick] (twm) -- (trace)
  node[midway, right=1mm, lab] {every step read/written};
\draw[<->, atpslate!80, thick] (twm.west) ++(0,-1mm) -- ([xshift=-0mm]wm.west |- twm.west);
\draw[->, atpslate!80, thick] (wm) -- (twm.south -| wm.north);
\draw[<->, atpslate!80, thick] (world) -- (wm)
  node[midway, right=1mm, lab] {observe / act};
\end{tikzpicture}
\caption{The reasoning stack. \system{} sits between the TRACE
record store below and transactional planning above; it neither
learns the predictor nor constructs the plan, and its output is the
set of claims a world model's predictions justify.}
\label{fig:stack}
\end{figure}
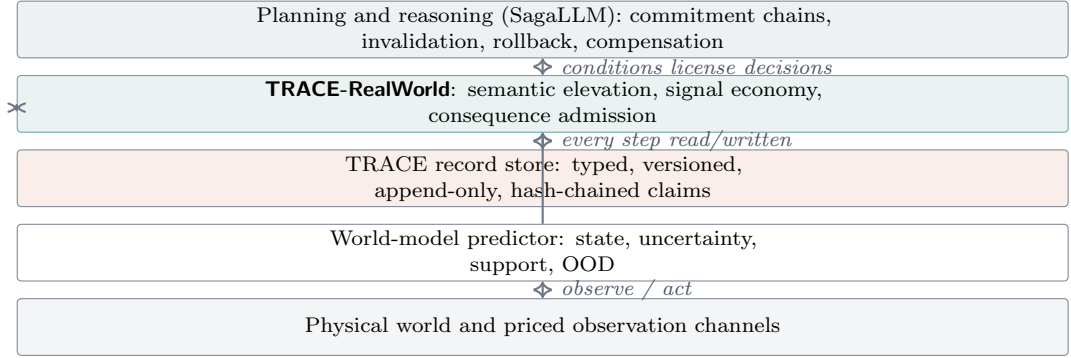


\section{Formal Extensions and Theorem-Regime Checks}\label{app:formal}

\paragraph{Realized-path cost identity.}
For the analytical unguarded scheduler that purchases an exact
observation at every admissibility deadline, let
$t_\infty=\lim_n t_n$, take the observation at $t_0=0$ as free
initialization, and let $N(T)=\max\{n:t_n\leq T\}$ count paid
re-verifications for $T<t_\infty$. Writing
$\eta_i^{\mathrm{pt}}:=\eta(a_i)\in\{\eta_f,\eta_s\}$ for the
point-event budget of the authorized action, the paid cost is $cN(T)$
and
\[
 \frac{1}{\sigma^2}\sum_{i=0}^{N(T)-1}
 \frac{d_i^2}{z_{\eta_i^{\mathrm{pt}}}^2}
 \leq T <
 \frac{1}{\sigma^2}\sum_{i=0}^{N(T)}
 \frac{d_i^2}{z_{\eta_i^{\mathrm{pt}}}^2}.
\]
This is a two-sided realized-path identity with the in-progress interval
censored at $T$, not an expectation-level cost formula. The deployed
guarded policy may HOLD or ESCALATE at a deadline and then has no next
$d_i$ on that branch, so the identity is not claimed for that policy.

\paragraph{Scheduler separation.}
For declared deadlines $\tau_1,\ldots,\tau_N$, take the initial
check at time zero as free and require a paid check at the closed
endpoint $S_N=\sum_i\tau_i$ to continue authorization there. The
deadline scheduler issues $N$ paid checks. A uniform-base scheduler
that honors every deadline must take $k\leq\tau_{\min}$ and requires
at least $\lceil S_N/k\rceil$ paid checks, allowing a shorter terminal
interval. Its check-count ratio is therefore at least
$\lceil S_N/\tau_{\min}\rceil/N\geq
\bar\tau/\tau_{\min}$, with equality possible. If the terminal check
is not required, the floor count replaces the ceiling and this ratio
need not follow. This proof is counting only; it does not compare
stochastic policies whose observations alter their later boundary
distances.

\paragraph{Zeno calculation and guards.}
In the symmetric-loss case $\Delta=L/2$, with
$\eta_f=\eta_s=\eta$, successive polled distances satisfy
$d_{i+1}=d_i|1-\xi_i/z_\eta|$ for standard normal $\xi_i$. For
$\eta<\eta_0\approx0.2602$, where
$\mathbb E\log|1-\xi/\Phi^{-1}(1-\eta_0)|=0$, the logarithmic
multiplier has negative mean and distances contract exponentially; the
epoch times can therefore accumulate. The workbench prevents continued
authorization in this regime: when the theoretical deadline is below
one tick, or the declared verification-rate or cost limit is exceeded,
the gate returns HOLD or ESCALATE. We make no optimality claim for this
fallback. Under the sustained-validity hazard of
Lemma~\ref{lem:sustained} the recursion becomes
$d_{i+1}=d_i|1-\xi_i/z_{\eta^{\mathrm{sus}}/2}|$. For iid Brownian
increments the same logarithmic-growth calculation gives almost-sure
accumulation when
$\eta^{\mathrm{sus}}<2(1-\Phi(z^{\ast}))\approx0.5204$, where
$z^{\ast}\approx0.6427$. If the numerical RQ6 limits $0.12$ and $0.20$
are separately declared as interval-risk budgets, both lie in this
analytic regime, so those extended-commitment analogues require the
tick, rate, and cost guards. This is a statement about the idealized
unguarded recursion, not a claim that the deployed fallback is optimal.

The deadline is model-conditional to exact observations, zero drift,
and no intermediate evidence. With constant drift, the point
commitment's terminal risk remains Gaussian and can be nonmonotone when
drift points away from the boundary, whereas an extended commitment
requires the first-passage distribution with drift. Passive observations
may advance or delay either deadline. The trigger controls instantaneous
decision-error loss and is not claimed to optimize cumulative action loss
plus sensing cost or to be robust to misspecification beyond the support
and OOD gates.

\paragraph{Synthetic regime test.}
The artifact instantiates the exact Brownian setup, checks the closed-form
deadline against the Gaussian flip probability, verifies both sides of
the censored cost identity, forces the sub-tick transition to HOLD, and
checks the finite-sequence scheduler count under the closed-span
convention. Flood-SAR has positive drift
and is used only to test the qualitative scheduling implication, not as
a numerical proof of the lemma.

\paragraph{Why the invocation-level limit is not admissible.}
Theorem~\ref{thm:composition} is stated with population parameters
$\varepsilon_{\mathrm{cal}}$ and $f_{\mathrm{restore}}$; substituting
estimated upper limits is licensed only on a joint coverage event, and
the limits must be valid for the design that produced them. The design
is clustered: a mission contributes many claims and many restoration
invocations, and invocations within a mission share a flood realization,
a predictor version, and a channel schedule. A Wilson interval over
pooled invocations treats them as independent trials and can therefore
be anticonservative when the intraclass correlation $\rho$ is positive.
Under the equal-size, exchangeable-correlation approximation the
variance inflation is the classical design effect
$1+(\bar n-1)\rho$. For illustration, $\bar n=8$ and $\rho=0.2$
would make a nominal two-sided $95\%$ normal interval cover about
$79\%$; that number is not a Wilson-coverage theorem and is not the
present design, whose observed mean is $97/58\approx1.67$ invocations
per repair-invoking mission. We therefore keep the pooled Wilson limit
only as a sensitivity input and report the empirical cluster structure
for any replacement interval.

\paragraph{A mission-cluster simultaneous construction.}
Let iid missions $m=1,\dots,M$ be the independent sampling units. For
mission $m$, let $A_m$ count gate-cleared irreversible commitments,
$B_m=\sum_i Z_{mi}$ count those stale under their declared hazard, and
$P_m=\sum_i\widehat r_{mi}$ be the sum of the risk scores actually
presented to the gate. Let $R_m$ count invoked restorations and $F_m$
those failing to restore by $T$. Assume finite second moments and
$\mathbb E[A_m]>0$, $\mathbb E[R_m]>0$. The minimal nonnegative A1
transport slack and the restoration estimand are
\[
 \gamma_{\mathrm{cal}}
 =\frac{\mathbb E[B_m-P_m]}{\mathbb E[A_m]},\qquad
 \varepsilon_{\mathrm{cal}}^{\star}
 =\max\{0,\gamma_{\mathrm{cal}}\},\qquad
 f_{\mathrm{restore}}
 =\frac{\mathbb E[F_m]}{\mathbb E[R_m]}.
\]
Thus $\varepsilon_{\mathrm{cal}}^{\star}$, not
$\mathbb E[B_m]/\mathbb E[A_m]$, is the score-subtracted parameter in
A1. A frozen bin-weighted selective ECE may be used instead if its
population definition provably upper-bounds this slack; in that case the
whole ECE calculation, including its bins, must be recomputed in every
resample.

The construction proceeds as follows. First freeze the estimand choice,
hazard labels, score extraction, bins if any, the per-violation
classifier, the complete simultaneous family, coverage target
$1-\alpha$, replicate count, and seed before the confirmatory campaign.
Second, for each of at least 10,000 replicates draw $M$ whole missions
with replacement and recompute
$\widehat\gamma_{\mathrm{cal}}^{\ast}
=\sum_m(B_m^{\ast}-P_m^{\ast})/\sum_m A_m^{\ast}$ and
$\widehat f_{\mathrm{restore}}^{\ast}
=\sum_mF_m^{\ast}/\sum_mR_m^{\ast}$, or recompute the frozen ECE in
place of the first ratio. A replicate with a zero denominator receives
the conservative upper value one and its frequency is reported; a
non-negligible frequency signals an inadequate sample rather than a
numerical nuisance. Third, for the two-term family allocate $\alpha/2$
to each one-sided percentile upper limit and set
$\varepsilon_{\mathrm{cal}}^{U}
=\max\{0,\gamma_{\mathrm{cal}}^{U}\}$. If estimated debt expectations
are also substituted, they join the same multiplicity family rather
than being analyzed separately. A predeclared cluster max-$t$ band over
studentized mission-level estimators is a less conservative alternative;
choosing between constructions after seeing their widths forfeits the
coverage claim.

Under the iid-mission and regularity assumptions, the resulting event
\[
 \mathcal C=
 \{\varepsilon_{\mathrm{cal}}^{\star}
      \le\varepsilon_{\mathrm{cal}}^{U}\}
 \cap
 \{f_{\mathrm{restore}}\le f_{\mathrm{restore}}^{U}\}
\]
has nominal asymptotic probability at least $1-\alpha$ for the
two-term Bonferroni family. The report must include the number of
independent missions, denominator diagnostics, intraclass dependence,
and Monte Carlo error of the upper quantiles. This is asymptotic
cluster inference, not a finite-sample guarantee; with the present
mission count and an unfrozen post-run construction, the values remain
sensitivity estimates.

\paragraph{Verifying that O6 instantiates A1.}
A1 names a random count, a target event, a semantic risk score, and a
weighting distribution; an observable estimating a different event, or
weighting it differently, does not instantiate it however well
calibrated. Five checks license the substitution, each a comparison of
definitions against the frozen record. Event and hazard identity: O6's
outcomes must be exactly the $Z_i^{h_i}$ events for gate-cleared
\emph{irreversible} commitments, using terminal truth for declared point
commitments and pathwise truth for declared extended commitments; the
two extracted index sets and labels must agree, not merely their counts.
Score identity: $P_m$ must sum the exact semantic stale-risk scores
$\widehat r_i$ presented to the gate before execution, rather than a
generic confidence, route-traversability score, or post-outcome
recalculation. Conditioning stratum: $A_m$ must count gate-cleared
irreversible commitments, not all scored claims; including held or
escalated claims changes the selected population. Estimator and
weighting identity: O6 must recompute either the score-subtracted ratio
$\sum_m(B_m-P_m)/\sum_mA_m$ or a predeclared bin-weighted ECE that
upper-bounds it, never an unweighted average of mission rates. Horizon
and consequence-class alignment: the statistic must be computed in the
declared cells and combined using the theorem's mission-mixture weights;
without frozen mixture weights, the supremum applicable cell slack is a
conservative alternative. Record all five verdicts as a contract test
beside the gate decisions. If any check fails,
$\varepsilon^{U}_{\mathrm{cal}}$ may still be reported as a diagnostic,
but the irreversible term must be quoted with its population parameter
and the failed check named.

\section{Calibration Development Diagnostics}\label{app:calibration}

The calibration pipeline collects every proposed dispatch and evacuation
route claim from the 19-configuration, 20-seed development grid; labels
continuous route openness by immutable event replay; holds out complete
seeds in five folds; freezes a weighted isotonic map; and then performs
smoke and wiring audits. Labeling all proposals avoids gate-induced
selective labels. Of 300,442 candidates, 169 are censored at mission
end. Figure~\ref{fig:app-calibration-process} summarizes the
pipeline; Figure~\ref{fig:calibration-reliability} gives the
ten-bin reliability diagram behind the sparse-middle-bin caveat;
Table~\ref{tab:app-calibration-primary} reports the per-stratum
development results and Table~\ref{tab:app-calibration-policy} the
policy-family diagnostic.

\begin{figure}[htbp]
\centering
\begin{tikzpicture}
\begin{axis}[
  width=0.8\columnwidth,
  height=0.36\columnwidth,
  xmin=0, xmax=1,
  ymin=0, ymax=1,
  xlabel={Mean predicted probability},
  ylabel={Observed frequency},
  tick label style={font=\scriptsize},
  label style={font=\small},
  grid=major,
  grid style={atpslate!18},
  legend style={font=\scriptsize, at={(0.98,0.02)}, anchor=south east,
    draw=none, fill=white, fill opacity=0.88, text opacity=1},
]
\addplot[atpslate, dashed, semithick, domain=0:1] {x};
\addlegendentry{Perfect}
\addplot[atpamber, semithick, mark=square*, mark size=1.8pt]
coordinates {(0.2172,0.0000) (0.4930,0.2094) (0.5334,0.9880)
  (0.6437,1.0000) (0.7150,1.0000)};
\addlegendentry{Raw}
\addplot[atpblue, thick, mark=*, mark size=1.8pt]
coordinates {(0.0003,0.0015) (0.1388,0.1308) (0.2590,0.8683)
  (0.3245,0.6918) (0.4462,0.4529) (0.5403,0.4479)
  (0.6436,0.7851) (0.7272,0.2806) (0.8565,0.8375)
  (0.9950,0.9962)};
\addlegendentry{Seed-grouped OOF}
\end{axis}
\end{tikzpicture}
\Description{Reliability diagram comparing raw and seed-grouped out-of-fold
route probabilities against the perfect-calibration diagonal. The sparse
middle bins deviate more than the high-mass extreme bins.}
\caption{Ten-bin reliability. ECE is count-weighted: 86.2\% of OOF claims
fall below 0.1 or above 0.9, while only 1.17\% fall between 0.2 and 0.8.
The sparse middle-bin deviations therefore bound local, not marginal,
claims.}
\label{fig:calibration-reliability}
\vspace{-.15in}
\end{figure}
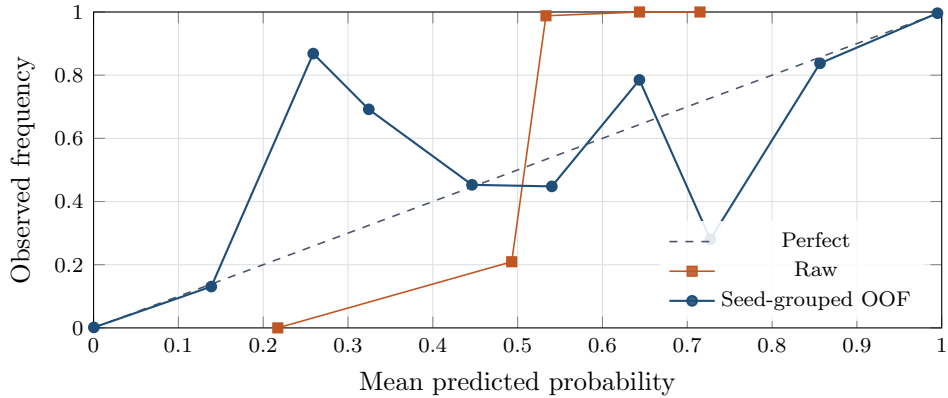

\begin{figure}[htbp]
\centering
\begin{tikzpicture}[
  node distance=2mm,
  calstep/.style={draw=atpslate, fill=atpslate!7, rounded corners=1.5pt,
    align=center, font=\scriptsize, inner sep=2.5pt, minimum width=23mm,
    minimum height=9mm},
  calarrow/.style={-{Latex[length=1.7mm]}, atpblue, semithick}
]
\node[calstep] (grid) {R-B grid\\19 configs, 20 seeds};
\node[calstep, right=of grid] (claims) {All proposed\\route claims};
\node[calstep, right=of claims] (labels) {Event replay\\interval labels};
\node[calstep, below=3mm of labels] (oof) {Five-fold OOF\\hold out seeds};
\node[calstep, left=of oof] (freeze) {Fit and freeze\\isotonic v2};
\node[calstep, left=of freeze] (audit) {Smoke plus\\380-run audit};
\draw[calarrow] (grid) -- (claims);
\draw[calarrow] (claims) -- (labels);
\draw[calarrow] (labels) -- (oof);
\draw[calarrow] (oof) -- (freeze);
\draw[calarrow] (freeze) -- (audit);
\end{tikzpicture}
\Description{Calibration workflow from the complete R-B development grid,
through all proposed route claims and event-replayed labels, to seed-grouped
out-of-fold evaluation, artifact freeze, smoke testing, and full audit.}
\caption{Calibration process. Complete-seed folds prevent mission records
from appearing on both sides; labeling all proposals prevents gate-induced
selective labels.}
\label{fig:app-calibration-process}
\end{figure}
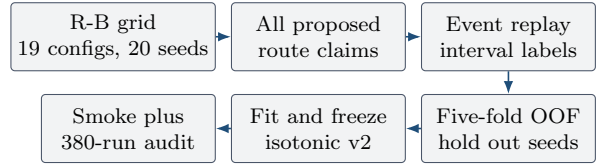

\begin{table}[htbp]
\centering
\caption{Development calibration. OOF folds hold out complete seeds;
intervals are 10,000-draw seed-cluster bootstraps.}
\label{tab:app-calibration-primary}
\begingroup\let\small\footnotesize
\begin{tabular}{@{}lrrrrrr@{}}
\toprule
Stratum & $n$ & Raw ECE & OOF ECE & ECE 95\% CI & OOF Brier & Brier 95\% CI \\
\midrule
Overall / horizon 1 & 300,273 & .32705 & .00477 & [.00383,.01434] & .02071 & [.01800,.02337] \\
Dispatch boat & 216,615 & .33802 & .00601 & [.00515,.01898] & .02628 & [.02286,.02988] \\
Evacuate to safety & 83,658 & .29864 & .00190 & [.00118,.00379] & .00628 & [.00440,.00820] \\
\bottomrule
\end{tabular}
\endgroup
\end{table}

\begin{table}[htbp]
\centering
\caption{Policy-family OOF diagnostic. The no-refresh interval crosses
the declared .05 threshold.}
\label{tab:app-calibration-policy}
\begin{tabular}{@{}lrrr@{}}
\toprule
Policy & $n$ & OOF ECE & ECE 95\% CI \\
\midrule
Adaptive & 32,030 & .00167 & [.00115,.01010] \\
Fixed-$k$ & 121,883 & .01464 & [.00844,.02673] \\
Validity clock & 103,760 & .01435 & [.00928,.02268] \\
No refresh & 42,600 & .04956 & [.02807,.07558] \\
\bottomrule
\end{tabular}
\end{table}

The failed fixed-$k=300$ pilot had grouped OOF ECE .00475 but
cross-policy smoke-test ECE .05537, so it was rejected before the full
rerun. In the final development map, 86.2\% of claims lie below .1 or
above .9 and only 1.17\% lie between .2 and .8; sparse middle-bin
deviations therefore prevent local or uniform calibration claims. The
frozen wiring audit verifies all 380 missions and 655,761 route-record
versions with zero score-range, map-reproduction, or confidence
mismatches. Its fitted-data ECE is only a wiring check; the OOF and
held-out results are inferential.
Figure~\ref{fig:calibration-reliability} shows the ten-bin
reliability diagram behind the sparse-middle-bin caveat.

\section{Development Operating-Point Evidence}\label{app:rq1-dev}

\begin{figure}[htbp]
\centering
\begin{tikzpicture}
\begin{groupplot}[
  group style={group size=2 by 1, horizontal sep=1.35cm},
  width=0.48\textwidth,
  height=0.38\textwidth,
  xlabel={Verification cost per mission},
  tick label style={font=\scriptsize},
  label style={font=\small},
  grid=major,
  grid style={atpslate!18},
]
\nextgroupplot[
  xmin=1.5, xmax=6.0,
  ymin=-0.03, ymax=0.80,
  ylabel={Stale execution rate (\%)},
]
\addplot[only marks, atpblue, mark=*, mark size=2.3pt]
coordinates {(2.0512,0) (3.8241,0) (3.1528,0) (2.4020,0)
  (3.5941,0) (3.8145,0)};
\addplot[only marks, atpamber, mark=square*, mark size=2.3pt]
coordinates {(2.7800,0.1003) (5.5700,0.3941)};
\node[font=\scriptsize, anchor=south west, atpblue]
  at (axis cs:2.0512,0.01) {$\varepsilon=.25$};
\node[font=\scriptsize, anchor=south west, atpamber]
  at (axis cs:2.7800,0.1003) {$k=600$};
\node[font=\scriptsize, anchor=south east, atpamber]
  at (axis cs:5.5700,0.3941) {$k=300$};

\nextgroupplot[
  xmode=log,
  xmin=1, xmax=400,
  ymin=0, ymax=52,
  ylabel={Execution coverage (\%)},
  legend style={font=\scriptsize, at={(0.03,0.04)}, anchor=south west,
    draw=none, fill=white, fill opacity=0.88, text opacity=1},
]
\addplot[only marks, atpblue, mark=*, mark size=2.3pt]
coordinates {(2.0512,33.3817) (3.8241,44.1144) (3.1528,40.6775)
  (2.4020,37.4297) (3.5941,42.0506) (3.8145,43.1452)};
\addlegendentry{Adaptive}
\addplot[only marks, atpamber, mark=square*, mark size=2.3pt]
coordinates {(14.4300,48.3620) (170.1100,42.6618)
  (5.5700,48.5407) (41.8500,47.2749) (333.6900,0.6509)
  (2.7800,47.8807)};
\addlegendentry{Fixed $k$}
\addplot[only marks, atpgreen, mark=triangle*, mark size=2.5pt]
coordinates {(87.4400,44.0961) (75.9000,43.3934)
  (67.4700,34.3958) (63.0600,15.1646) (64.1500,6.4012)
  (58.8400,2.9835)};
\addlegendentry{Validity clock}
\node[font=\scriptsize, anchor=south east, atpamber]
  at (axis cs:2.7800,47.8807) {$k=600$};
\node[font=\scriptsize, anchor=south west, atpblue]
  at (axis cs:3.8241,44.1144) {$\varepsilon=.5$};
\end{groupplot}
\end{tikzpicture}
\vspace{-.1in}
\Description{Two-panel RQ1 development plot. The left panel zooms into
the low-cost region and shows adaptive settings at zero observed stale
rate plus fixed k equals 600 and 300. The right panel shows that adaptive
settings have lower execution coverage than the strongest fixed settings;
validity-clock and all positive-cost fixed settings are also shown.}
\caption{RQ1 development point estimates. The left-panel cost--staleness
view favors adaptive refresh, but the right panel exposes the associated
coverage loss. The no-refresh arm has zero cost and staleness but only
1.41\% coverage and is reported in
Table~\ref{tab:app-rq1-selected}; zero cost prevents its inclusion
on the logarithmic axis.}
\label{fig:rq1-frontier}
\end{figure}
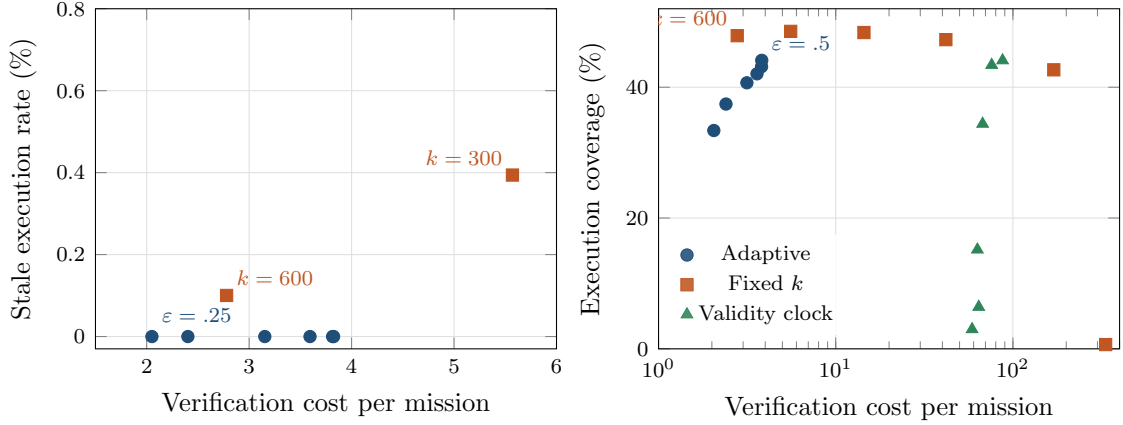

The 19-point R-B development grid uses seeds 1--20. All 380 missions and
chains verify, and event trajectories match across policies. These
results motivated the validation protocol but used preliminary coverage
accounting and are not confirmatory.
Figure~\ref{fig:rq1-frontier} plots the development
point estimates on both declared axes;
Table~\ref{tab:app-rq1-selected} reports the selected means and
Table~\ref{tab:app-rq1-paired} the paired development contrasts.

\begin{table}[htbp]
\centering
\caption{Selected RQ1 development means with 95\% seed-cluster intervals.}
\label{tab:app-rq1-selected}
\begingroup\let\small\footnotesize
\begin{tabular}{@{}llrrr@{}}
\toprule
Policy & Setting & Cost & Stale execution (\%) & Coverage (\%) \\
\midrule
No refresh & default & 0 [0,0] & 0 [0,0] & 1.41 [1.41,1.41] \\
Adaptive & $\varepsilon_c=.25$ & 2.051 [1.320,2.942] & 0 [0,0] & 33.38 [28.29,38.29] \\
Adaptive & $\varepsilon_c=.5$ & 3.824 [2.722,4.896] & 0 [0,0] & 44.11 [41.26,46.39] \\
Fixed & $k=600$ & 2.780 [2.680,2.880] & .100 [0,.253] & 47.88 [47.10,48.64] \\
Fixed & $k=300$ & 5.570 [5.380,5.760] & .394 [.108,.731] & 48.54 [47.73,49.27] \\
Clock & $\alpha=.10$ & 87.440 [70.150,103.900] & 5.366 [2.405,8.557] & 44.10 [40.87,47.33] \\
Fixed diagnostic & $k=5$ & 333.690 [325.660,341.020] & 41.997 [33.523,49.216] & .65 [.53,.77] \\
\bottomrule
\end{tabular}
\endgroup
\end{table}

\begin{table}[htbp]
\centering
\caption{Paired development differences; rate differences are percentage
points and brackets are 95\% paired seed-bootstrap intervals.}
\label{tab:app-rq1-paired}
\begin{tabular}{@{}lrrr@{}}
\toprule
Contrast & $\Delta$cost & $\Delta$stale & $\Delta$coverage \\
\midrule
$\varepsilon=.25-k=600$ & $-.729$ [$-1.450$,.152] & $-.100$ [$-.253$,0] & $-14.50$ [$-19.83$,$-9.32$] \\
$\varepsilon=.25-k=300$ & $-3.519$ [$-4.290$,$-2.588$] & $-.394$ [$-.731$,$-.108$] & $-15.16$ [$-20.67$,$-9.71$] \\
$\varepsilon=.5-k=600$ & 1.044 [$-.028$,2.076] & $-.100$ [$-.253$,0] & $-3.77$ [$-6.67$,$-1.53$] \\
\bottomrule
\end{tabular}
\end{table}

The development frontier already warned against unconditional
dominance: adaptive lowered point staleness but also coverage, no refresh
achieved zero staleness by executing only 1.41\% of opportunities, and
the latency-saturated $k=5$ diagnostic was pathological. Validation and
held-out tables in the main paper supersede these estimates.

\section{Trigger Taxonomy Reference}\label{app:triggers}

\begin{table}[htbp]
\centering
\begingroup
\setlength{\tabcolsep}{2.5pt}
\begin{tabular}{@{}>{\raggedright\arraybackslash}p{0.136\linewidth}
>{\raggedright\arraybackslash}p{0.35\linewidth}
>{\raggedright\arraybackslash\scriptsize}p{0.28\linewidth}
>{\raggedright\arraybackslash}p{0.185\linewidth}@{}}
\toprule
\textbf{Family} & \textbf{Fires when} & {\small\textbf{Policy fields}} & \textbf{Reason code} \\
\midrule
Predicted & validity shorter than commitment horizon, or predicted
uncertainty above the action bound & \texttt{max\_horizon}, \texttt{max\_uncertainty}, \texttt{max\_obs\_age} & scenario outcome \\
\addlinespace[2pt]
Observed & innovation from partial observations exceeds the declared
tolerance & \texttt{refresh\_condition}, tolerance & injected disruption,
else scenario \\
\addlinespace[2pt]
Structural & OOD above bound or support below floor: the view is
outside admitted competence & \texttt{min\_support}, \texttt{max\_ood} &
model transient, else scenario \\
\addlinespace[2pt]
Normative & policy or authority demands fresh evidence for the action
class & \texttt{authority} per class & contract violation if
bypassed \\
\bottomrule
\end{tabular}
\endgroup
\caption{The four trigger families, their parameterization, and the
reason codes under which their outcomes are attributed (short forms
of the identifiers listed in Section~\ref{sec:eval-method}). A
predicted trigger fires from the model alone, before any
observation; each family is a declared observable in the sense of
Section~\ref{sec:contract-observables}.}
\label{tab:triggers}
\vspace{-.2in}
\end{table}

\section{RQ2 Mechanism Study}\label{app:rq2-main}

\paragraph{Process.}
RQ2 uses the corrected R-C development protocol because it exercises
all four trigger families. Each of 20 seeds receives two registered,
policy-independent changes: a levee pulse at $t=2520$ s followed by a
keyed noisy passive depth report two seconds later, and a structural
debris closure at $t=3600$ s. The passive monitors are common telemetry,
not selectable evidence channels. At $\varepsilon_c=1$, we compare the
full four-family policy with four leave-one-family-out variants and,
holding triggers fixed, cheapest, cheapest-adequate, and
maximum-net-value routing. This produces 160 missions. All completed, all event and
TRACE chains verified, and the common-random-number check passed for all
20 seeds. Table~\ref{tab:rq2-effects} reports ablation-minus-full paired
differences using 10,000 bootstrap draws over the 20 seed pairs.

\paragraph{Finding 1: the trigger families perform different jobs.}
The full arm emits 91 triggers, with dispositions plotted in
\extpointer{app:rq2-dev}. Predicted triggers occur closest to the
continuous route boundary, median distance $0.0359$, and 9 of 25
invoke the rate guard rather than purchase evidence; observed
triggers occur in 19 missions exactly three seconds after the levee
pulse; structural triggers occur in 13 missions, the first always one
second after the debris closure, with 22 of 27 acquiring a drone
report and 5 withdrawing authorization because no channel has
positive net value; and normative triggers fire once per mission, far
from the physical boundary, because the rule requires fresh evidence
before an irreversible action. The quantile diagnostics collected in
\extpointer{app:rq2-dev} therefore support the declared family
separation, not a pooled ``near-boundary'' explanation; zero expected
decision-error loss for the observed and normative rows is expected,
since innovation and policy obligation, not the predicted decision-loss
threshold, own those triggers.

\begin{table}[htbp]
\centering
\caption{Paired RQ2 development effects. Rows are variant minus full; rates
are percentage-point differences; brackets are 95\% paired seed-bootstrap
intervals. The full-arm means are cost 6.057, stale execution 0.295\%,
coverage 23.87\%, and mission completion 44.39\%. Cheapest-adequate and
maximum-net-value routing are identical to the full arm on all four
measures.}
\label{tab:rq2-effects}
\begingroup
\setlength{\tabcolsep}{3pt}
\resizebox{\textwidth}{!}{%
\begin{tabular}{@{}lcccc@{}}
\toprule
\textbf{Variant$-$full} & $\Delta$\textbf{cost} &
$\Delta$\textbf{stale (pp)} & $\Delta$\textbf{coverage (pp)} &
$\Delta$\textbf{completion (pp)} \\
\midrule
No predicted & 1.592 [.441,2.864] & .361 [.081,.706] & 7.68 [$-.95$,16.39] & 1.21 [.36,2.25] \\
No observed & .340 [$-1.441$,2.293] & .003 [0,.008] & 1.02 [$-6.41$,8.23] & .22 [$-.38$,.88] \\
No structural & $-.250$ [$-.751$,0] & 0 [0,0] & $-1.86$ [$-5.59$,0] & 0 [0,0] \\
No normative & $-2.562$ [$-4.305$,$-.801$] & $-.173$ [$-.589$,.244] & $-11.64$ [$-19.38$,$-3.90$] & $-5.89$ [$-7.53$,$-4.34$] \\
Cheapest & $-5.227$ [$-6.889$,$-3.555$] & .323 [.081,.645] & 10.14 [3.63,17.90] & 0 [0,0] \\
Cheapest adequate; max net value & 0 [0,0] & 0 [0,0] & 0 [0,0] & 0 [0,0] \\
\bottomrule
\end{tabular}}
\endgroup
\end{table}

\paragraph{Finding 2: only two leave-one-family-out ablations show
identified effects.}
Removing predicted triggers increases cost by $1.592$ and stale
execution by $0.361$ percentage points, with both paired intervals
excluding zero; it also increases completion by $1.21$ points. In this
coupled controller, the predicted family therefore buys lower staleness
and lower downstream verification cost at a small completion cost. In
contrast, removing normative triggers reduces cost by $2.562$ but also
reduces coverage by $11.64$ points and completion by $5.89$ points. The
normative family is an evidence-enabling coverage mechanism here, not an
empirically separable stale-execution safeguard. Although observed and
structural triggers fire promptly, their aggregate leave-one-out
intervals include zero on the main RQ2 outcomes. With 20 seeds and
overlapping gates, their incremental mission-level effects are not
identified.

\paragraph{Finding 3: adequacy matters; net-value superiority is not
identified.}
Unconstrained cheapest routing lowers cost from $6.057$ to $0.830$ but
raises stale execution from $0.295\%$ to $0.618\%$ and coverage from
$23.87\%$ to $34.01\%$. It chooses the gauge for all 27 structural
triggers even though the recorded adequacy probability is zero. The
result is a cheaper, less conservative, and significantly staler policy,
not a free efficiency gain. Cheapest-adequate and maximum-net-value
routing, however, produce exactly identical per-seed cost, staleness,
coverage, completion, rescued-person, pending-person, and
weighted-stale-loss vectors. With the present two-channel scores, the gauge is selected
whenever adequate and the drone is the only adequate positive-net option
otherwise. This experiment supports the adequacy constraint but cannot
support superiority of maximum net value over cheapest adequate.

\paragraph{Answer to RQ2.}
The development mechanism is plural. Predicted triggers account for
an identified reduction in staleness and cost; normative triggers
preserve coverage and completion; and channel adequacy prevents cheap
but uninformative structural checks. The observed and structural
families are exercised and timely but have no detectable aggregate
incremental effect in this sample, and without a channel-stress grid
in which costs, adequacy, and VoI rankings cross, the two adequate
routing rules are structurally indistinguishable. We therefore retain
RQ2 as a completed mechanism study rather than a held-out routing
comparison; RQ3 separately tests whether these trade-offs improve
rescue outcomes rather than inferring benefit from trigger behavior.

\section{RQ2 Trigger Diagnostics}\label{app:rq2-dev}

\begin{table}[htbp]
\centering
\caption{Full-policy trigger diagnostics across 20 R-C missions. The
distance, expected decision-error loss, VoI, and delay entries are median [IQR]. Evidence
delay is conditional on acquisition; HOLD triggers have no delay value.}
\label{tab:app-rq2-triggers}
\begingroup
\setlength{\tabcolsep}{3pt}
\resizebox{\textwidth}{!}{%
\begin{tabular}{@{}lrlrrrr@{}}
\toprule
\textbf{Family} & \textbf{Count} & \textbf{Disposition} &
\textbf{Boundary dist.} & \textbf{Expected error loss} &
\textbf{Max VoI} & \textbf{Delay (s)} \\
\midrule
Predicted & 25 & 16 gauge, 9 HOLD & .0359 [.0311,.0434] & 1.025 [.895,1.865] & 14.375 [14.231,14.397] & 5 [5,5] \\
Observed & 19 & 19 gauge & .0984 [.0706,.1278] & 0 [0,0] & 14.550 [14.550,14.550] & 5 [5,5] \\
Structural & 27 & 22 drone, 5 HOLD & .1632 [.0465,.1632] & 42.500 [21.793,42.500] & 7.275 [7.275,10.820] & 10 [7,11] \\
Normative & 20 & 20 gauge & .3605 [.3605,.3605] & 0 [0,0] & 14.550 [14.550,14.550] & 5 [5,5] \\
\bottomrule
\end{tabular}}
\endgroup
\end{table}

Figure~\ref{fig:app-rq2-dispositions} (placed with the RQ3 outcome
chart at the head of Appendix~\ref{app:rq3-dev}) plots the
per-family dispositions, and the quantile entries of
Table~\ref{tab:app-rq2-triggers} substantiate the family separation reported in
Section~\ref{sec:eval-rq2}: observed triggers follow the passive
report, structural triggers follow the registered debris closure, and
normative triggers fire far from the physical boundary because the
action class requires fresh evidence, so their zero expected
decision-error loss is expected rather than anomalous.

\section{RQ3 Development Outcomes and Sensitivity}\label{app:rq3-dev}

Each R-B mission generates 240 people. A completed rescue requires
\fld{PEOPLE_DELIVERED} at the safe dock; pickup remains distinct. The
projected hospital time adds a single measured 2079-s geography
bridge and is not an observed clinical outcome. The bridge is the
realized pickup-to-hospital duration of one noise-free
transfer-chain run on the real geography: the schedule predicted
1883.08 s against 2079.00 s realized, an underprediction of
195.92 s (10.40\%), so the constant is measured, with a known
schedule error, rather than modeled.
Figure~\ref{fig:rq3-outcomes} shows the stacked
outcome accounting across the selected arms and boundary controls,
paired with the RQ2 trigger disposition chart of
Appendix~\ref{app:rq2-dev};
Table~\ref{tab:app-rq3-outcomes} reports the selected means and
Table~\ref{tab:app-rq3-components} the raw components and nominal
composite.

\begin{figure}[tb]
\centering
\begin{minipage}[t]{0.35\linewidth}
\centering
\setlength{\emergencystretch}{2em}
\begin{tikzpicture}
\begin{axis}[
  width=0.92\linewidth,
  height=0.88\linewidth,
  xbar stacked,
  bar width=7pt,
  xmin=0, xmax=32,
  xlabel={Triggers across 20 missions},
  symbolic y coords={Normative,Structural,Observed,Predicted},
  ytick=data,
  enlarge y limits=0.14,
  yticklabel style={font=\scriptsize},
  tick label style={font=\scriptsize},
  label style={font=\scriptsize},
  axis line style={atpslate!60},
  legend style={font=\scriptsize, at={(0.5,1.02)}, anchor=south,
    legend columns=-1, draw=none, /tikz/every even column/.append
    style={column sep=3pt}},
]
\addplot[xbar, fill=atpblue!70, draw=atpblue!80!black]
  coordinates {(16,Predicted) (19,Observed) (0,Structural)
    (20,Normative)};
\addlegendentry{Gauge}
\addplot[xbar, fill=atpteal!70, draw=atpteal!80!black]
  coordinates {(0,Predicted) (0,Observed) (22,Structural)
    (0,Normative)};
\addlegendentry{Drone}
\addplot[xbar, fill=atpamber!80, draw=atpamber!90!black]
  coordinates {(9,Predicted) (0,Observed) (5,Structural)
    (0,Normative)};
\addlegendentry{HOLD}
\end{axis}
\end{tikzpicture}
\Description{Horizontal stacked bars give the trigger counts per
family across 20 missions: predicted triggers split between gauge
and HOLD, observed triggers all resolve to gauge, structural
triggers split between drone and HOLD, and normative triggers all
resolve to gauge.}
\vspace{-.28in}
\caption{Full-policy R-C trigger disposition. HOLD means
authorization was withdrawn without evidence acquisition because
the rate guard fired or no channel had positive declared net
value.}
\label{fig:app-rq2-dispositions}
\end{minipage}
\hfill
\begin{minipage}[t]{0.62\linewidth}
\centering
\begin{tikzpicture}
\begin{axis}[
  width=0.93\linewidth,
  height=0.50\linewidth,
  xbar stacked,
  bar width=4.5pt,
  symbolic y coords={F5,None,V10,F15,F45,F300,F600,A7},
  ytick=data,
  yticklabels={$k=5$,No refresh,$\alpha=.10$,$k=15$,$k=45$,
    $k=300$,$k=600$,$\varepsilon=7$},
  yticklabel style={font=\scriptsize},
  xlabel={People per mission},
  xmin=0, xmax=250,
  enlarge y limits=0.05,
  tick label style={font=\scriptsize},
  label style={font=\scriptsize},
  grid=major,
  grid style={atpslate!18},
  legend style={font=\scriptsize, at={(0.5,1.02)}, anchor=south,
    legend columns=3, draw=none},
]
\addplot[fill=atpblue!78, draw=atpblue]
coordinates {(29.6,F5) (60.0,None) (201.6,V10) (202.4,F15)
  (201.6,F45) (201.0,F300) (200.0,F600) (197.8,A7)};
\addlegendentry{Delivered}
\addplot[fill=atpgreen!72, draw=atpgreen]
coordinates {(2.0,F5) (0.0,None) (3.6,V10) (3.0,F15)
  (3.2,F45) (2.4,F300) (2.8,F600) (2.4,A7)};
\addlegendentry{In transit}
\addplot[fill=atpamber!72, draw=atpamber]
coordinates {(208.4,F5) (180.0,None) (34.8,V10) (34.6,F15)
  (35.2,F45) (36.6,F300) (37.2,F600) (39.8,A7)};
\addlegendentry{Not picked up}
\end{axis}
\end{tikzpicture}
\Description{Horizontal stacked bars partition all 240 people in a
mission into delivered at the safe dock, picked up but still in
transit at the mission horizon, and not picked up. The adaptive
epsilon 7 arm completes fewer rescues than the selected fixed and
validity-clock arms. No refresh and fixed k equals 5 are boundary
controls with very few completed rescues.}
\vspace{-.1in}
\caption{RQ3 development outcome accounting; the boundary controls
run only in development, and held-out means appear in
Table~\ref{tab:rq3-heldout}. Pickup is kept distinct from
safe-transfer completion; each stacked bar sums to 240 people.}
\label{fig:rq3-outcomes}
\end{minipage}
\end{figure}

\begin{table}[htbp]
\centering
\caption{Selected RQ3 development means with seed-cluster intervals.}
\label{tab:app-rq3-outcomes}
\begin{tabular}{@{}lrrrr@{}}
\toprule
Policy & Completed & In transit & Not picked up & Projected time (s) \\
\midrule
Adaptive $\varepsilon_c=7$ & 197.8 [192.6,202.4] & 2.4 [1.6,3.2] & 39.8 [34.6,45.2] & 2118.52 [2118.48,2118.56] \\
Fixed $k=600$ & 200.0 [194.6,204.8] & 2.8 [2.0,3.6] & 37.2 [31.8,43.2] & 2118.06 [2118.06,2118.06] \\
Fixed $k=300$ & 201.0 [195.8,205.4] & 2.4 [1.6,3.2] & 36.6 [32.0,42.0] & 2119.49 [2118.06,2121.40] \\
Fixed $k=15$ & 202.4 [198.0,206.2] & 3.0 [2.2,3.6] & 34.6 [30.6,39.2] & 2118.20 [2118.09,2118.39] \\
Clock $\alpha=.10$ & 201.6 [197.0,205.6] & 3.6 [3.0,4.0] & 34.8 [30.8,39.6] & 2118.38 [2118.19,2118.63] \\
No refresh & 60 [60,60] & 0 [0,0] & 180 [180,180] & 2118.20 [2118.20,2118.20] \\
\bottomrule
\end{tabular}
\end{table}

\begin{table}[htbp]
\centering
\caption{Development raw components and nominal provisional composite;
verification cost remains a separate resource coordinate.}
\label{tab:app-rq3-components}
\begingroup\let\small\footnotesize
\begin{tabular}{@{}lrrrrrr@{}}
\toprule
Policy & Cost & Stale loss & Withdrawals & False HOLDs & Interruptions & Composite \\
\midrule
Adaptive $\varepsilon_c=7$ & 3.814 & 0 & 5.35 & 4.60 & 2.15 & 8022.1 \\
Fixed $k=600$ & 2.780 & 10 & .95 & 0 & .20 & 7452.9 \\
Fixed $k=300$ & 5.570 & 40 & .90 & 0 & .60 & 7362.7 \\
Fixed $k=45$ & 41.850 & 120 & 1.00 & 0 & 1.50 & 7163.0 \\
Fixed $k=15$ & 170.110 & 480 & 1.00 & 0 & 4.95 & 7403.0 \\
Clock $\alpha=.10$ & 87.440 & 620 & 9.40 & 3.40 & 6.75 & 7642.2 \\
No refresh & 0 & 0 & 0 & 0 & 0 & 36000.0 \\
\bottomrule
\end{tabular}
\endgroup
\end{table}

The provisional composite uses
\begin{align*}
L_w={}&w_s L_{\mathrm{stale}}+w_m N_{\mathrm{unpicked}}\\
&+w_a N_{\mathrm{withdrawn}}+w_h N_{\mathrm{false\mbox{-}HOLD}}.
\end{align*}
Across 144 combinations of $w_s\in\{.5,1,2,5\}$,
$w_m\in\{50,\allowbreak 100,\allowbreak 200,\allowbreak 400\}$,
$w_a\in\{0,3,10\}$, and $w_h\in\{0,10,50\}$, adaptive is never
lower than $k=600$ and is lower than $k=300$ in only 2 cases;
Table~\ref{tab:rq3-sensitivity} tabulates the grid. These
development preferences do not determine the frozen test weights.

\begin{table}[htbp]
\centering
\caption{Provisional composite-weight sensitivity. Negative margin favors
adaptive $\varepsilon_c=7$; one validity-clock setting is tied.}
\label{tab:rq3-sensitivity}
\resizebox{0.6\columnwidth}{!}{%
\begin{tabular}{@{}lrr@{}}
\toprule
\textbf{Comparator} & \textbf{Adaptive lower} &
\textbf{Margin range} \\
\midrule
Fixed $k=600$ & 0/144 (0.0\%) & [80.0, 1309.0] \\
Fixed $k=300$ & 2/144 (1.4\%) & [$-40.0$, 1534.5] \\
Fixed $k=45$ & 16/144 (11.1\%) & [$-370.0$, 2053.5] \\
Fixed $k=15$ & 60/144 (41.7\%) & [$-2140.0$, 2113.5] \\
Validity clock $\alpha=0.10$ & 89/144 (61.8\%) & [$-2890.5$, 1750.0] \\
\bottomrule
\end{tabular}
}
\end{table}

\section{RQ4 Scaling Detail}\label{app:rq4-detail}

This appendix carries the full RQ4 component benchmark summarized in
the main text.

\paragraph{Process and estimand.}
We benchmark two frozen seed-1 R-B traces before inspecting timing:
the RQ3-reported adaptive $\varepsilon_c=7$ arm, with 492 TRACE
versions and 11,743 events, and the predeclared latency-saturated
$k=5$ diagnostic, with 16,334 versions and 52,030 events; the former
is the representative context, the latter a volume stress test
ineligible for operating-point selection. Each context receives five
complete component trials in which every source version is written to
a fresh hash-chained repository and its gate inputs are evaluated by
the same policy engine; Table~\ref{tab:rq4-overhead} reports the
median of the five within-trace p50/p95/p99 estimates. Full-chain
verification is timed against parsed in-memory envelopes (cached) and
from a newly constructed repository that must parse and validate the
file (cold). Every source and replay chain verifies.

To isolate the online append path, ten additional paired replays use
the identical, preloaded record/evidence sequence and alternate
order: the audit-off side retains gate evaluation but suppresses the
repository write, the audit-on side adds the buffered hash-chain
append, so the paired difference estimates record-path append
overhead, not end-to-end mission slowdown. Timing used one Python
3.12 process on nine AMD EPYC 9V74 cores, Linux 6.12, and an ext4
temporary filesystem; \fld{TraceRepository.write} closes each append
but does not call \texttt{fsync}, so these are buffered, not
durable-commit, latencies.

\begin{table}[htbp]
\centering
\caption{RQ4 development component benchmark. Append quantiles are
the median-of-five-trial estimates in ms; paired total is the
ten-trial median [min,max] in ms; verification is in ms. Gate and
bytes-per-version diagnostics are in Table~\ref{tab:app-rq4-secondary}.}
\label{tab:rq4-overhead}
\begingroup
\setlength{\tabcolsep}{4pt}
\begin{tabular}{@{}lrr@{}}
\toprule
\textbf{Metric} & \textbf{Adaptive $\varepsilon_c{=}7$} &
\textbf{Fixed $k{=}5$ stress} \\
\midrule
TRACE versions & 492 & 16,334 \\
Repository size (MiB) & 1.956 & 61.154 \\
Append p50 / p95 / p99 & .115 / .170 / .213 & .111 / .172 / .231 \\
Paired total & 56.9 [56.5, 68.9] & 2399.5 [2184.2, 2888.6] \\
Verify cached / cold & 40.1 / 82.6 & 1323.5 / 3023.7 \\
\bottomrule
\end{tabular}
\endgroup
\end{table}

\paragraph{Finding 1: per-version online cost is small, but record volume
turns it into mission-scale work.}
Representative buffered append latency is 0.115 ms at p50 and 0.213
ms at p99; gate evaluation is another 0.0033 ms at p50. The paired
replay attributes 56.9 ms to appending all 492 versions, or 0.116 ms
per version; the $k=5$ trace has similar append quantiles, but its
33.2-fold larger version count raises the paired total to 2399.5 ms,
or 0.147 ms per version. The stress penalty is thus driven by version
volume and cumulative file work, not the single-append median,
mirroring the admission-path separation
of~\cite{chang2027mnemosyne}: gate computation, online recording,
and offline audit are different costs.

\paragraph{Finding 2: verification is linear and belongs off the online
critical path.}
Cached prefix verification from 10 to 16,334 versions gives a raw-scale
least-squares slope of 80.45 microseconds per version
($R^2=0.999999$); the scaling plot is Figure~\ref{fig:app-rq4-chain}.
Cached full verification takes 40.1 ms
for 492 versions and 1323.5 ms for 16,334; cold verification, including
file parsing and model validation, takes 82.6 and 3023.7 ms. Storage
rises from 1.956 to 61.154 MiB, while mean bytes per version remains
near 4 KiB. Full verification is therefore predictable but unsuitable
for synchronous per-commitment use; incremental appends remain online,
while full-chain scans are an offline or checkpointed operation.

\paragraph{Finding 3: the benchmark does not establish production durability or
whole-system overhead.}
The paired replay isolates buffered TRACE appends after inputs are
materialized; it excludes simulator, planner, network,
concurrent-writer, crash-recovery, and durable-flush effects and uses
one machine and one filesystem. Gate evaluation is reported as a
controller component, not misclassified as audit-only overhead.
Production claims require an audit-on/off mission harness, a durable
backend with an explicit flush policy, concurrency sweeps, and
hardware replication.

\paragraph{Answer to RQ4, in full.}
In this development implementation, the online buffered append path is
about 0.1 ms per TRACE version and adds 56.9 ms to the representative
record replay. The pathological high-volume trace costs 2.40 s and 61.2
MiB, and full-chain verification grows linearly. The result supports
modest per-version component cost and explicit offline verification; it
does not support a durable-database or end-to-end mission-overhead claim.

\begin{figure}[htbp]
\centering
\begin{tikzpicture}
\begin{axis}[
  width=0.8\columnwidth,
  height=0.41\columnwidth,
  xmode=log,
  ymode=log,
  xmin=8, xmax=20000,
  ymin=0.5, ymax=2000,
  xlabel={TRACE versions},
  ylabel={Cached verification (ms)},
  tick label style={font=\scriptsize},
  label style={font=\small},
  grid=major,
  grid style={atpslate!18},
  legend style={font=\scriptsize, at={(0.04,0.96)}, anchor=north west,
    draw=none, fill=white, fill opacity=0.88, text opacity=1},
]
\addplot[atpblue, mark=*, mark size=2.2pt, thick]
coordinates {(10,0.7540) (25,1.8965) (50,3.7688) (100,7.7435)
  (250,19.1130) (500,39.5137) (1000,79.1013) (16334,1313.5501)};
\addlegendentry{Five-trial median}
\addplot[atpamber, dashed, thick, domain=10:16334, samples=2]
  {0.080447*x};
\addlegendentry{Linear slope 80.45 $\mu$s/version}
\end{axis}
\end{tikzpicture}
\Description{Log-log plot of cached TRACE chain verification time against
chain length. Median time rises from 0.754 milliseconds for 10 versions
to 1313.55 milliseconds for 16,334 versions and closely follows a linear
fit with slope 80.45 microseconds per version.}
\caption{RQ4 chain-length scaling on the stress trace. The raw-scale
linear fit has $R^2=0.999999$; min-to-max five-trial variation at 16,334
versions is 1305.3--1475.6 ms.}
\label{fig:app-rq4-chain}
\end{figure}
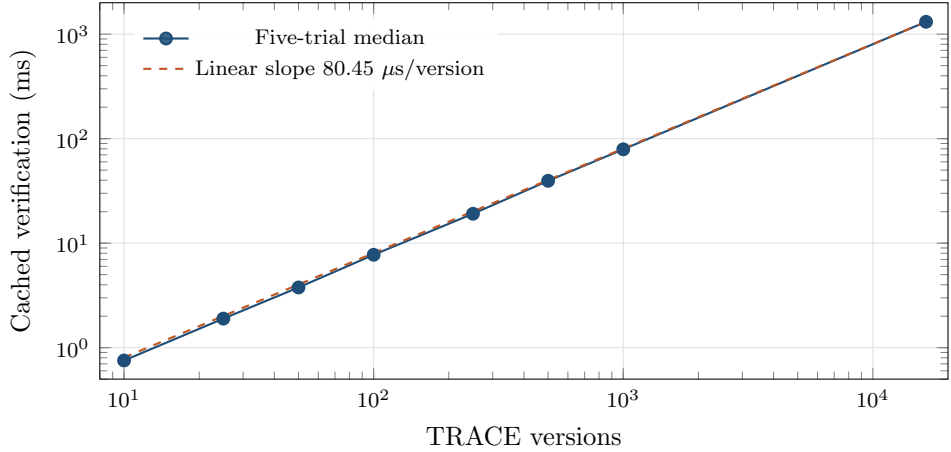

Figure~\ref{fig:app-rq4-chain} shows the cached-verification
chain-length scaling behind the 80.45 microsecond-per-version
slope;
Table~\ref{tab:app-rq4-secondary} reports the secondary storage and
gate-latency diagnostics.

\begin{table}[htbp]
\centering
\caption{Secondary RQ4 storage and gate diagnostics. Gate entries are
p50/p95/p99 ms over five trials.}
\label{tab:app-rq4-secondary}
\begin{tabular}{@{}lrr@{}}
\toprule
Context & Bytes/version & Gate p50/p95/p99 \\
\midrule
Adaptive $\varepsilon_c=7$ & 4168 & .0033/.0061/.0091 \\
Fixed $k=5$ stress & 3926 & .0029/.0059/.0089 \\
\bottomrule
\end{tabular}
\end{table}

\FloatBarrier
\section{Final R-C Campaign}\label{app:rc}

Table~\ref{tab:rc-heldout} reports the final held-out R-C means and
the adaptive-minus-fixed differences summarized in the main text.
For the registered case, both chains verify; no registered shock
intersects the record's $[2523,4323]$ validity interval, so
fixed-priority attribution is \fld{scenario_outcome}; evidence age
is 718 s, the recorded belief is open with confidence .95,
calibrated claim confidence is .97425, and the last innovation
.04420 is below the .15 tolerance. The registered repairs set
evidence latency to zero and tighten the tolerance to .05; each
preserves the same stale proposal, and the six-step replay follows
in Appendix~\ref{app:relit-heldout}.

\begin{table}[htbp]
\centering
\caption{Final held-out R-C means over 40 common seeds, with the
adaptive-minus-fixed difference and its simultaneous 95\% max-$t$
band in brackets. Rates are percentages.}
\label{tab:rc-heldout}
\begingroup
\setlength{\tabcolsep}{2pt}
\begin{tabular}{@{}lrrr@{}}
\toprule
\textbf{Metric} & \textbf{Adaptive} & \textbf{Fixed $k{=}600$} & \textbf{$\Delta$ [band]} \\
\midrule
Cost & 6.83 & 3.51 & 3.323 {\scriptsize[1.86,4.79]} \\
Stale (\%) & .413 & .969 & $-.556$ {\scriptsize[$-1.20$,.08]} \\
Coverage (\%) & 86.11 & 94.82 & $-8.706$ {\scriptsize[$-14.9$,$-2.5$]} \\
Rescued & 109.30 & 111.50 & $-2.20$ {\scriptsize[$-4.16$,$-.24$]} \\
Pending & 129.60 & 127.70 & 1.90 {\scriptsize[.03,3.77]} \\
Stale loss & 22.50 & 57.50 & $-35.00$ {\scriptsize[$-73.0$,3.0]} \\
\bottomrule
\end{tabular}
\endgroup
\end{table}
\FloatBarrier

\FloatBarrier
\section{Re-litigation Evidence}\label{app:relit-heldout}

The replay proceeds as follows: (1) both baseline chains verify;
(2) no registered shock intersects $[2523,4323]$, so attribution remains
\fld{scenario_outcome}; (3) record
\fld{trace-84c9ca78d3616c6a0e64} v2 authorizes the plan at confidence
.97425 from 718-s-old evidence; (4) innovation .04420 is below tolerance
.15 and the recorded belief remains open at confidence .95; (5) neither
zero evidence latency nor tolerance .05 suppresses the stale proposal;
and (6) the case remains causally open.

\subsection{Development diagnostic}\label{app:relit-dev}
The earlier seed-2 diagnostic selects the stale south-detour dispatch at
$t=2521$ s, one second after a registered levee breach and one second
before the two-second passive report arrives. Both chains verify.
Tightening the tolerance from .15 to .05 leaves the plan stale, while
zero report latency exposes the blockage before selection and replaces
the dispatch with reversible drone verification. This closes the case
only relative to those two declared repairs; zero latency is diagnostic,
not a deployable service promise. The held-out nonclosure in the main
paper supersedes this example as confirmatory evidence.

\FloatBarrier
\section{Record-Store Mechanics and Provenance Scope}\label{app:moved-system}

The store is an append-only hash chain from a GENESIS entry; each
entry carries previous, payload, and entry hashes; a
\fld{verify_chain} operation checks the chain end to end; and any
attempted rewrite raises \texttt{ImmutableWriteError}. The scope is
stated precisely: the chain provides internal tamper evidence and
API-level immutability, while adversarial immutability against a
privileged actor who could rewrite payloads and recompute the chain
requires externally anchoring the chain head, which the deployment
plan includes and the workbench does not yet enforce. Within those
assumptions provenance is structural: a disputed dispatch,
preemption, or refresh is re-litigated from the store, not from logs
beside it; the record establishes what the system represented and why
the gate accepted it, not that the belief was correct. Preemption
illustrates the discipline: the engine prefers free assets, requires
a priority margin, never preempts an asset carrying survivors, and
records proposal and commit as separate events.

\FloatBarrier
\section{The Condition Hierarchy and the Refresh Rule in Practice}\label{app:moved-mech}

Flood-SAR makes concrete a three-level hierarchy of conditions that
the contract of Section~\ref{sec:contract} types and the machinery
of Section~\ref{sec:refresh} maintains. At the bottom are
\emph{signal conditions}: tolerances on state claims, a gauge level
within $0.15$ m of its prediction, a fuel reading above a floor.
Above them are \emph{semantic conditions}: event claims that
aggregate signals into a named predicate whose truth licenses
reasoning, ``rescue boat B2 can serve incident group 027'' holds only
if the boat has fuel for the required distance, water turbulence is
below the boat's rating, and the river segment is open. A semantic
condition can change because any constituent signal drifts, and its
calibrated confidence is attached at this level, not at the sensors.
At the top are \emph{action conditions}: the gate's requirements on
a practical claim, freshness covering the commitment horizon,
adequacy scaled to consequence, authority present, no undefeated
defeaters. An action condition can be disrupted from below, a
semantic condition flips while the boat is en route, or from beside,
a priority change preempts the asset for a more urgent incident; in
either case the affected commitment is not silently rewritten but
invalidated, compensated, or escalated through the chain semantics of
Section~\ref{sec:gate-chains}, and every transition is recorded.
This hierarchy is what grounds the trigger taxonomy of
Section~\ref{sec:refresh-taxonomy}: predicted and observed triggers
watch signal conditions, structural triggers watch whether the
semantic level is still within the model's competence, and normative
triggers enforce action conditions.

The boundary makes the rule concrete. Take the running commitment:
send the rescue boat through channel E12 now (fast, but lost if the
level exceeds 2.4m) or reroute through E17 (a certain delay cost
$\Delta$). Writing $p$ for the belief that the level exceeds 2.4m and
$L_{\mathrm{fail}}$ for the loss of a swamped boat, the two actions
are indifferent at $p^{*} = \Delta / L_{\mathrm{fail}}$. If every
feasible channel's possible posteriors remain on the same side of
$p^{*}$, or if every net value $\mathrm{VoI}_j - c_j$ is nonpositive,
additional precision is wasteful for this commitment, however
uncertain the level estimate remains; adaptive refresh thus
concentrates evidence where an observation could change a pending
decision, while a fixed interval spends uniformly.

Three refinements matter in practice. Consequence scaling: tolerances
are bound to action-authorization classes, so the same view can
simultaneously clear a reversible probe and hold an irreversible
dispatch. Reversibility: a probe receives QUALIFY rather than full
authorization precisely because it is reversible, so its adequacy bar
and loss exposure are low; whether probe-first verification actually
emerges from the net-value rule, given realistic channel costs,
latencies, and information profiles, is a hypothesis the
channel-choice ablation of Section~\ref{sec:eval-rq2} tests rather
than an arithmetic fact we assert. Auditability: the routing decision
is itself recorded, which channels were priced, which won, at what
expected value, so an auditor can dispute the routing as well as the
trigger.

\FloatBarrier
\section{Artifact Inventory}\label{app:artifact}

The frozen snapshot passes 84 tests (one V-JEPA adapter test is
skipped without \texttt{torch}) and the repository verifier. All
campaign missions complete: 380 development, 475 validation, and
240 held-out R-B missions; 160 RQ2 leave-one-family-out mechanism
missions; the RQ5 guard campaign over 80 paired seeds
(3101--3180); 320 RQ6 seed-cell missions (80 seeds, 2101--2180, by
four cells); and 80 held-out final R-C missions. Every event and
TRACE chain verifies,
all declared common-random-number checks pass, and the 380-mission
calibration audit finds zero mapping or confidence mismatches across
655,761 route-record versions. The artifact contains machine-readable
campaign summaries, the validation-frozen operating points, separate
RQ3, RQ4, and re-litigation evidence packages, a source archive, a
self-contained Git bundle with the execution commit and annotated
\fld{exp-freeze-v1} tag, and a top-level SHA-256 manifest; ZIP
integrity and Git-bundle completeness verify. The public
implementation is at
\url{https://github.com/eyuchang/trace-worldmodel-flood-sar}. Its geography release asset is identical to the
execution tree; the full digest and per-file inventory are recorded
in the artifact README.

\FloatBarrier
\fi

\end{document}